\documentclass[twoside]{article}

\usepackage[accepted]{aistats2024}

\usepackage{algorithm}
\usepackage{algorithmic}

\usepackage{microtype}
\usepackage{graphicx}
\usepackage{subfig}
\usepackage{booktabs} 
\usepackage{tcolorbox}
\usepackage{wrapfig}
\usepackage{hyperref}

\usepackage{amsmath}
\usepackage{amssymb}
\usepackage{mathtools}
\usepackage{amsthm}
\usepackage{rotating}
\usepackage{comment}
\usepackage{enumerate}
\usepackage{enumitem}

\usepackage[capitalize,noabbrev]{cleveref}

\theoremstyle{plain}
\newtheorem{theorem}{Theorem}[section]
\newtheorem{proposition}[theorem]{Proposition}
\newtheorem{lemma}[theorem]{Lemma}
\newtheorem{corollary}[theorem]{Corollary}
\theoremstyle{definition}
\newtheorem{definition}[theorem]{Definition}
\newtheorem{assumption}[theorem]{Assumption}
\theoremstyle{remark}

\usepackage[textsize=tiny]{todonotes}

\usepackage[round]{natbib}

\usepackage{enumitem}
\newcommand{\bigO}{\mathcal{O}}

\newcommand{\tr}{\text{Tr}}

\newcommand{\shgd}{\scalebox{0.6}{\text{SHGD}}}
\newcommand{\seg}{\scalebox{0.6}{\text{SEG}}}
\newcommand{\bgamma}{{\boldsymbol{\gamma}}}

\DeclareMathOperator{\diag}{diag}

\def \lf{\left\lfloor}   
\def \rf{\right\rfloor}

\newcommand{\E}{{\mathbb E}}

\newcommand{\R}{{\mathbb{R}}}

\newtcolorbox{mybox}[2][]{
  colframe = white, 
  colback  = gray!7,
  #1
}

\newcommand{\mycoloredbox}[2]{%
  \begin{tcolorbox}[
    colback=gray!7, 
    colframe=white, 
    boxrule=0mm, 
    left=-3pt, 
    right=4pt, 
    top=3pt, 
    bottom=3pt, 
    width=0.48\textwidth, 
    #1 
  ]
  #2
  \end{tcolorbox}
}

\begin{document}

%

%

\twocolumn[

\aistatstitle{SDEs for Minimax Optimization}

\aistatsauthor{Enea Monzio Compagnoni \And Antonio Orvieto \And Hans Kersting}

\aistatsaddress{Department of Mathematics \\ and Computer Science \\ University of Basel \And ELLIS Institute Tübingen\\ MPI for Intelligent Systems\\ Tübingen AI Center \And Yahoo! Research}

\aistatsauthor{Frank Norbert Proske \And Aurelien Lucchi }

\runningauthor{Enea Monzio Compagnoni, Antonio Orvieto, Hans Kersting, Frank Norbert Proske, Aurelien Lucchi}

\aistatsaddress{Department of Mathematics \\ University of Oslo \And  Department of Mathematics \\ and Computer Science \\ University of Basel } ]

\begin{abstract}
Minimax optimization problems have attracted a lot of attention over the past few years, with applications ranging from economics to machine learning. While advanced optimization methods exist for such problems, characterizing their dynamics in stochastic scenarios remains notably challenging. In this paper, we pioneer the use of stochastic differential equations (SDEs) to analyze and compare Minimax optimizers. Our SDE models for Stochastic Gradient Descent-Ascent, Stochastic Extragradient, and Stochastic Hamiltonian Gradient Descent are provable approximations of their algorithmic counterparts, clearly showcasing the interplay between hyperparameters, implicit regularization, and implicit curvature-induced noise. This perspective also allows for a unified and simplified analysis strategy based on the principles of Itô calculus. Finally, our approach facilitates the derivation of convergence conditions and closed-form solutions for the dynamics in simplified settings, unveiling further insights into the behavior of different optimizers.
\end{abstract}

\vspace{-0.5cm}

\section{INTRODUCTION} \label{sec:intro}
\vspace{-2mm}

Minimax optimization plays a fundamental role in decision theory, game theory, and machine learning~\citep{goodfellow2016deep}. The problem it addresses is finding the solution of the following optimization problem:
\begin{equation}
\min_{x \in \mathcal{X}} \max_{y \in \mathcal{Y}} \left[ f(x,y) := \frac{1}{N} \sum_{i=1}^N f_i(x,y) \right],
\label{eq:empirical_loss}
\end{equation}
where $f,f_i: \mathcal{X} \times \mathcal{Y} \to \R$ for $i=1,\dots,N$. In machine learning, $f$ is an empirical risk function where $f_i$ is the contribution of the $i$-th data point of the training data. In this notation, $(x,y) \in \mathcal{X} \times \mathcal{Y} $ is a vector of trainable parameters and $N$ is the size of the dataset. The goal is to find optimal saddle points $(\mathrm{x}^*, \mathrm{y}^*)$ such that
\begin{equation*}
    f\left(\mathrm{x}^*, \mathrm{y}\right) \leq f\left(\mathrm{x}^*, \mathrm{y}^*\right) \leq f\left(\mathrm{x}, \mathrm{y}^*\right) \quad \forall \mathrm{x} \in \mathcal{X}, \quad \forall \mathrm{y} \in \mathcal{Y}.
\end{equation*}
The most intuitive algorithm to solve Eq.~\eqref{eq:empirical_loss} is Gradient Descent Ascent (GDA). However, its updates are computationally expensive for large datasets. Therefore, a common choice is to use \textit{mini-batches} to approximate the gradients, which gives rise to Stochastic Gradient Descent Ascent (SGDA).
Unfortunately, it is known that both GDA and SGDA do not converge on relatively simple landscapes such as $f(x,y) = x y$ for $(x,y) \in \R$. This led to the design of alternative optimizers such as Extragradient~(EG)~\citep{korpelevich1976extragradient} and Hamiltonian GD~\citep{balduzzi2018mechanics}. While these methods exhibit more favorable convergence guarantees compared to SGDA, they are relatively complex to study and some of their properties are still not well understood, especially in a stochastic setting.

In this paper, we leverage continuous-time models in the form of stochastic differential equations~(SDEs) to study these minimax optimizers. SDEs have recently become popular in the \textit{minimization} community: They provide a unified and simplified analysis strategy rooted in Itô calculus which facilitates the derivation of novel insights about the discrete algorithms, see e.g.~\citep{Su2014nesterov, li2017stochastic}. It is worth mentioning that the interest in applying SDEs to minimax problems has been a topic of prior research discussions~ \citep{chavdarova2022continuous}. Following the framework of \cite{li2017stochastic} for minimization, our work provides the first \emph{formal} derivation --- rooted in the theory of weak approximation~\citep{mil1986weak} --- of the SDEs of SGDA,
\begin{align} \label{eq:SGDA_Discr_Update}
z_{k+1} = z_k - \eta F_{\gamma_k }(z_k),
\end{align}
SEG,
\begin{align}\label{eq:SEG_Discr_Update}
z_{k+1} = z_k - \eta F_{\gamma^1_k}(z_k - \rho F_{\gamma^2_k}(z_k)),
\end{align}
and SHGD
\begin{equation}\label{eq:SHGD_Discr_Update}
    z_{k+1} = z_k - \eta \nabla \mathcal{H}_{\gamma_k^1,\gamma_k^2}(z_k),
\end{equation}
where $F$ is the drift field and $\mathcal{H}$ the Hamiltonian:
\begin{align}
    &F_{\gamma}(z) = F_{\gamma}(x,y) := (\nabla_x f_{\gamma}(x,y),-\nabla_y f_{\gamma}(x,y)),\\
    &\mathcal{H}_{\gamma^1, \gamma^2}(z):= \frac{F_{\gamma^1}^{\top}(z)F_{\gamma^2}(z)}{2}.
\end{align}
Above, $\eta \in \R^{>0}$ is the stepsize and $\rho \in \R$ is the extra stepsize of SEG\footnote{We also support the cases where the stepsizes and extra steps depend on time, e.g. $\eta_k$ and $\rho_k$, as well as depend on the coordinates, e.g. $\eta = \left(\eta_1, \cdots, \eta_d \right)$.}. The mini-batches $\{ \gamma^j_k \}$ are modelled as i.i.d.~random variables uniformly distributed on $\{ 1, \cdots, N \}$, and of size $B\geq 1$.

Formally, these continuous-time models are weak approximations, i.e.~approximations in distribution, of their respective discrete-time algorithms.
We will exploit these models to derive novel insights into the convergence behavior, the effect of the noise and the curvature of the landscape, or the role of hyper-parameters such as the extra stepsize $\rho$ appearing in SEG. 

\paragraph{Contributions.}

\vspace{-0.3cm}

\begin{itemize}[leftmargin=*]
\setlength\itemsep{0.2em}

\vspace{-0.3cm}

\item We provide the \textit{first formal} derivation of the SDE models of popular minimax optimizers. Then, we use them to make the following additional contributions:
\begin{enumerate}[itemsep=0mm, left=0.1em]
\item \textbf{Moderate Exploration regime.} If $\rho = \mathcal{O}(\eta)$, we show that SEG essentially behaves like SGDA;
\item \textbf{Aggressive Exploration regime} \citep{hsieh2020explore}. For $\rho = \mathcal{O}(\sqrt{\eta})$, the dynamics of SEG can be interpreted as that of SGDA on an \textbf{implicitly} regularized vector field with additional \textit{implicit curvature-induced} noise;
\item SHGD uses \textbf{explicit} curvature-based information. Thus, it has an \textit{explicit curvature-induced} noise;
\item We characterize the evolution of the Hamiltonian under the dynamics of SEG and SHGD; 
    \item We use the latter to derive convergence conditions for SEG and SHGD on a wide class of functions.
\end{enumerate}
\item For Bilinear Games with different noise structures:
\begin{enumerate}[itemsep=0mm, left=0.1em]
    \item We explicitly solve the differential equation of the Hamiltonian, thus elucidating the interplay of all hyperparameters in determining the speed of convergence (or divergence) of these methods;
    \item We provide necessary and sufficient conditions for stepsize schedulers to recover convergence.
\end{enumerate}    

\item We explicitly solve the SDEs for some \textit{Quadratic} Games, meaning that we derive the \textit{first} closed-form formula for the dynamics of SEG and SHGD on these landscapes. This allows for a 1-to-1 comparison of the two optimizers, particularly of their first and second moments.
One key takeaway of this comparison is that selecting $\rho$ is a matter of trade-off between the speed of convergence and asymptotic optimality: Our formulas show how a suitable choice of $\rho$ allows SEG to match (or outperform) SHGD w.r.t. convergence speed but negatively impacts its optimality (the iterates converge to a larger neighborhood of the optimum). Interestingly, the curvature determines whether SEG or SHGD is faster at converging as well as more suboptimal. Importantly, we provide the first experimental and theoretical evidence that \textbf{negative} $\rho$ might be advisable for certain landscapes.

\item Finally, we present extensive experiments on various relevant minimax problems: these are meant to verify that each formula derived from our SDEs correctly describes the behavior of the respective discrete-time algorithms. Figure \ref{fig:SDEs} offers a preliminary glimpse at the accuracy of the SDEs approximations.

\end{itemize}

\begin{figure}%
    \centering
    \vspace{-6mm}
    \subfloat{{\includegraphics[width=0.49\linewidth]{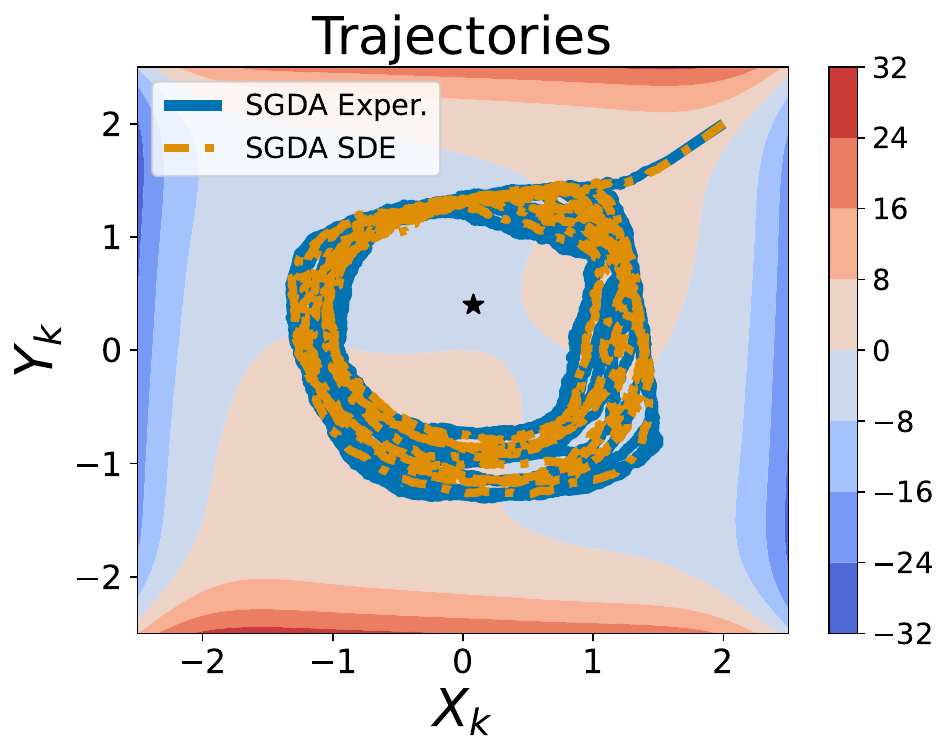} }}%
    \subfloat{{\includegraphics[width=0.49\linewidth]{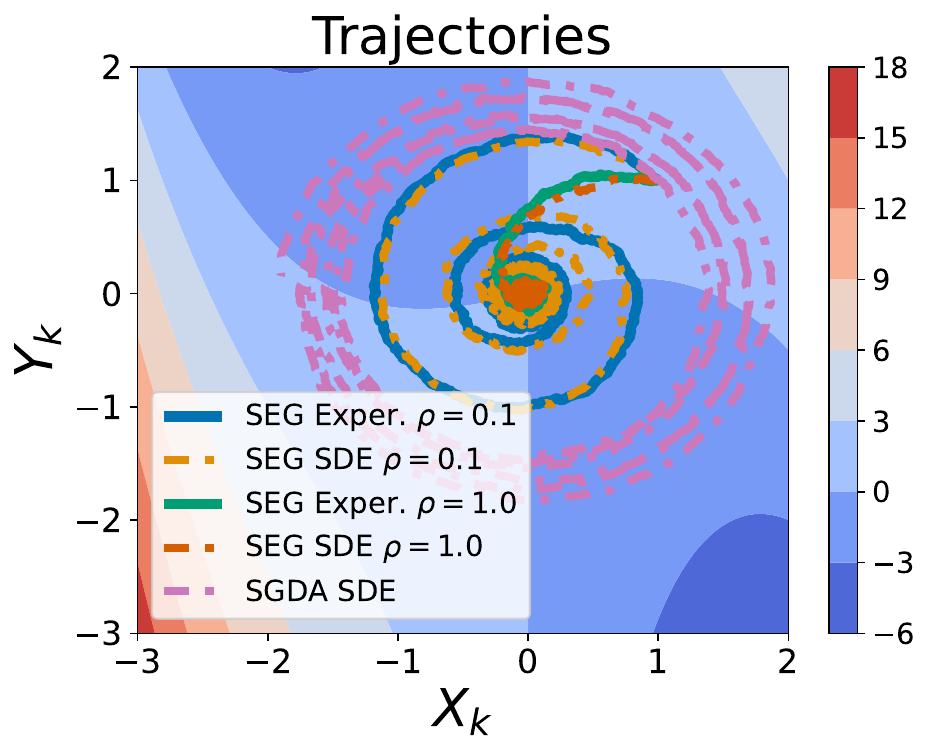} }} \\
    \subfloat{{\includegraphics[width=0.49\linewidth]{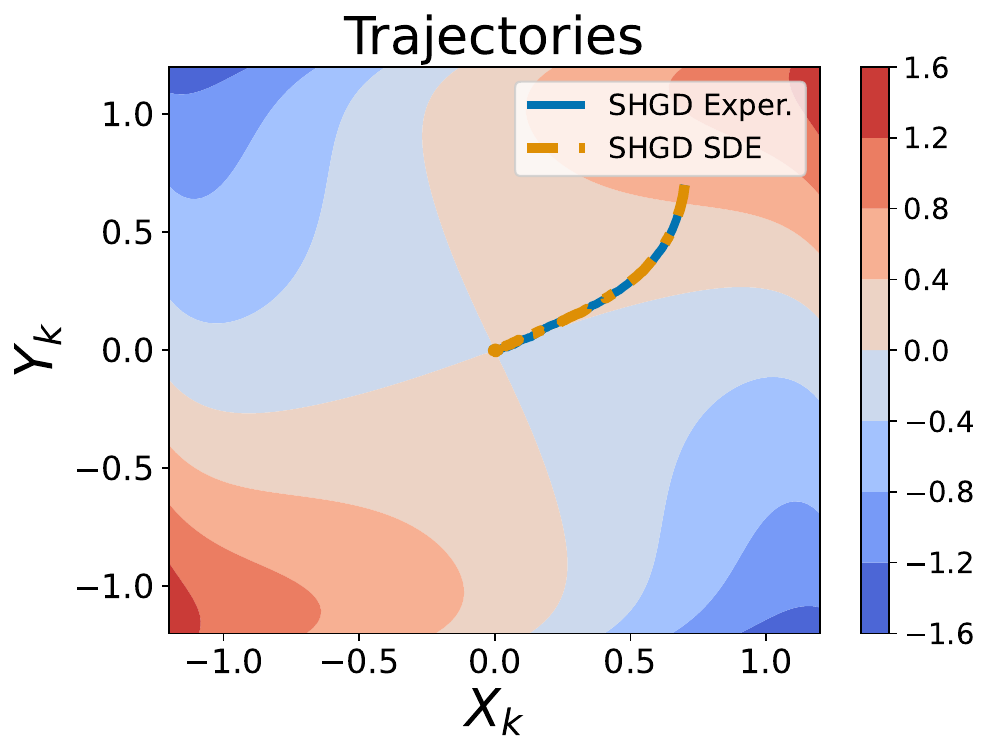} }}%
    \subfloat{{\includegraphics[width=0.49\linewidth]{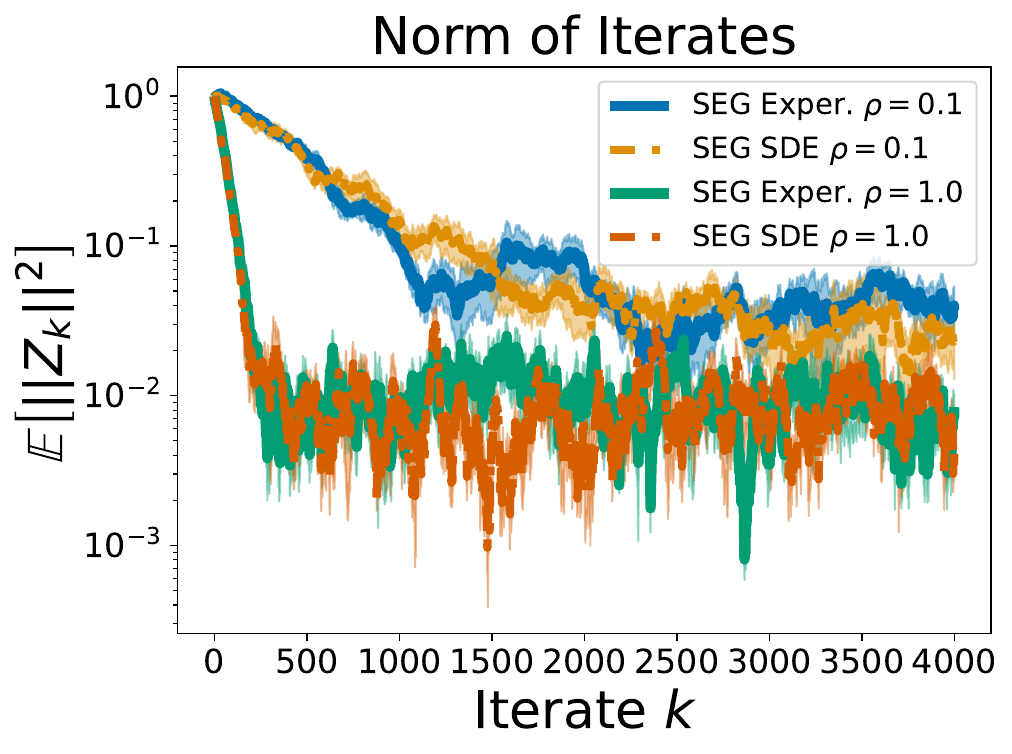} }}%
    \caption{Empirical validation of Theorem \ref{thm:SEG_SDE_Insights} and \ref{thm:SHGD_SDE_Insights}: The trajectories of the simulated SDEs match those of the respective algorithms averaged over $5$ runs - 
    That of SGDA gets trapped in limit cycles as well (Top Left); That of SHGD converges to the optimum of a highly nonlinear landscape (Bottom Left); The SDE of SGDA would not be a good model for SEG (Top Right); The SDEs and the optimizers move along the trajectory at the same speed (Bottom Right). For a description of the landscapes and of the simulation settings for the SDEs, see Appendix \ref{app:Experiments}.}%
    \label{fig:SDEs}%
\end{figure}

\vspace{-0.4cm}

\section{RELATED WORKS}\label{sec:RelWorks}

\vspace{-0.4cm}

We start by discussing existing continuous-time analysis for minimax optimization and related applications. For related works regarding SGDA, SEG, SHGD, and bilinear games, we refer the reader to Appendix \ref{app:AddRelWorks}.

\vspace{-0.35cm}

\paragraph{ODE Approximations and Applications.}
Several works use \textit{continuous-time models} to describe the dynamics of minimax optimizers. First, \cite{ryu2019ode} informally derived ODEs to study Stochastic Gradient Methods with Optimism and Anchoring. Then, \cite{lu2022sr} formally showed that different saddle-point optimizers yield the same ODE and derived High-Resolution (Ordinary) Differential Equations (HRDEs) to provide convergence conditions on a wide class of problems. Similarly, \cite{chavdarova2021last} derived HRDEs as well and established the convergence of certain methods in continuous time on bilinear games. Finally, \cite{hsieh2021limits} modeled a wide class of zeroth- and first-order minimax algorithms with ODEs and proved that they may be subject to inescapable convergence failures, meaning that they could get attracted by spurious attractors. Unfortunately, their approach based on Robbins–Monro templates cannot handle ergodic averages, second-order methods, adaptive methods, and constant stepsizes.

\vspace{-0.35cm}

\paragraph{SDE Approximations and Applications.}
\citep{li2017stochastic} first proposed a formal theoretical framework to derive SDEs to appropriately capture the intrinsic stochasticity of stochastic optimizers. These SDEs can be understood as weak approximations of stochastic gradient algorithms (See Definition \ref{def:weak_approximation}). 
SDEs open to a variety of concrete applications that include \emph{stochastic optimal control} to select the stepsize~\citep{li2017stochastic,li2019stochastic} or the batch size~\citep{zhao2022batch} and \emph{scaling rules}~\citep{Malladi2022AdamSDE} to adjust the optimization hyperparameters w.r.t.~the batch size.
Additionally, SDEs give access to the fine-grained structure of the interaction between stochasticity and curvature. For example, the study of \emph{escape times} of SGD from minima of different sharpness~\citep{xie2020diffusion}, the factors influencing the minima found by SGD~\cite{jastrzkebski2017three}, the convergence bounds for mini-batch SGD and SVRG derived in~\cite{orvieto2019continuous}, and the fundamental interplay between noise and curvature of the landscape for SAM~\citep{compagnoni2023sde}. For more references, see \citep{kushner2003stochastic,ljung2012stochastic,chen2015convergence,mandt2015continuous,chaudhari2018stochastic,zhu2018anisotropic,ijcai2018p307,an2020stochastic}. A gentle introduction to SDEs is provided in Appendix \ref{subsec:SDE}

\vspace{-0.3cm}

\section{RESULTS \& INSIGHTS: THE SDEs}\label{sec:Insights}
\vspace{-2mm}

This section provides the general formulations of the SDEs of SGDA (Theorem \ref{thm:SGDA_SDE_Insights_Full}), SEG (Theorem \ref{thm:SEG_SDE_Insights}), and SHGD (Theorem \ref{thm:SHGD_SDE_Insights}). Due to the technical nature of the analysis, we refer the reader to the appendix for the complete formal statements and proofs.
\begin{assumption}\label{ass:regularity_f_Insights}
We assume that
\begin{enumerate}
\item $\nabla f, \nabla f_i $ satisfy a Lipschitz condition: $ \exists L>0 $ s.t. $|\nabla f(u)-\nabla f(v)|+\sum_{i=1}^N\left|\nabla f_i(u)-\nabla f_i(v)\right| \leq L|u-v|$;

\vspace{-0.2cm}

\item $ f, f_i $ and their partial derivatives up to order 7 have polynomial growth;

\vspace{-0.2cm}

\item $ \nabla f, \nabla f_i $ satisfy a linear growth condition: $ \exists M>0 $ s.t. $|\nabla f(z)|+\sum_{i=1}^N\left|\nabla f_i(z)\right| \leq M(1+|z|).$
\end{enumerate}
\end{assumption}
\begin{definition}[Weak Approximation]\label{def:weak_approximation}
A continuous-time stochastic process $\{Z_t\}_{ t \in [0, T]}$ is an order $\alpha$ weak approximation (or $\alpha$-order SDE) of a discrete stochastic process $\{z_k\}_{k=0}^{\lf T/\eta \rf}$ if for every polynomial growth function $g$, there exists a positive constant $C$, independent of the stepsize $\eta$, such that $ \max _{k=0, \ldots, \lf T/\eta \rf}\left|\E g\left(z_k\right)-\E g\left(Z_{k \eta}\right)\right| \leq C \eta^\alpha.$
\end{definition}
This definition comes from the field of numerical analysis of SDEs, see \cite{mil1986weak}. When $g(z)=\lVert z \rVert^j$, the bound restricts the disparity between the $j$-th moments of the discrete and the continuous process. 
\vspace{-0.2cm}
\subsection{SGDA SDE}
\begin{theorem}[SGDA SDE - Informal Statement of Theorem \ref{thm:SGDA_SDE}] \label{thm:SGDA_SDE_Insights_Full}
Under sufficient regularity conditions, the solution of the following SDE is an order $1$ weak approximation of the discrete update of SGDA \eqref{eq:SGDA_Discr_Update}:
\begin{align}\label{eq:SGDA_SDE_Full_Insights}
& d Z_t = - F \left( Z_t \right) dt +\sqrt{\eta\Sigma\left( Z_t \right)}d W_{t},
\end{align}
where $\Sigma(z)$ is the noise covariance
\begin{align} \label{eq:SGDA_Covar_Insights}
  \Sigma(z) = \E[\xi_{\gamma}(z)\xi_{\gamma}(z)^\top],
\end{align}
and $\xi_{\gamma}(z):= F \left(z\right) - F_{\gamma}\left(z\right)$ the noise in the sample $F_\gamma$.
\end{theorem}
\subsection{SEG SDE}
A notable characteristic of SEG is the inclusion of the variable $\rho$ that controls the magnitude of the extra step. This variable plays an important role in the derivation of the SDE of SEG and one has to differentiate between two different regimes:
\mycoloredbox{}{\begin{enumerate}
    \item When $\rho \sim \eta$, the SDE of SEG is the same as SGDA, which is consistent with the literature on ODEs \citep{chavdarova2021last,lu2022sr}. The formal proof is given in Theorem \ref{thm:SEG_SDE_small_rho}.
    \item  However, if the extra stepsize $\rho$ is sizeably larger than $\eta$ \citep{hsieh2020explore} (i.e. $\rho=\mathcal{O}(\sqrt{\eta}))$, SEG enters a more exploratory regime, for which the SDE becomes distinct from the first regime.
\end{enumerate}
}

Before presenting our main result, we introduce some notation. Let $\bgamma := (\gamma^1,\gamma^2)$, $\bar F_{\bgamma}(z) := \nabla F_{\gamma^1}(z) F_{\gamma^2}(z)$, and $\bar F(z):=\E[\bar F_{\bgamma}(z)]$ be its expectation. We denote the noise in $\bar F$ as $\bar \xi_{\bgamma}(z) := \bar F_{\bgamma}(z) -\bar F(z)$ and consider the mixed (non-symmetric) covariance matrix $\bar \Sigma(z) = \E[\xi_{\gamma^1}(z)\bar \xi_{\bgamma}(z)^\top]$.

\begin{theorem}[Informal Statement of Theorem \ref{thm:SEG_SDE}] \label{thm:SEG_SDE_Insights}
Let
\begin{align}
    & F^{\seg}(z):= F(z) - \rho \bar F(z),\\
    & \Sigma^{\seg}(z) := \Sigma(z) +\rho\left[\bar \Sigma(z) + \bar \Sigma(z)^\top\right].
\end{align}
Under sufficient regularity conditions and $\rho = \mathcal{O}(\sqrt{\eta})$, the solution of the following SDE is the order 1 weak approximation of the discrete update of SEG
\begin{align}\label{eq:SEG_SDE_Full_Insights}
& d Z_t = - F^{\seg}\left( Z_t \right) dt+ \sqrt{\eta\Sigma^{\seg}\left( Z_t \right)}d W_{t}.
\end{align}
\end{theorem}

\begin{proof}
    To prove that the SDE is a weak approximation of SEG as per Definition \ref{def:weak_approximation}, we prove that the first and second moments of its discretization match those of SEG up to an error of order $\eta$ and $\eta^2$, respectively. 
\end{proof}

For didactic reasons, we now present Corollary \ref{thm:SEG_SDE_Simplified_Insights}, a consequence of Theorem \ref{thm:SEG_SDE_Insights} that provides a more interpretable SDE for SEG which we will use to establish a comparison with SGDA (Eq.\eqref{eq:SGDA_SDE_Full_Insights}) and SHGD (Eq.\eqref{eq:SHGD_SDE_Simplified_Insights}).

\begin{figure}%
    \centering
    \subfloat{{\includegraphics[width=0.49\linewidth]{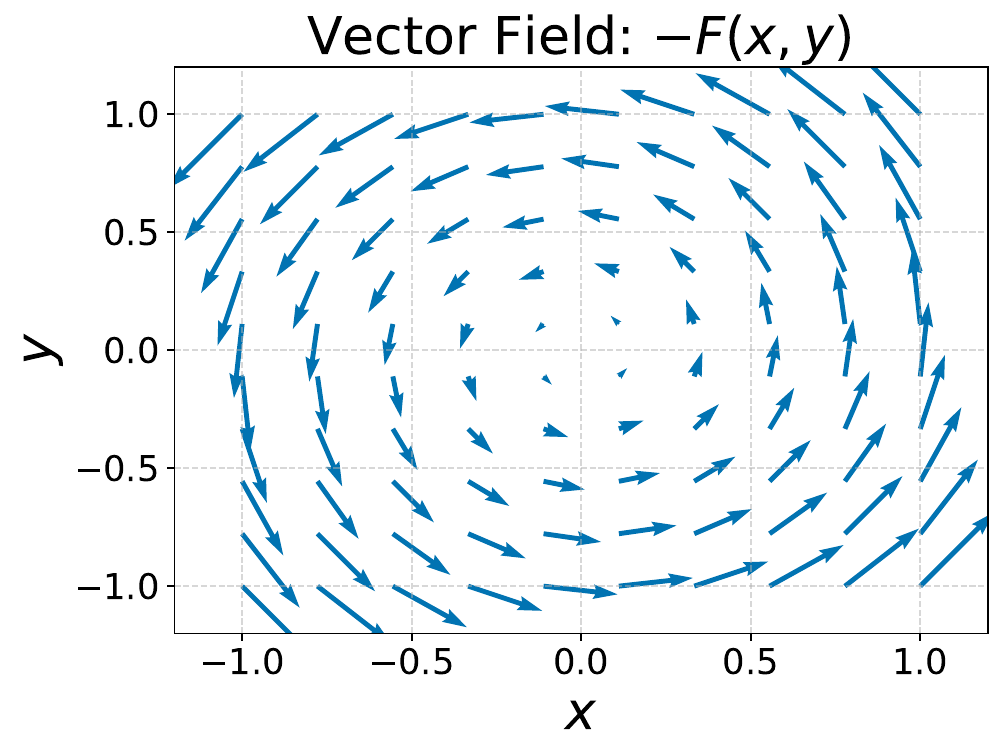} }}%
    \subfloat{{\includegraphics[width=0.49\linewidth]{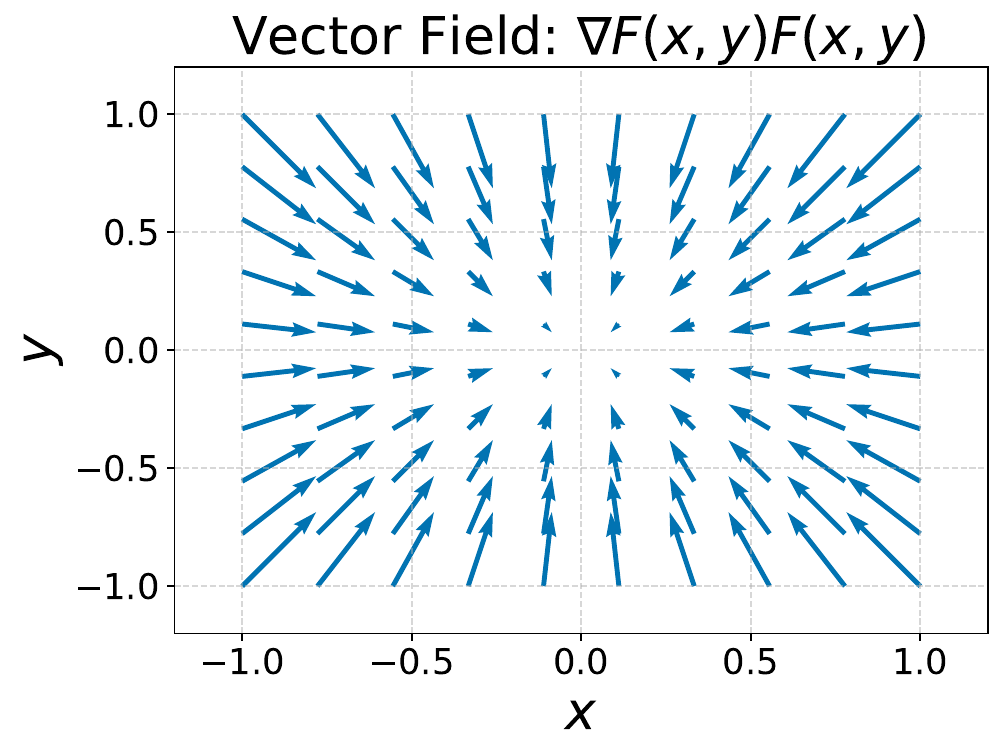} }} \\
    \subfloat{{\includegraphics[width=0.49\linewidth]{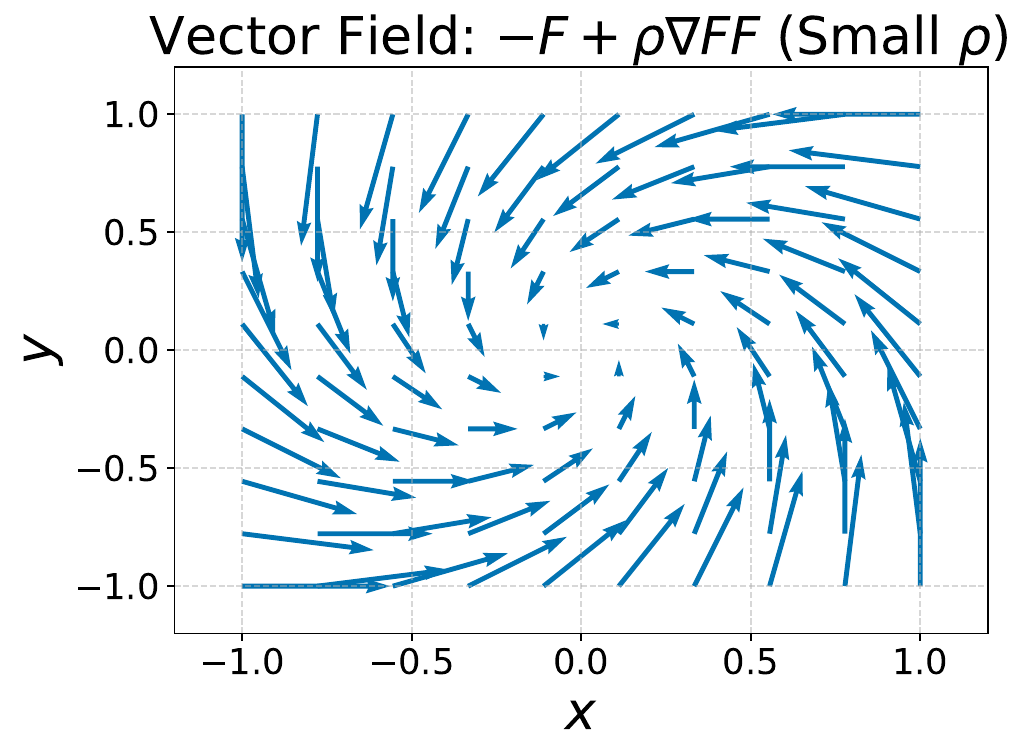} }}%
    \subfloat{{\includegraphics[width=0.49\linewidth]{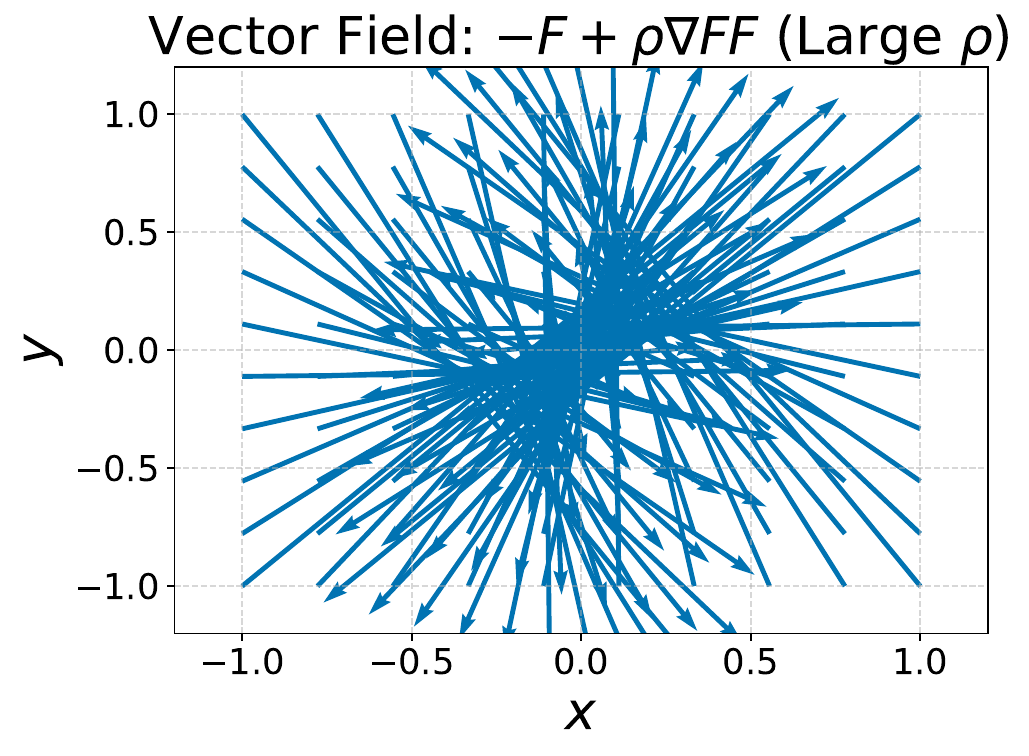} }}%
    \vspace{-2mm}
    \caption{Graphical representation of the \textit{implicit regularization} of the vector field of SEG for $f(x,y) = x y$: $-F$ spins the dynamics in a circle (Top Left); $+ \nabla F F$ pulls it towards $0$ (Top Right); If $\rho$ is small, $-F + \rho \nabla F F$ combines the two fields and spirals towards the origin (Bottom Left); If $\rho$ is large, $-F + \rho \nabla F F$ is a chaotic field that makes the dynamics diverge (Bottom Right).}%
    \label{fig:VectorFields}%
    \vspace{-2mm}
\end{figure}
\begin{corollary}[Informal Statement of Corollary \ref{thm:SEG_SDE_Simplified_SameSample}] \label{thm:SEG_SDE_Simplified_Insights}
Under the assumptions of Theorem~\ref{thm:SEG_SDE_Insights}, that $\gamma^1=\gamma^2=\gamma$, and that the stochastic gradients are $\nabla_x f_{\gamma}(z) = \nabla_x f(z) + U^x$ and $\nabla_y f_{\gamma}(z) = \nabla_y f(z) + U^y$ such that $U^x$ and $U^y$ are independent noises that do not depend on $z$, the following SDE provides a 1 weak approximation of the discrete update of SEG
\begin{align}\label{eq:SEG_SDE_Simplified_SameSample_Insights}
d Z_{t} = & - \left(F(Z_{t}) - \rho \nabla F(Z_{t}) F(Z_{t}) \right) dt \\
 & +   \left(\mathbf{I}_{2d} - \rho \nabla F\left( Z_{t} \right) \right) \sqrt{\eta \Sigma}  d W_t. \nonumber
\end{align}
If instead~(under the same assumptions) $\gamma^1$ and $\gamma^2$ are uncorrelated, the SDE has a drift regularization term but no variance regularization:
\begin{align}
d Z_{t} = - \left(F(Z_{t}) - \rho \nabla F(Z_{t}) F(Z_{t}) \right) dt + \sqrt{\eta \Sigma}  d W_t. \nonumber
\end{align}
\end{corollary}
\vspace{-3mm}
\begin{proof}
The noise assumption implies $\nabla F_{\gamma^1}(z) = \nabla F(z)$.  Therefore
$\bar F(z) := \nabla F(z) F(z)$. Next, note that $\bar \xi_{\bgamma}(z) := \nabla F(z)\xi_{\gamma^2}(z)$ and therefore $\bar \Sigma(z) = \E[\xi_{\gamma^1}(z)\xi_{\gamma^2}^\top(z)\nabla F(z)^\top] = \Sigma\nabla F(z)^\top$ if $\gamma^1=\gamma^2$ and is zero otherwise. Next, note that $\Sigma^{\seg}(z) := \Sigma+\rho\Sigma \nabla F(z)^\top + \rho \nabla F(z)\Sigma= (\mathbf{I}_{2d}+\rho\nabla F(z))\Sigma(\mathbf{I}_{2d}+\rho\nabla F(z)^\top) + \mathcal{O}(\rho^2)$. Since terms of order $\rho^2$ have a vanishing influence, this proves the result. 
\end{proof}

\subsection{SHGD SDE}
\begin{theorem}[SHGD SDE - Informal Statement of Theorem \ref{thm:SHGD_SDE}] \label{thm:SHGD_SDE_Insights}
Let $\mathcal{H}_\bgamma := \mathcal{H}_{\gamma^1,\gamma^2}$ and let us define 
\begin{align}
    &F^{\shgd}(z):=\nabla \E \left[ \mathcal{H}_{\bgamma} \left( z \right)  \right],\\
    &\Sigma^{\shgd}(z):=\E \left[ \hat \xi_\bgamma(z)\hat \xi_\bgamma(z)^\top\right],
\end{align}
where $\hat \xi_\bgamma(z) = F^{\shgd}(z) - \nabla \mathcal{H}_{\bgamma} \left( z \right)$.
 Under sufficient regularity conditions, the solution of the following SDE is the order 1 weak approximation of the discrete update of SHGD \eqref{eq:SHGD_Discr_Update}:
\begin{align}\label{eq:SHGD_SDE_Full_Insights}
d Z_t = - F^{\shgd} \left( Z_t \right) dt +\sqrt{\eta\Sigma^{\shgd}\left( Z_t \right)}d W_{t}.
\end{align}
\end{theorem}

Once again, we provide a more interpretable SDE under additional assumptions:
\begin{corollary}[Informal Statement of Corollary \ref{thm:SHGD_SDE_Simplified_SameSample}] \label{thm:SHGD_SDE_Simplified_Insights}
Under the assumptions of Theorem~\ref{thm:SHGD_SDE_Insights}, that $\gamma^1=\gamma^2=\gamma$, and that the stochastic gradients are $\nabla_x f_{\gamma}(z) = \nabla_x f(z) + U^x$ and $\nabla_y f_{\gamma}(z) = \nabla_y f(z) + U^y$ such that $U^x$ and $U^y$ are independent noises that do not depend on $z$, the SDE is
\begin{align}\label{eq:SHGD_SDE_Simplified_Insights}
    d Z_t & = - \nabla \mathcal{H} \left( Z_t \right) dt + \sqrt{\eta} \nabla^2 f \left( Z_t \right) \sqrt{\Sigma} d W_{t}.
\end{align}
If instead $\gamma^1$ and $\gamma^2$ are uncorrelated, the SDE is the same but with less variance:
\begin{equation}
    d Z_t = - \nabla \mathcal{H} \left( Z_t \right) dt + \sqrt{\frac{\eta}{2}} \nabla^2 f \left( Z_t \right) \sqrt{\Sigma} d W_{t}.
\end{equation}
\end{corollary}
\paragraph{Empirical Validation}
Figure \ref{fig:SDEs} shows the empirical validation of Theorem \ref{thm:SEG_SDE_Insights} and Theorem \ref{thm:SHGD_SDE_Insights}: The top left shows that the SDE of SGDA matches the algorithm and is also attracted to a limit cycle. The bottom left shows that the SDE of SHGD matches the empirical optimization of a highly nonlinear landscape. The top right shows that the SDE of SEG matches its discrete-time counterpart for different values of $\rho$. Also, the SDE of SGDA is not a good model to describe the dynamics of SEG. The bottom right shows the evolution of the norm of the iterates in time: We understand that the SDE $Z_{k\eta}$ and optimizer $z_k$ move at the exact same speed along the trajectory --- This justifies their use as investigation tools.
Figure \ref{fig:SDE_Validation} shows that if $\rho = \mathcal{O}(\eta)$ or smaller, the SDE of SGDA models the dynamics of SEG accurately. However, if $\rho = \mathcal{O}(\sqrt{\eta})$ or larger, the SDE of SGDA no longer does so while the SDE of SEG does.
All experiments are averaged over 5 runs and additional details are in Appendix \ref{app:Experiments}.
\subsection{Comparisons}
\label{subsec:Interpretation}
There are three notable observations we immediately derive from the SDEs presented above:
\mycoloredbox{}{
\begin{enumerate}
    \item Let us use $\Tilde{\nabla}:= \left(\nabla_x, -\nabla_y\right)$. Then, one can see that the drift term $F$ is simply equal to $F=\Tilde{\nabla} f$ for SGDA, while SEG \textit{implicitly} introduces an additional regularizer such that
    \begin{equation*}
    F^{\seg}= \Tilde{\nabla} \left[f+ \frac{\rho}{2} \left[\lVert \nabla_y f  \rVert_2^2- \lVert \nabla_x f  \rVert_2^2 \right]\right].
    \end{equation*}
    Therefore, the dynamics of SEG is equivalent to that of SGDA on an \textbf{implicitly} regularized vector field. Figure \ref{fig:VectorFields} illustrates this phenomenon.
    \item The presence of $\rho \nabla F$ in the diffusion term of SEG shows that the extra step \textbf{implicitly} adds (on top of that of SGDA) a noise component that depends on the curvature of the landscape.
    \item SHGD is a second-order method that \textbf{explicitly} optimizes the Hamiltonian which by definition uses curvature-based information. The SDE in Eq.\eqref{eq:SHGD_SDE_Simplified_Insights} shows how this results in $\nabla^2 f$ directly affecting its noise structure.
\end{enumerate}}
\section{CONVERGENCE CONDITIONS}\label{sec:ConvConditions}
In this section, we derive the ODE that characterizes the evolution of the expected Hamiltonian $H_t$ along the dynamics of SEG and SHGD. We use it to derive convergence conditions on a wide class of functions. Then, we focus on Bilinear Games where we can explicitly solve the ODE which allows us to single out the role of each ingredient of the dynamics. Finally, we provide sufficient conditions to craft stepsize schedulers that induce convergence.
\subsection{SHGD}
We begin by introducing an auxiliary result that elucidates the evolution of the Hamiltonian. Here we denote by $Z_t$ the stochastic process that defines the evolution of SHGD. We also define $H_t := \E_\bgamma[\mathcal{H}_{\bgamma}(Z_t)]$ and $\Sigma^{\shgd}_t:= \Sigma^{\shgd}(Z_t)$. Then,
$$\E\left[\dot H_t\right] = - \E\left[\|\nabla H_t\|^2\right] + \frac{\eta}{2} \tr\left(\E\left[\Sigma^{\shgd}_t\nabla^2 H_t\right]\right).$$
We observe that:
\mycoloredbox{}{\begin{enumerate}
    \item $- \E \left[\|\nabla H_t\|^2\right]$ comes from the drift of the SDE and is pulling the dynamics towards regions with zero energy;
    \item $\frac{\eta}{2} \tr\left(\E\left[\Sigma^{\shgd}_t\nabla^2 H_t\right]\right)$ is induced by the diffusion term and has an adversarial effect;
    \item Convergence can only be achieved if the pulling force is stronger than the repulsive one, even at vanishing energies.
\end{enumerate}}
We formalize this in Theorem \ref{theorem:SHGD_GenConv_Insights} and Corollary \ref{cor:SHGD_Conv_Insights}.

\begin{theorem}[SHGD General Convergence] \label{theorem:SHGD_GenConv_Insights}Consider the solution $Z_t$ of the SHGD SDE with $\gamma^1\ne\gamma^2$.
Let $v_t := \E \left[\E_{\bgamma}  \left[ \lVert \nabla \mathcal{H}_t \left( Z_t \right) -  \nabla \mathcal{H}_{\bgamma} \left( Z_t \right) \rVert_2^2 \right]\right]$ measure the error in $ \nabla \mathcal{H}$, in expectation over the whole randomness up to time $t$. Suppose that:
\begin{enumerate}
    \item The smallest eigenvalue (in absolute value) $\mu$ of $\nabla^2 f(z)$ is non-zero;
    \item $\lVert \nabla^2 H(z)\rVert_{\text{op}} < \mathcal{L}_{\mathcal{T}}$, for all $z\in\R^{2d}$.
\end{enumerate}
Then,
\begin{equation}
    \E \left[ H_t \right] \leq  e^{-2\mu^2t} \left[ H_0 + \frac{\eta\mathcal{L}_{\mathcal{T}}}{2} \int_0^t v_s  e^{2\mu^2 s} ds \right].
\end{equation}
\end{theorem}

\begin{proof}
    We derive the SDE of $H_t$ via Itô's Lemma and take its expectation to obtain the ODE of $\E \left[ H_t \right]$. Then, we use the assumptions to derive a bound on it.
\end{proof}

\begin{corollary} \label{cor:SHGD_Conv_Insights}
Under the assumptions of Theorem \ref{theorem:SHGD_GenConv_Insights}, if for $\mathcal{L}_{\mathcal{V}}>0$
\begin{equation}\label{eq:Scaling_Vt}
    v_t \leq \mathcal{L}_{\mathcal{V}} \E \left[ H_t \right],
\end{equation}
the solution is more explicit:
\begin{equation}
    \E \left[H_t \right] \leq  H_0  e^{\left(-2 \mu^2 + \eta \mathcal{L}_{\mathcal{V}} \mathcal{L}_{\mathcal{T}} \right)t}.
\end{equation}
If instead
\begin{equation}\label{eq:Bounded_Vt}
    v_t \leq \mathcal{L}_{\mathcal{V}},
\end{equation}
we have
\begin{equation}
    \E \left[ H_t \right] \leq H_0 e^{-2 \mu^2 t} + \left( 1 - e^{-2 \mu^2 t}  \right) \frac{\eta \mathcal{L}_{\mathcal{V}} \mathcal{L}_{\mathcal{T}}}{2 \mu^2}.
\end{equation}
\end{corollary}

\textbf{Discussion about Assumptions} Note that:
\mycoloredbox{}{\begin{enumerate}
    \item \citep{loizou2020stochastic} which first proposed SHGD assumed independent mini-batches;
    \item Lipschitzianity on $\nabla \mathcal{H}_{\gamma^1,\gamma^2}$ and Error Bound on $F$ imply Eq. \eqref{eq:Scaling_Vt};
    \item Bounded variance on $\nabla \mathcal{H}_{\gamma^1,\gamma^2} $ implies Eq. \eqref{eq:Bounded_Vt}.
\end{enumerate}}

\paragraph{Concrete Examples}
We analyze Bilinear Games for which we can provide explicit formulas for the results presented above. We focus on:  $f(x,y) = x^{\top} \E_{\xi} \left[ \mathbf{\Lambda}_{\xi} \right] y$  and $f(x,y) = x^{\top} \mathbf{\Lambda} y$ where $\mathbf{\Lambda}$ and $\mathbf{\Lambda}_{\xi}$ are square, diagonal, and positive semidefinite matrices.

\begin{proposition}\label{prop:SHGD_Convergence_PIBG_NoM_Insights}
For $f(x,y) = x^{\top} \E_{\xi} \left[ \mathbf{\Lambda}_{\xi} \right] y$, Eq. \eqref{eq:Scaling_Vt} holds and we have
\begin{equation}
    \frac{\E \left[\lVert Z_t \rVert^2 \right]}{2} = \sum_{i=1}^{d} \frac{\lVert Z^i_0 \rVert^2}{2} e^{- \left(2 \lambda^2_i - \eta \sigma_i^2 \left(2 \lambda_i^2 + \sigma_i^2\right) \right) t}.
\end{equation}
In particular, $\frac{\E \left[\lVert Z_t \rVert^2 \right]}{2} \overset{t \rightarrow \infty}{=} 0 $ if $\eta < \frac{2 \lambda^2_i}{\sigma_i^2 \left(2 \lambda_1^2 + \sigma_i^2\right)}, \quad \forall i$.
\end{proposition}
In this case, the noise structure is such that $v_t$ scales like $\E \left[ H_t \right]$. Thus, $\E \left[ H_t\right]$ exponentially decays to $0$.

\begin{proposition}\label{prop:SHGD_Convergence_PIBG_NoG_Insights_NoSched}
For $f(x,y) = x^{\top}\mathbf{\Lambda} y$ and covariance noise $\Sigma:= \diag(\sigma_1, \cdots, \sigma_d)$,  Eq. \ref{eq:Bounded_Vt} holds and 
\begin{equation}
     \frac{\E \left[\lVert Z_t \rVert^2 \right]}{2} = \sum_{i=1}^{d}  \frac{\lVert Z^i_0 \rVert^2}{2} e^{-2 \lambda^2_i t} + \frac{\eta \sigma_i^2}{2}  \left( 1- e^{-2 \lambda^2_i t} \right),
\end{equation}
which implies that $\frac{\E \left[\lVert Z_t \rVert^2 \right]}{2} \overset{t \rightarrow \infty}{=} \frac{\eta}{2} \sum_{i=1}^{d} \sigma_i^2 >0$.
\end{proposition}
In this case, $v_t$ is bounded, meaning that $\E \left[ H_t\right]$ reaches an asymptotic suboptimality exponentially fast.

Now we provide sufficient and necessary conditions to craft stepsize schedulers that recover convergence.
\begin{proposition}\label{prop:SHGD_Convergence_PIBG_NoG_Insights_Sched}
Under the assumptions of Prop. \ref{prop:SHGD_Convergence_PIBG_NoG_Insights_NoSched}, for any stepsize scheduler $\eta_t$, $\frac{\E \left[\lVert Z_t \rVert^2 \right]}{2}$ is equal to
\begin{equation*}
     \sum_{i=1}^{d} e^{-2 \lambda^2_i \int_0^t \eta_s ds} \left(\frac{\lVert Z^i_0 \rVert^2}{2} + \eta \sigma_i^2 \lambda_i^2 \int_0^t e^{2 \lambda^2_i \int_0^s\eta_r dr} \eta^2_s ds\right).
\end{equation*}
Therefore,
\begin{equation}
    \frac{\E \left[\lVert Z_t \rVert^2 \right]}{2} \overset{t \rightarrow \infty}{\rightarrow}  0 \iff  \int_0^\infty \eta_s ds = \infty \text{ and } \lim_{t \rightarrow \infty} \eta_t = 0.
\end{equation}
Among other possible choices of $\eta_t$,
\begin{equation*}
    \eta_t = \frac{1}{(t+1)^{\gamma}} \implies \frac{\E \left[\lVert Z_t \rVert^2 \right]}{2} \rightarrow 0,  \text{ for }\gamma \in \{0.5, 1\}.
\end{equation*}
\end{proposition}

\vspace{-0.2cm}

\begin{figure*}%
    \centering
    \subfloat{{\includegraphics[width=0.47\linewidth]{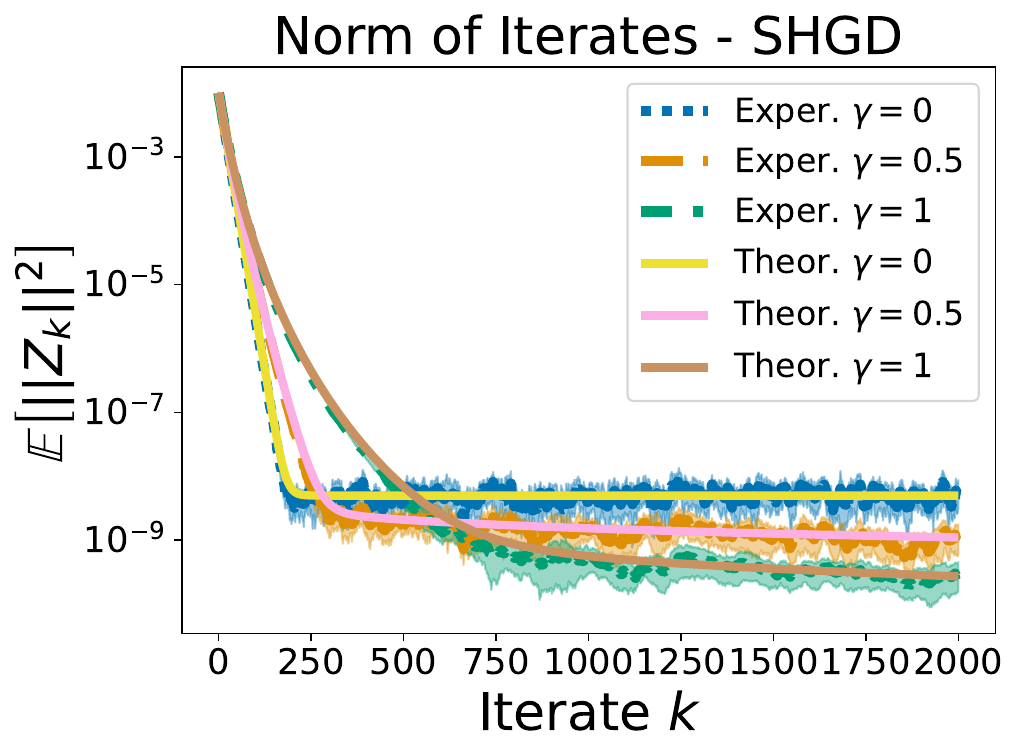} }}%
    \subfloat{{\includegraphics[width=0.47\linewidth]{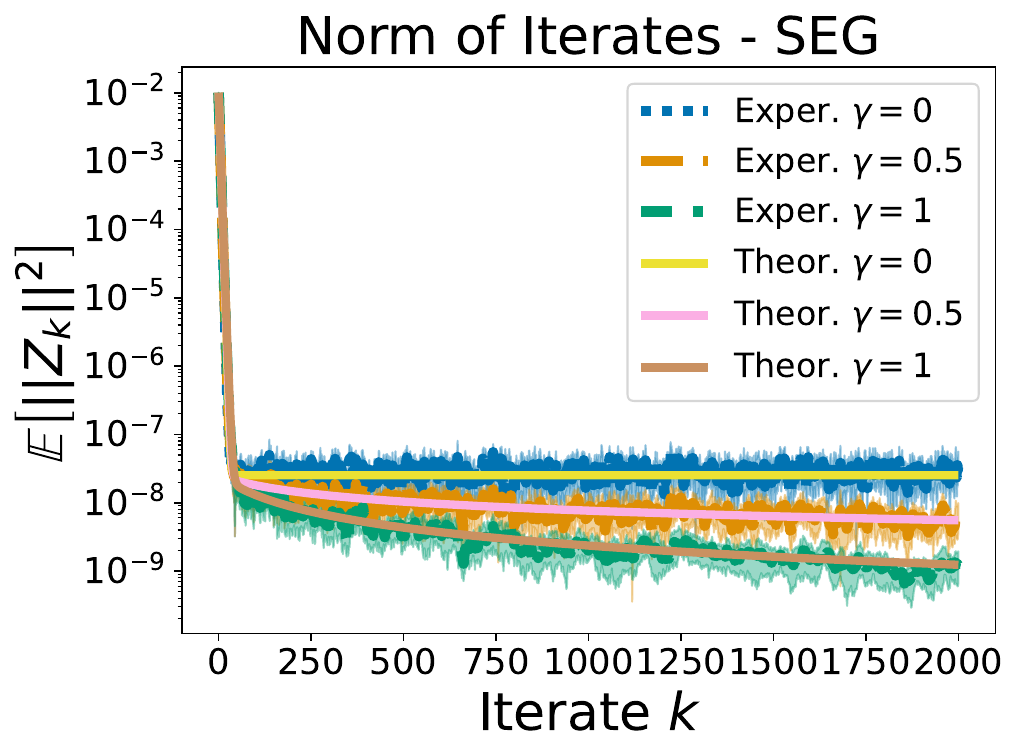} }}
    \caption{Empirical validation of Prop. \ref{prop:SHGD_Convergence_PIBG_NoG_Insights_NoSched} and Prop. \ref{prop:SHGD_Convergence_PIBG_NoG_Insights_Sched} (Left); Prop. \ref{prop:SEG_Convergence_PIBG_NoG_Insights_NoSched} and Prop. \ref{prop:SEG_Convergence_PIBG_NoG_Insights_Sched} (Right): The dynamics of $\E \left[\lVert Z_t \rVert^2 \right]$ averaged across $5$ runs perfectly matches that prescribed by our results for all schedulers. Both for SEG and SHGD, $\eta=0.01$, while $\rho=1$.}%
    \label{fig:Schedulers}%
\end{figure*}
\vspace{-0.1cm}
\subsection{SEG}
Let $Z_t$ be the solution of the SEG SDE, $\Sigma^{\seg}_t= \Sigma^{\seg}(Z_t)$, $H_t = \E_\bgamma[\mathcal{H}_{\bgamma}(Z_t)]$, and $F^{\seg}_t=F^{\seg}(Z_t)$. Then,
$$\E\left[\dot H_t\right] = - \E\left[\nabla H_t^{\top} F^{\seg}_t\right] + \frac{\eta}{2} \tr\left(\E\left[\Sigma^{\seg}_t\nabla^2 H_t\right]\right).$$
Once again, we have to study how the pulling and repulsive forces balance each other in order to dissect the convergence behavior of SEG.

\begin{theorem}[SEG General Convergence] \label{theorem:SEG_GenConv_Insights}Consider the solution $Z_t$ of the SEG SDE with $\gamma^1\ne\gamma^2$.
Let $v_t := \E \left[\E_{\bgamma}  \left[\lVert  F^{\seg}\left( Z_t \right) - F^{\seg}_{\bgamma}\left( Z_t \right) \rVert_2^2 \right]\right]$ measure the error in $ F^{\seg}$, in expectation over the whole randomness up to time $t$. Suppose that:
\begin{enumerate}
    \item The smallest eigenvalue (in absolute value) $\mu_{\rho}$ of $\mathbf{M}$ is non-zero, where $\mathbf{M} = \diag(\mathbf{M}_{1,1},\mathbf{M}_{2,2})$, with $\mathbf{M}_{1,1}:=  \nabla^2 f_{xx}+\rho \left( \nabla^2 f_{xy} \nabla^2 f_{xy}^T  - \nabla^2 f_{xx}^2  \right)$, and $\mathbf{M}_{2,2}:= -\nabla^2 f_{yy}+\rho \left( \nabla^2 f_{xy} \nabla^2 f_{xy}^T  - \nabla^2 f_{yy}^2  \right)$;
    \item $\lVert \nabla^2 H(z)\rVert_{\text{op}} < \mathcal{L}_{\mathcal{T}}$, for all $z\in\R^{2d}$.
\end{enumerate}
Then,
\begin{equation}
    \E \left[ H_t \right] \leq  e^{-2\mu_{\rho}^2t} \left[ H_0 + \frac{\eta\mathcal{L}_{\mathcal{T}}}{2} \int_0^t v_s  e^{2\mu_{\rho}^2 s} ds \right].
\end{equation}
\end{theorem}
\begin{corollary}\label{cor:SEG_Conv_Insights}
   The very same result as Corollary \ref{cor:SHGD_Conv_Insights} holds where we substitute $\mu$ with $\mu_{\rho}$. 
\end{corollary}

\textbf{Discussion about Assumptions} Note that:

\mycoloredbox{}{
\begin{enumerate}
    \item Independent mini-batches are used in Indep.-Sample SEG \citep{du2022optimal,gorbunov2022stochastic}, and are not a necessary condition;
    \item $\kappa_1$-Lipschitzianity on $F_{\gamma^1}$, $\kappa_2$-Lipschitzianity on $\nabla F_{\gamma^1} F_{\gamma^2}$, and $\beta$-Error Bound on $F$, imply $v_t \leq \beta^2 (\kappa_1^2 + \rho^2 \kappa_2^2) \E \left[ H_t \right]$;
    \item $\sigma_1$-Bounded variance on $F_{\gamma^1}$ and $\sigma_2$-Bounded variance on $\nabla F_{\gamma^1} F_{\gamma^2}$, imply $v_t \leq \sigma_1^2 + \rho^2 \sigma_2^2$.
\end{enumerate}
}
Interestingly, we notice that increasing $\rho$ might be detrimental as it could possibly lead to divergence.
\paragraph{Concrete Examples}
\begin{proposition} \label{prop:SEG_Convergence_PIBG_NoM_Insights}
For $f(x,y) = x^{\top} \E_{\xi} \left[ \mathbf{\Lambda}_{\xi} \right] y$, we have
\begin{equation}\label{eq:SEG_StochPIBG_ExpDecay}
    \frac{\E \left[\lVert Z_t \rVert^2 \right]}{2} = \sum_{i=1}^{d} \frac{\lVert Z^i_0 \rVert^2}{2} e^{- \left(2 \rho \lambda^2_i - \eta \sigma_i^2 \left( 1 + \rho^2\left(2 \lambda_i^2 + \sigma_i^2\right)\right) \right) t}.
\end{equation}

In particular, $\frac{\E \left[\lVert Z_t \rVert^2 \right]}{2} \overset{t \rightarrow \infty}{=} 0 $ if $, \forall i \in \{i, \cdots, d\}$
\begin{equation} \label{eq:ConvConditionSEG}
    2 \rho \lambda^2_i - \eta \sigma_i^2 \left( 1 + \rho^2\left(2 \lambda_i^2 + \sigma_i^2\right)\right)>0.
\end{equation}
\end{proposition}

\begin{proposition}\label{prop:SEG_Convergence_PIBG_NoG_Insights_NoSched}
For $f(x,y) = x^{\top} \mathbf{\Lambda} y$ and covariance noise $\Sigma:= \diag(\sigma_1, \cdots, \sigma_d)$, $\frac{\E \left[\lVert Z_t \rVert^2 \right]}{2}$ is equal to
\begin{equation}\label{eq:SEG_PIBG_ExpDecay}
     \sum_{i=1}^{d}  \frac{\lVert Z^i_0 \rVert^2}{2} e^{-2\rho \lambda^2_i t} + \frac{\eta \sigma_i^2}{2} \frac{1 + \rho^2 \lambda_i^2}{\rho \lambda_i^2} \left( 1- e^{-2 \rho\lambda^2_i t} \right),
\end{equation}
which implies that
\begin{equation}\label{eq:SuboptSEG}
    \frac{\E \left[\lVert Z_t \rVert^2 \right]}{2} \overset{t \rightarrow \infty}{=} \frac{\eta}{2} \sum_{i=1}^{d} \sigma_i^2 \frac{1 + \rho^2 \lambda_i^2}{\rho \lambda_i^2}>0.
\end{equation}
\end{proposition}
We derive necessary and sufficient conditions for stepsize schedulers to remediate the convergence deficiency.
\begin{proposition} \label{prop:SEG_Convergence_PIBG_NoG_Insights_Sched}
Under the assumptions of Prop. \ref{prop:SEG_Convergence_PIBG_NoG_Insights_NoSched}, for any stepsize scheduler $\eta_t$ and $\rho_t$, $\frac{\E \left[\lVert Z_t \rVert^2 \right]}{2}$ is equal to
\begin{align}
     & \sum_{i=1}^{d} e^{-2 \lambda^2_i \rho \int_0^t \eta_s \rho_s ds} \left(\frac{\lVert Z^i_0 \rVert^2}{2} \right. \\
     & \left.+ \eta \sigma_i^2  \int_0^t e^{2 \lambda^2_i \rho\int_0^s\eta_r \rho_r dr} \eta^2_s (1 + \lambda_i^2 \rho^2 \rho_s^2 ) ds\right). \nonumber
\end{align}
Therefore, $\frac{\E \left[\lVert Z_t \rVert^2 \right]}{2} \overset{t \rightarrow \infty}{\rightarrow}  0$ if and only if
\begin{equation}
     \int_0^\infty \eta_s \rho_s ds = \infty \text{ and } \lim_{t \rightarrow \infty} \eta_t \rho_t = \lim_{t \rightarrow \infty} \frac{\eta_t}{\rho_t} = 0.
\end{equation}
In particular, when $\rho_t=1$,
\begin{equation*}
    \eta_t = \frac{1}{(t+1)^{\gamma}} \implies \frac{\E \left[\lVert Z_t \rVert^2 \right]}{2} \rightarrow 0,  \text{ for }\gamma \in \{0.5, 1\}.
\end{equation*}
\end{proposition}
The left of Figure \ref{fig:Schedulers} shows the empirical validation of Prop. \ref{prop:SHGD_Convergence_PIBG_NoG_Insights_NoSched} and Prop. \ref{prop:SHGD_Convergence_PIBG_NoG_Insights_Sched} while its right side shows that of Prop. \ref{prop:SEG_Convergence_PIBG_NoG_Insights_NoSched} and Prop. \ref{prop:SEG_Convergence_PIBG_NoG_Insights_Sched}. Figure \ref{fig:NoM} shows the same for Prop. \ref{prop:SHGD_Convergence_PIBG_NoM_Insights} and Prop. \ref{prop:SEG_Convergence_PIBG_NoM_Insights}. More details are available in Appendix \ref{app:Experiments}.
\mycoloredbox{}{\textbf{Conclusion:}
\begin{enumerate}
    \item If the uncertainty $v_t$ is well behaved as in Prop. \ref{prop:SHGD_Convergence_PIBG_NoM_Insights} and Prop.  \ref{prop:SEG_Convergence_PIBG_NoM_Insights}, the Hamiltonian decays exponentially to $0$;
    \item When $v_t$ is constant as in Prop. \ref{prop:SHGD_Convergence_PIBG_NoG_Insights_NoSched} and Prop. \ref{prop:SEG_Convergence_PIBG_NoG_Insights_NoSched}, both algorithms exponentially reach a level of suboptimality that depends on the curvature of the landscape (and on $\rho$ for SEG);
    \item Prop. \ref{prop:SHGD_Convergence_PIBG_NoG_Insights_Sched} and Prop. \ref{prop:SEG_Convergence_PIBG_NoG_Insights_Sched} provide a recipe to craft schedulers that recover convergence. We provide examples of such necessary and sufficient conditions;
    \item Eq. \eqref{eq:SEG_StochPIBG_ExpDecay} and Eq. \eqref{eq:SEG_PIBG_ExpDecay} clearly show that large $\rho$ speeds up the convergence. However, this might violate Eq. \eqref{eq:ConvConditionSEG} and increase the suboptimality in Eq. \eqref{eq:SuboptSEG}.
\end{enumerate}
}

\vspace{-0.3cm}

\section{QUADRATIC GAMES: EXACT DYNAMICS EXPRESSION} \label{sec:Inter2}

\vspace{-0.2cm}

In this section, we derive the exact solution to the SDEs of SEG and SHGD for the Quadratic Games $f(x,y) = \frac{x^{\top} \mathbf{A} x}{2} + x^{\top} \mathbf{\Lambda} y - \frac{y^{\top} \mathbf{A} y}{2}$ where $\mathbf{\Lambda}$ and $\mathbf{A}$ are square, diagonal and positive semidefinite matrices. We notice that if $\mathbf{A}=\mathbf{0}$, these are classic Bilinear Games.

\vspace{-0.2cm}

\subsection{Exact Dynamics - SEG}

\vspace{-0.1cm}

\begin{theorem}[Exact Dynamics of SEG] \label{thm:SEG_Dynamic_Insights_FBG}
Under the assumptions of Corollary \ref{thm:SEG_SDE_Simplified_Insights}, we take the covariance of the noise on the gradients to be $\sigma^2 \mathbf{I}_d$ and have that
    \begin{equation}
    Z_t = \mathbf{\mathbf{\Tilde{E}}}(t) \mathbf{\mathbf{\Tilde{R}}}(t) \left( z + \sqrt{\eta} \sigma \int_{0}^t \mathbf{\mathbf{\Tilde{E}}}(-s) \mathbf{\mathbf{\Tilde{R}}}(-s)  \mathbf{M} d W_s\right),
\end{equation}
    $\mathbf{\mathbf{\Tilde{E}}}(t)= \left[\begin{array}{ll} \mathbf{E}(t) & \mathbf{0}_d  \\  \mathbf{0}_d  & \mathbf{E}(t) \end{array}\right],\mathbf{\mathbf{\Tilde{R}}}(t)=\left[\begin{array}{ll} \mathbf{C}(t) & -\mathbf{S}(t)  \\  \mathbf{S}(t)  & \mathbf{C}(t) \end{array}\right]$, and $\mathbf{M}=\left[\begin{array}{ll} \mathbf{I}_d - \rho \mathbf{A} & - \rho \mathbf{\Lambda}  \\ \rho \mathbf{\Lambda}  & \mathbf{I}_d - \rho \mathbf{A}  \end{array}\right] $,
where
\begin{equation}\label{eq:SEGExpMatrix}
    \mathbf{E}(t):=\diag{\left( e^{\rho \left(a_{1}^{2} - \lambda_1^2\right)t - a_1 t}, \cdots, e^{\rho \left(a_{d}^{2} - \lambda_d^2\right)t - a_d t} \right)},
\end{equation}
\begin{equation}
    \mathbf{C}(t):= \diag{\left( \cos{(\hat{\lambda}_1 t)}, \cdots, \cos{(\hat{\lambda}_d t)}\right)},
\end{equation} 
\begin{equation}
    \mathbf{S}(t):= \diag{\left( \sin{(\hat{\lambda}_1 t)}, \cdots, \sin{(\hat{\lambda}_d t)}\right)},
\end{equation}
and $\hat{\lambda}_i:= \lambda_i (1-2 \rho a_i)$.
If $\rho \left(a_{i}^{2} - \lambda_i^2\right) - a_i <0$:
\begin{enumerate}
    \item $\E \left[ Z_t\right] =\mathbf{\mathbf{\Tilde{E}}}(t) \mathbf{\mathbf{\Tilde{R}}}(t)  z \overset{t \rightarrow \infty}{=} 0$;
    \item The covariance matrix of $Z_t$ is equal to \begin{equation}
    \frac{\eta \sigma^2}{2}  \left[\begin{array}{ll} \mathbf{I}_d - \mathbf{E}(2t) & \mathbf{0}_d  \\  \mathbf{0}_d  &  \mathbf{I}_d - \mathbf{E}(2t) \end{array}\right] \bar{\Sigma} \overset{t \rightarrow \infty}{=} \frac{\eta \sigma^2}{2}  \bar{\Sigma}
\end{equation}
where $\bar{\Sigma}:= \diag(\mathbf{B},\mathbf{B})$ and $\mathbf{B}$ is defined as
\begin{equation}\label{eq:SEGCovMatr}
    \diag{\left(  \frac{(1-\rho a_1)^2 + \rho^2 \lambda_1^2}{a_1 + \rho( \lambda_{1}^{2}-a_1^2)}, \cdots,  \frac{(1-\rho a_d)^2 + \rho^2 \lambda_d^2}{a_d + \rho( \lambda_{d}^{2}-a_d^2)} \right)}.
\end{equation}
\end{enumerate}
\end{theorem}
\vspace{-0.4cm}
\begin{proof}
    Since the SDE is linear, the closed-form formula of the solution $Z_t$ is known. We use the martingale property of Brownian motion to calculate $\E \left[ Z_t\right]$ while that of the second moment uses the Itô Isometry.
\end{proof}
\vspace{-0.4cm}
We verify Eq. \eqref{eq:SEGCovMatr} in Figure \ref{fig:Variance} in Appendix.

\vspace{3.9cm}

\textbf{On the sign of $\rho$ and its magnitude:}

\mycoloredbox{}{
If one can chose $\rho_i$ for each coordinate:
\begin{enumerate}
    \item Eq.~\eqref{eq:SEGExpMatrix} implies that SEG converges only if $\rho_i \left(a_{i}^{2} - \lambda_i^2\right) - a_i <0$, and that $\rho_i \left(a_{i}^{2} - \lambda_i^2\right)<0$ is necessary to be faster than SGDA, meaning that \textbf{negative} $\rho_i$ might be convenient if $a_{i}>\lambda_i$;
    \item If $\rho_i$ has the correct sign, a \textbf{larger} absolute value implies \textbf{faster convergence};
    \item  Eq.~\eqref{eq:SEGCovMatr} implies that the asymptotic variance along the $i$-th coordinate $\mathbf{B}_{i,i}(\rho_i)$ \textbf{explodes} if $|\rho_i|$ is too \textbf{large} or if $ \rho_i \rightarrow \frac{-a_i}{\lambda_i^2-a_i^2}$ ;
    \item $B_{i,i}(\rho_i)$ is a convex function of $\rho_i$ whose minimum is realized at $\rho^{\text{V}}_i=\frac{1}{a_i + \lambda_i}$; However, if $\rho^{\text{V}}_i$ is \textbf{small}, it \textbf{slows down} the convergence.
\end{enumerate}
If one has to choose a single value of $\rho$:
\begin{enumerate}
    \item One has to select it as it will (de)accelerate different coordinates based on its sign;
    \item The trace of $\mathbf{B}$ is a convex function of $\rho$, meaning that there is an optimal $\rho^{*}$ that minimizes it.
\end{enumerate}
}

\begin{figure*}%
    \centering
    \subfloat{{\includegraphics[width=0.47\linewidth]{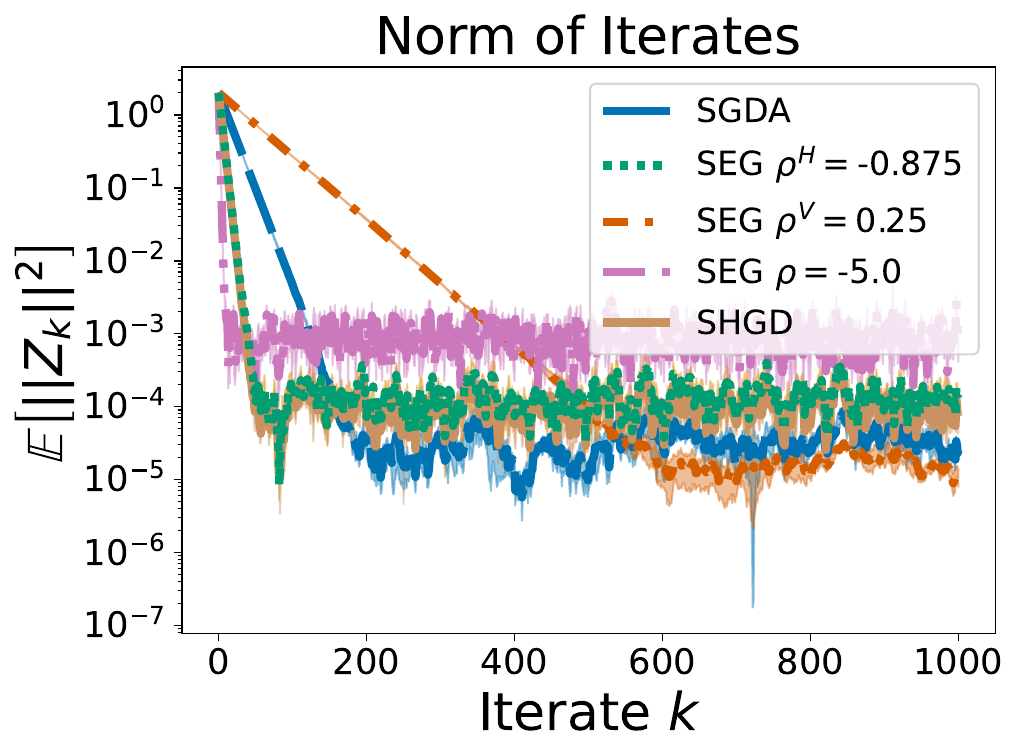} }}%
    \subfloat{{\includegraphics[width=0.47\linewidth]{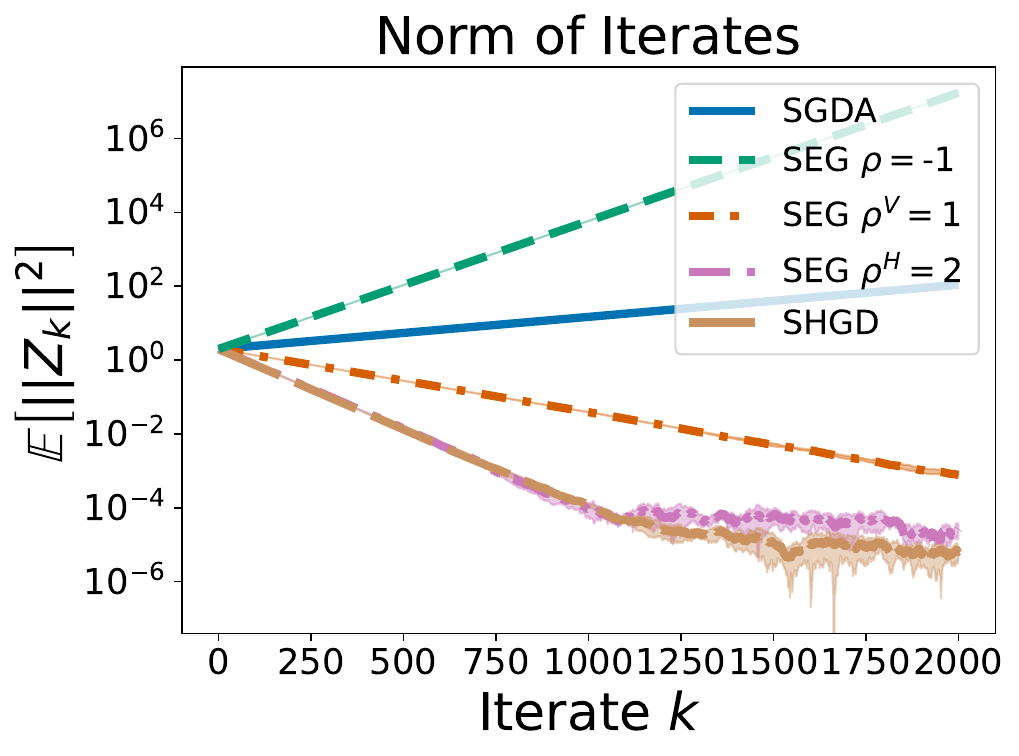} }}%
    \caption{Comparison between SEG and SHGD on Quadratic Games: (Left), $\rho^{V}$ and $\rho^{H}$ meet the designated goals, sometimes \textbf{negative} $\rho$ is desirable as positive ones \textbf{slow down} the convergence. Large $|\rho|$ induces faster convergence which in turn results in larger suboptimality. (Right), negative $\rho$ escapes the \textit{bad saddle} faster than SGDA, positive ones induce convergence, and $\rho^{H}$ matches the decay of SHGD. In both experiments, $\eta=0.01$.}%
    \label{fig:FBG_Comparison_Saddle}%
\end{figure*}

\vspace{-0.2cm}

\subsection{Exact Dynamics - SHGD}
\begin{theorem}[Exact Dynamics of SHGD] \label{thm:SHGD_Dynamic_FBG_Insights}
Under the assumptions of Corollary \ref{thm:SHGD_SDE_Simplified_Insights}, we take the covariance of the noise on the gradients to be equal to $\sigma^2 \mathbf{I}_d$ and have
    \begin{equation}
    Z_t = \mathbf{\mathbf{\Tilde{E}}}(t) \left( z + \sqrt{\eta} \sigma \int_{0}^t \mathbf{\mathbf{\Tilde{E}}}(-s) \mathbf{M} d W_s\right),
\end{equation}
    $\mathbf{\mathbf{\Tilde{E}}}(t)= \left[\begin{array}{ll} \mathbf{E}(t) & \mathbf{0}_d  \\  \mathbf{0}_d  & \mathbf{E}(t) \end{array}\right]$, $\mathbf{M}=\left[\begin{array}{ll} \mathbf{A} &  \mathbf{\Lambda}  \\ \mathbf{\Lambda}  & -\mathbf{A}  \end{array}\right] $,
where 
\begin{equation}
    \mathbf{E}(t):= \diag{\left( e^{-\left(\lambda_{1}^{2}+a_{1}^{2}  \right)t}, \cdots, e^{-\left(\lambda_{d}^{2}+a_{d}^{2}  \right)t} \right)}.
\end{equation}

\vspace{-0.2cm}

In particular, we have that

\vspace{-0.2cm}

\begin{enumerate}
    \item $\E \left[ Z_t\right] = \mathbf{\mathbf{\Tilde{E}}}(t) z \overset{t \rightarrow \infty}{=} 0$;

\vspace{-0.1cm}
    
    \item The covariance matrix of $Z_t$ is equal to \begin{equation}
    \eta \frac{\sigma^2 }{2}\left[\begin{array}{ll} \mathbf{I}_d - \mathbf{E}(2t) & \mathbf{0}_d  \\  \mathbf{0}_d  &  \mathbf{I}_d - \mathbf{E}(2t) \end{array}\right] \bar{\Sigma} \overset{t \rightarrow \infty}{=} \frac{\eta \sigma^2}{2}  \bar{\Sigma},
\end{equation}
where $\bar{\Sigma} := \diag(\mathbf{I}_d,\mathbf{I}_d)$.
\end{enumerate}
\end{theorem}

\vspace{3.5cm}

\textbf{SEG vs SHGD: Insights}
\mycoloredbox{}{
\begin{enumerate}
    \item The curvature influences the convergence speed of SHGD, but differently than for SEG, it does not affect the asymptotic covariance matrix;
    \item If $(\lambda_i^2-a_i^2)\rho^{H}_i > a_i^2 + \lambda_i^2 - a_i$, SEG exponentially decays \textbf{faster} than SHGD. However, this means that SEG has a \textbf{larger} asymptotic variance;
    \item If $\rho^{V}_i = \frac{1}{\lambda_i + a_i}$, SEG attains its lowest asymptotic variance $\frac{\eta \sigma^2\lambda_i}{2(a_i + \lambda_i)^2}$, which is smaller than $\frac{\eta \sigma^2}{2}$ reached by SHGD only if $a_i^2+\lambda_i^2-\lambda_i>0$;
    \item If $\lambda_{i}^{2}+a_{i}^{2} \sim 0$, SHGD is essentially stuck and SEG is intrinsically faster;
    \item If $a_i<0$, $Z=0$ is a \textit{bad saddle}. While SEG can escape it, SHGD is pulled towards it.
\end{enumerate}}

\vspace{-0.2cm}

\paragraph{Conclusion} Our results allowed us to carry out a 1-to-1 comparison of the two methods, shedding light on the role of $\rho$ in influencing the behavior of SEG w.r.t. SHGD. The interaction between the curvature and the noise implies that selecting $\rho_i$ is a trade-off between the speed of convergence and the asymptotic variance. There is no clear winner between SEG and SHGD, as they are preferable for different landscapes. Figure \ref{fig:FBG_Comparison_Saddle} shows experiments that support the latter claim.

\vspace{-0.3cm}

\section{CONCLUSIONS}
\label{sec:conclusion}

\vspace{-0.4cm}

We have presented and analyzed the first formal SDE models for SGDA, SEG, and SHGD. We have shown the \textit{implicit} regularization in SEG in contrast with the \textit{explicit} use of curvature-based information in SHGD, which leads to different noise structures and asymptotic suboptimality.

\vspace{-1mm}
Furthermore, we have used these SDEs to fully characterize the evolution of the Hamiltonian under the dynamics of these algorithms in useful scenarios. We derived convergence bounds and established conditions under which stepsize schedulers guarantee convergence.

Finally, our comparative analysis of SEG and SHGD for Quadratic Games sheds light on the role of $\rho$, revealing a trade-off between convergence speed and suboptimality. We also presented the first theoretical and experimental evidence that, depending on the curvature of the loss, the optimal $\rho$ might be negative.

\vspace{-3mm}
\paragraph{Outlook.} Our framework offers a unified and structured analytical approach rooted in Itô calculus to study minimax optimizers. Our approach not only facilitates the derivation of novel insights but also enables straightforward comparisons between discrete algorithms. We believe our findings provide a foundation for future research, which may include the analysis of momentum, adaptive methods, derivation of scaling laws, and the design of new optimizers.

\vspace{-3mm}
\paragraph{Limitations.} Modeling discrete-time algorithms using SDEs hinges on Assumption \ref{ass:regularity_f_Insights}. As documented \citep{Li2021validity}, large values of the stepsize $\eta$ or the absence of specific conditions on $\nabla f$ and the noise covariance matrix can result in an approximation failure. While these shortcomings can be mitigated by increasing the order of the weak approximation,  our perspective aligns with the idea that SDEs should primarily serve as simplification tools --- to solidify our intuition --- and might not gain substantial benefits from additional complexity.

\section{ACKNOWLEDGEMENTS}
We thank the reviewers for their feedback which greatly helped us improve this manuscript. 

Additionally, we thank Dr. Junchi Yang and Prof. Dr. Niao He for the insightful discussions.

Enea Monzio Compagnoni and Aurelien Lucchi acknowledge the financial support of the Swiss National Foundation, SNF grant No 207392.
Antonio Orvieto acknowledges the financial support of the Hector Foundation.
Frank Norbert Proske acknowledges the financial support of the Norwegian Research Council (project No 274410) and MSCA4Ukraine (project No 101101923).

\bibliography{biblio}

\begin{thebibliography}{}

\bibitem[An et~al., 2020]{an2020stochastic}
An, J., Lu, J., and Ying, L. (2020).
\newblock Stochastic modified equations for the asynchronous stochastic
  gradient descent.
\newblock {\em Information and Inference: A Journal of the IMA}, 9(4):851--873.

\bibitem[Balduzzi et~al., 2018]{balduzzi2018mechanics}
Balduzzi, D., Racaniere, S., Martens, J., Foerster, J., Tuyls, K., and Graepel,
  T. (2018).
\newblock The mechanics of n-player differentiable games.
\newblock In {\em International Conference on Machine Learning}, pages
  354--363. PMLR.

\bibitem[Chaudhari and Soatto, 2018]{chaudhari2018stochastic}
Chaudhari, P. and Soatto, S. (2018).
\newblock Stochastic gradient descent performs variational inference, converges
  to limit cycles for deep networks.
\newblock In {\em 2018 Information Theory and Applications Workshop (ITA)},
  pages 1--10. IEEE.

\bibitem[Chavdarova et~al., 2019]{chavdarova2019reducing}
Chavdarova, T., Gidel, G., Fleuret, F., and Lacoste-Julien, S. (2019).
\newblock Reducing noise in gan training with variance reduced extragradient.
\newblock {\em Advances in Neural Information Processing Systems}, 32.

\bibitem[Chavdarova et~al., 2022]{chavdarova2022continuous}
Chavdarova, T., Hsieh, Y.-P., and Jordan, M.~I. (2022).
\newblock Continuous-time analysis for variational inequalities: An overview
  and desiderata.
\newblock {\em arXiv preprint arXiv:2207.07105}.

\bibitem[Chavdarova et~al., 2023]{chavdarova2021last}
Chavdarova, T., Jordan, M.~I., and Zampetakis, M. (2023).
\newblock Last-iterate convergence of saddle point optimizers via
  high-resolution differential equations.
\newblock In {\em Minimax Theory and its Applications 08 (2023), No. 2}, pages
  333--380. Heldermann Verlag.

\bibitem[Chen et~al., 2015]{chen2015convergence}
Chen, C., Ding, N., and Carin, L. (2015).
\newblock On the convergence of stochastic gradient mcmc algorithms with
  high-order integrators.
\newblock {\em Advances in neural information processing systems}, 28.

\bibitem[Chen and Rockafellar, 1997]{chen1997convergence}
Chen, G.~H. and Rockafellar, R.~T. (1997).
\newblock Convergence rates in forward--backward splitting.
\newblock {\em SIAM Journal on Optimization}, 7(2):421--444.

\bibitem[Compagnoni et~al., 2023]{compagnoni2023sde}
Compagnoni, E.~M., Biggio, L., Orvieto, A., Proske, F.~N., Kersting, H., and
  Lucchi, A. (2023).
\newblock An sde for modeling sam: Theory and insights.
\newblock In {\em International Conference on Machine Learning}, pages
  25209--25253. PMLR.

\bibitem[Du et~al., 2022]{du2022optimal}
Du, S.~S., Gidel, G., Jordan, M.~I., and Li, C.~J. (2022).
\newblock Optimal extragradient-based bilinearly-coupled saddle-point
  optimization.
\newblock {\em arXiv preprint arXiv:2206.08573}.

\bibitem[Gidel et~al., 2019]{gidel2018variational}
Gidel, G., Berard, H., Vignoud, G., Vincent, P., and Lacoste-Julien, S. (2019).
\newblock A variational inequality perspective on generative adversarial
  networks.
\newblock {\em ICLR}.

\bibitem[Goodfellow et~al., 2016]{goodfellow2016deep}
Goodfellow, I., Bengio, Y., Courville, A., and Bengio, Y. (2016).
\newblock {\em Deep learning}, volume~1.
\newblock MIT Press.

\bibitem[Gorbunov et~al., 2022]{gorbunov2022stochastic}
Gorbunov, E., Berard, H., Gidel, G., and Loizou, N. (2022).
\newblock Stochastic extragradient: General analysis and improved rates.
\newblock In {\em International Conference on Artificial Intelligence and
  Statistics}, pages 7865--7901. PMLR.

\bibitem[He et~al., 2018]{ijcai2018p307}
He, L., Meng, Q., Chen, W., Ma, Z.-M., and Liu, T.-Y. (2018).
\newblock Differential equations for modeling asynchronous algorithms.
\newblock In {\em Proceedings of the 27th International Joint Conference on
  Artificial Intelligence}, IJCAI'18, page 2220–2226. AAAI Press.

\bibitem[Hsieh et~al., 2019]{hsieh2019convergence}
Hsieh, Y.-G., Iutzeler, F., Malick, J., and Mertikopoulos, P. (2019).
\newblock On the convergence of single-call stochastic extra-gradient methods.
\newblock {\em Advances in Neural Information Processing Systems}, 32.

\bibitem[Hsieh et~al., 2020]{hsieh2020explore}
Hsieh, Y.-G., Iutzeler, F., Malick, J., and Mertikopoulos, P. (2020).
\newblock Explore aggressively, update conservatively: Stochastic extragradient
  methods with variable stepsize scaling.
\newblock {\em Advances in Neural Information Processing Systems},
  33:16223--16234.

\bibitem[Hsieh et~al., 2021]{hsieh2021limits}
Hsieh, Y.-P., Mertikopoulos, P., and Cevher, V. (2021).
\newblock The limits of min-max optimization algorithms: Convergence to
  spurious non-critical sets.
\newblock In {\em International Conference on Machine Learning}, pages
  4337--4348. PMLR.

\bibitem[Jastrzebski et~al., 2018]{jastrzkebski2017three}
Jastrzebski, S., Kenton, Z., Arpit, D., Ballas, N., Fischer, A., Bengio, Y.,
  and Storkey, A. (2018).
\newblock Three factors influencing minima in sgd.
\newblock {\em ICANN 2018}.

\bibitem[Juditsky et~al., 2011]{juditsky2011solving}
Juditsky, A., Nemirovski, A., and Tauvel, C. (2011).
\newblock Solving variational inequalities with stochastic mirror-prox
  algorithm.
\newblock {\em Stochastic Systems}, 1(1):17--58.

\bibitem[Korpelevich, 1976]{korpelevich1976extragradient}
Korpelevich, G.~M. (1976).
\newblock The extragradient method for finding saddle points and other
  problems.
\newblock {\em Matecon}, 12:747--756.

\bibitem[Kushner and Yin, 2003]{kushner2003stochastic}
Kushner, H. and Yin, G.~G. (2003).
\newblock {\em Stochastic approximation and recursive algorithms and
  applications}, volume~35.
\newblock Springer Science \& Business Media.

\bibitem[Li et~al., 2022]{li2022convergence}
Li, C.~J., Yu, Y., Loizou, N., Gidel, G., Ma, Y., Le~Roux, N., and Jordan, M.
  (2022).
\newblock On the convergence of stochastic extragradient for bilinear games
  using restarted iteration averaging.
\newblock In {\em International Conference on Artificial Intelligence and
  Statistics}, pages 9793--9826. PMLR.

\bibitem[Li et~al., 2017]{li2017stochastic}
Li, Q., Tai, C., and Weinan, E. (2017).
\newblock Stochastic modified equations and adaptive stochastic gradient
  algorithms.
\newblock In {\em International Conference on Machine Learning}, pages
  2101--2110. PMLR.

\bibitem[Li et~al., 2019]{li2019stochastic}
Li, Q., Tai, C., and Weinan, E. (2019).
\newblock Stochastic modified equations and dynamics of stochastic gradient
  algorithms i: Mathematical foundations.
\newblock {\em The Journal of Machine Learning Research}, 20(1):1474--1520.

\bibitem[Li et~al., 2021]{Li2021validity}
Li, Z., Malladi, S., and Arora, S. (2021).
\newblock On the validity of modeling {SGD} with stochastic differential
  equations ({SDE}s).
\newblock In Beygelzimer, A., Dauphin, Y., Liang, P., and Vaughan, J.~W.,
  editors, {\em Advances in Neural Information Processing Systems}.

\bibitem[Ljung et~al., 2012]{ljung2012stochastic}
Ljung, L., Pflug, G., and Walk, H. (2012).
\newblock {\em Stochastic approximation and optimization of random systems},
  volume~17.
\newblock Birkh{\"a}user.

\bibitem[Loizou et~al., 2021]{loizou2021stochastic}
Loizou, N., Berard, H., Gidel, G., Mitliagkas, I., and Lacoste-Julien, S.
  (2021).
\newblock Stochastic gradient descent-ascent and consensus optimization for
  smooth games: Convergence analysis under expected co-coercivity.
\newblock {\em Advances in Neural Information Processing Systems},
  34:19095--19108.

\bibitem[Loizou et~al., 2020]{loizou2020stochastic}
Loizou, N., Berard, H., Jolicoeur-Martineau, A., Vincent, P., Lacoste-Julien,
  S., and Mitliagkas, I. (2020).
\newblock Stochastic hamiltonian gradient methods for smooth games.
\newblock In {\em International Conference on Machine Learning}, pages
  6370--6381. PMLR.

\bibitem[Lu, 2022]{lu2022sr}
Lu, H. (2022).
\newblock An o (sr)-resolution ode framework for understanding discrete-time
  algorithms and applications to the linear convergence of minimax problems.
\newblock {\em Mathematical Programming}, 194(1-2):1061--1112.

\bibitem[Malladi et~al., 2022]{Malladi2022AdamSDE}
Malladi, S., Lyu, K., Panigrahi, A., and Arora, S. (2022).
\newblock On the {SDEs} and scaling rules for adaptive gradient algorithms.
\newblock In {\em Advances in Neural Information Processing Systems}.

\bibitem[Mandt et~al., 2015]{mandt2015continuous}
Mandt, S., Hoffman, M.~D., Blei, D.~M., et~al. (2015).
\newblock Continuous-time limit of stochastic gradient descent revisited.
\newblock {\em NIPS-2015}.

\bibitem[Mao, 2007]{mao2007stochastic}
Mao, X. (2007).
\newblock {\em Stochastic differential equations and applications}.
\newblock Elsevier.

\bibitem[Mil’shtein, 1986]{mil1986weak}
Mil’shtein, G. (1986).
\newblock Weak approximation of solutions of systems of stochastic differential
  equations.
\newblock {\em Theory of Probability \& Its Applications}, 30(4):750--766.

\bibitem[Mishchenko et~al., 2020]{mishchenko2020revisiting}
Mishchenko, K., Kovalev, D., Shulgin, E., Richt{\'a}rik, P., and Malitsky, Y.
  (2020).
\newblock Revisiting stochastic extragradient.
\newblock In {\em International Conference on Artificial Intelligence and
  Statistics}, pages 4573--4582. PMLR.

\bibitem[Nemirovski et~al., 2009]{nemirovski2009robust}
Nemirovski, A., Juditsky, A., Lan, G., and Shapiro, A. (2009).
\newblock Robust stochastic approximation approach to stochastic programming.
\newblock {\em SIAM Journal on optimization}, 19(4):1574--1609.

\bibitem[Noor, 2003]{noor2003new}
Noor, M.~A. (2003).
\newblock New extragradient-type methods for general variational inequalities.
\newblock {\em Journal of Mathematical Analysis and Applications},
  277(2):379--394.

\bibitem[{\O}ksendal, 1990]{oksendal1990stochastic}
{\O}ksendal, B. (1990).
\newblock When is a stochastic integral a time change of a diffusion?
\newblock {\em Journal of theoretical probability}, 3(2):207--226.

\bibitem[Orvieto and Lucchi, 2019]{orvieto2019continuous}
Orvieto, A. and Lucchi, A. (2019).
\newblock Continuous-time models for stochastic optimization algorithms.
\newblock {\em Advances in Neural Information Processing Systems}, 32.

\bibitem[Ryu et~al., 2019]{ryu2019ode}
Ryu, E.~K., Yuan, K., and Yin, W. (2019).
\newblock Ode analysis of stochastic gradient methods with optimism and
  anchoring for minimax problems.
\newblock {\em arXiv preprint arXiv:1905.10899}.

\bibitem[Su et~al., 2014]{Su2014nesterov}
Su, W., Boyd, S., and Candes, E. (2014).
\newblock A differential equation for modeling {Nesterov’s} accelerated
  gradient method: Theory and insights.
\newblock In {\em Advances in Neural Information Processing Systems}.

\bibitem[Xie et~al., 2021]{xie2020diffusion}
Xie, Z., Sato, I., and Sugiyama, M. (2021).
\newblock A diffusion theory for deep learning dynamics: Stochastic gradient
  descent exponentially favors flat minima.
\newblock In {\em International Conference on Learning Representations}.

\bibitem[Xu et~al., 2022]{xu2022experimental}
Xu, M. et~al. (2022).
\newblock {\em Experimental Evaluation of Iterative Methods for Games}.
\newblock PhD thesis, Johns Hopkins University.

\bibitem[Zhao et~al., 2022]{zhao2022batch}
Zhao, J., Lucchi, A., Proske, F.~N., Orvieto, A., and Kersting, H. (2022).
\newblock Batch size selection by stochastic optimal control.
\newblock In {\em Has it Trained Yet? NeurIPS 2022 Workshop}.

\bibitem[Zhu et~al., 2019]{zhu2018anisotropic}
Zhu, Z., Wu, J., Yu, B., Wu, L., and Ma, J. (2019).
\newblock The anisotropic noise in stochastic gradient descent: Its behavior of
  escaping from sharp minima and regularization effects.
\newblock {\em ICML 2019}.

\end{thebibliography}


\clearpage
\appendix

\thispagestyle{empty}

\onecolumn

{\Large \textbf{APPENDIX}}

\section{ADDITIONAL RELATED WORKS} \label{app:AddRelWorks}

SGDA is one of the most popular algorithms for solving min-max optimization problems that arise in machine learning. Since it does not converge even on simple landscapes \citep{chen1997convergence,noor2003new,gidel2018variational,loizou2021stochastic}, researchers have derived several advanced extensions such as the Extragradient method \citep{korpelevich1976extragradient} and variants with arbitrary sampling and variance \citep{gorbunov2022stochastic}, as well as alternative optimizers such as (Stochastic) Hamiltonian Gradient Descent \citep{balduzzi2018mechanics,loizou2020stochastic}.

Stochastic ExtraGradient (SEG) is a prominent extension of SGDA that has been studied extensively in recent years. Indeed, many versions have been proposed and studied: \citep{nemirovski2009robust,juditsky2011solving} studied Independent-Samples SEG, while \cite{mishchenko2020revisiting} and \cite{li2022convergence} showed that the average iterate of Same-Sample SEG converges to a neighbor of the optimum. While \citep{chavdarova2019reducing} showed that same-stepsize SEG diverges in the unconstrained monotone case, \cite{mishchenko2020revisiting,hsieh2020explore} focused on two-scale SEG, showcasing how this design choice is crucial by deriving schedulers that guarantee convergence. \cite{hsieh2019convergence} studies the convergence of variations of SEG engineered to mitigate the cost of the extra gradient. Finally, \cite{gorbunov2022stochastic} provides a rich analysis that encompasses several variants of SEG with different choices of stepsizes and sampling techniques, and ends up designing new promising methods. The latter endeavor is key for future research: As highlighted by \citep{hsieh2019convergence}, existing min-max algorithms may be subject to inescapable convergence failures in important cases.

Among other works, we refer the reader interested in previous analyses of bilinear and quadratic games to \citep{hsieh2021limits,li2022convergence, xu2022experimental, chavdarova2021last}: These give a detailed presentation of the behavior of GDA, EG and HGD, and their stochastic versions on such tasks.

We highlight that some convergence conditions and some of the convergence bounds derived in the literature for SEG (see among others \citep{hsieh2020explore,mishchenko2020revisiting,hsieh2019convergence,gorbunov2022stochastic,lu2022sr,li2022convergence}) and SHGD (see \citep{loizou2020stochastic,loizou2021stochastic}) are somehow related to those we present in this paper.


\section{STOCHASTIC CALCULUS}
\label{sec:stoc_cal}
In this section, we summarize some important results in the analysis of Stochastic Differential Equations~\cite{mao2007stochastic,oksendal1990stochastic}. The notation and the results in this section will be used extensively in all proofs in this paper. We assume the reader to have some familiarity with Brownian motion and with the definition of stochastic integral (Ch.~1.4 and 1.5 in~\cite{mao2007stochastic}).

\subsection{Itô's Lemma}
\label{subsec:ito}

We start with some notation: Let $\left(\Omega, \mathcal{F}, \{\mathcal{F}_t\}_{t\ge0}, \mathbb{P}\right)$ be a filtered probability space. We say that an event $E\in\mathcal{F}$ holds almost surely (a.s.) in this space if $\mathbb{P}(E)=1$. We call $\mathcal{L}^p([a,b],\R^d)$, with $p>0$, the family of $\R^d$-valued $\mathcal{F}_t$-adapted processes $\{f_t\}_{a\le t\le b}$ such that $$\int_a^b\|f_t\|^p dt\le\infty.$$ Moreover, we denote by   $\mathcal{M}^p([a,b],\R^d)$, with $p>0$, the family of $\R^d$-valued processes $\{f_t\}_{a\le t\le b}$ in $\mathcal{L}([a,b],\R^d)$ such that $\E\left[\int_a^b\|f_t\|^p dt\right]\le\infty$. 
We will write $h \in \mathcal{L}^p\left(\R_{+},\R^d\right)$, with $p>0$, if $h \in \mathcal{L}^p\left([0,T],\R^d\right)$ for every $T>0$. Similar definitions hold for matrix-valued functions using the Frobenius norm $\|A\| := \sqrt{\sum_{ij}|A_{ij}|^2}$.

Let $W=\{W_t\}_{t\ge0}$ be a one-dimensional Brownian motion defined on our probability space and let $X=\{X_t\}_{t\ge0}$ be an $\mathcal{F}_t$-adapted process taking values on $\R^d$.

\begin{definition}
Let the \textit{drift} be $b \in \mathcal{L}^1\left(\R_{+},\R^d\right)$ and the diffusion term be $\sigma \in \mathcal{L}^2\left(\R_{+},\R^{d\times m}\right)$. $X_t$ is an Itô process if it takes the form
\vspace{-3mm}
$$X_t = x_0 + \int_0^t b_sds + \int_0^t \sigma_s dW_s.$$
We shall say that $X_t$ has the stochastic differential
\begin{equation}
    dX_t = b_t dt + \sigma_t dW_t.
    \label{eq:stoc_diff}
\end{equation}
\label{def:SDE_ito_process}
\end{definition}
\vspace{-5mm}
\vspace{2mm}
\begin{theorem}[Itô's Lemma]
Let $X_t$ be an Itô process with stochastic differential $dX_t = b_t dt + \sigma_t dW_t$. Let $f\left(x,t\right)$ be twice continuously differentiable in $x$ and continuously differentiable in $t$, taking values in $\R$. Then $f(X_t,t)$ is again an Itô process with stochastic differential
\begin{equation}
        df(X_t,t) = \partial_t f (X_t, t)) dt + \langle\nabla f(X_t,t), b_t \rangle dt +  \frac{1}{2}\tr\left(\sigma_t\sigma_t^{\top} \nabla^2 f(X_t,t) \right)dt + \langle\nabla f(X_t,t),\sigma_t\rangle dW_t.
\end{equation}

\label{lemma:SDE_ito}
\end{theorem}
\subsection{Stochastic Differential Equations}
\label{subsec:SDE}
Stochastic Differential Equations (SDEs) are equations of the form
\begin{equation*}
    dX_t = b(X_t,t)dt + \sigma(X_t,t) dW_t.
\end{equation*}

First of all, we need to define what it means for a stochastic process $X=\{X_t\}_{t\ge0}$ with values in $\R^d$ to solve an SDE.

\begin{definition}
Let $X_t$ be as above with deterministic initial condition $X_0 = x_0$. Assume $b:\R^d\times[0,T]\to\R^d$ and $\sigma:\R^d\times[0,T]\to\R^{d\times m}$ are Borel measurable; $X_t$ is called a solution to the corresponding SDE if
\begin{enumerate}[leftmargin=*]
    \item $X_t$ is continuous and $\mathcal{F}_t$-adapted;
    \item $b \in \mathcal{L}^1\left([0,T],\R^d\right)$;
    \item $\sigma \in \mathcal{L}^2\left([0,T],\R^{d\times m}\right)$;
    \item For every $t \in [0,T]$
    $$X_t = x_0 + \int_0^t b(X_s,s)ds + \int_0^t\sigma(X_s,s) dW(s) \ \ \ a.s.$$
\end{enumerate}
Moreover, the solution $X_t$ is said to be unique if any other solution $X^\star_t$ is such that
$$\mathbb{P}\left\{X_t = X^\star_t, \text{ for all } 0\le t\le T\right\}=1.$$
\label{def:SDE_sol}
\end{definition}
\vspace{-5mm}
Notice that since the solution to an SDE is an Itô process, we can use Itô's Lemma. The following theorem gives a sufficient condition on $b$ and $\sigma$ for the existence of a solution to the corresponding SDE.

\begin{theorem}
Assume that there exist two positive constants $\bar K$ and $K$ such that
\begin{enumerate}
    \item (Global Lipschitz condition) for all $x,y\in \R^d$ and $t\in [0,T]$
    $$\max \{\|b(x,t) - b(y,t)\|^2,  \ \|\sigma(x,t)-\sigma(y,t)\|^2\} \le \bar K \|x-y\|^2;$$
    \item (Linear growth condition) for all $x\in \R^d$ and $t\in [0,T]$
    $$\max\{\|b(x,t)\|^2, \ \|\sigma(x,t)\|^2\}\le K(1+ \|x\|^2).$$
\end{enumerate}
Then, there exists a unique solution $X_t$ to the corresponding SDE, and $X_t\in \mathcal{M}^2([0,T],\R^d).$
\label{thm:SDE_existence_uniqueness_global}
\end{theorem}

\paragraph{Numerical approximation.} Often, SDEs are solved numerically. The simplest algorithm to provide a sample path $(\hat x_k)_{k\ge 0}$ for $X_t$, so that $X_{k\Delta t} \approxeq \hat{x}_k$ for some small $\Delta t$ and for all $k\Delta t\le M$ is called Euler-Maruyama (Algorithm~\ref{algo:EulerMaruryama_SDE}). For more details on this integration method and its approximation properties, the reader can check~\cite{mao2007stochastic}.
\begin{algorithm}
\caption{Euler-Maruyama Integration Method for SDEs}
    \label{algo:EulerMaruryama_SDE}
\begin{algorithmic}
    \INPUT{The drift $b$, the volatility $\sigma$, and the initial condition $x_0$}.
    \STATE Fix a stepsize $\Delta t$;
    \STATE Initialize $\hat x_0 = x_0$;
    \STATE $k=0$;
    \WHILE{$k \le \left\lfloor\frac{T}{\Delta t}\right\rfloor$}
        \STATE Sample some $d$-dimensional Gaussian noise $Z_k\sim\mathcal{N}(0,I_d)$;
        \STATE Compute $\hat x_{k+1} = \hat x_k + \Delta t \ b(\hat x_k,k \Delta t)+ \sqrt{\Delta t} \ \sigma(\hat x_k,k \Delta t) Z_k;$
        \STATE $k=k+1$;
    \ENDWHILE
    \OUTPUT The approximated sample path $(\hat x_k)_{0\le k\le\left\lfloor\frac{T}{\Delta t}\right\rfloor }$.
\end{algorithmic}
\end{algorithm}

\section{THEORETICAL FRAMEWORK - WEAK APPROXIMATION}\label{sec:theor}
In this section, we introduce the theoretical framework used in the paper, together with its assumptions and notations.

First of all, many proofs will use Taylor expansions in powers of $ \eta $. For ease of notation,  we introduce the shorthand that whenever we write $ \mathcal{O}\left(\eta^\alpha\right) $, we mean that there exists a function $ K(z) \in G $ such that the error terms are bounded by $ K(z) \eta^\alpha $. For example, we write
$$
b(z+\eta)=b_0(z)+\eta b_1(z)+\mathcal{O}\left(\eta^2\right)
$$
to mean: there exists $ K \in G $ such that
$$
\left|b(z+\eta)-b_0(z)-\eta b_1(z)\right| \leq K(z) \eta^2 .
$$
Additionally, we introduce the following shorthand:

\begin{itemize}
\item A multi-index is $\alpha=\left(\alpha_1, \alpha_2, \ldots, \alpha_n\right)$ such that $\alpha_j \in\{0,1,2, \ldots\}$;
\item $|\alpha|:=\alpha_1+\alpha_2+\cdots+\alpha_n$;
\item $\alpha !:=\alpha_{1} ! \alpha_{2} ! \cdots \alpha_{n} !$;
\item For $z=\left(z_1, z_2, \ldots, z_n\right) \in \mathbb{R}^n$, we define $z^\alpha:=z_1^{\alpha_1} z_2^{\alpha_2} \cdots z_n^{\alpha_n}$;
\item For a multi-index $\beta$, $\partial_{\beta}^{|\beta|}f(z) := \frac{\partial^{|\beta|}}{\partial^{\beta_1}_{z_1}\partial^{\beta_2}_{z_2} \cdots \partial^{\beta_n}_{z_n} }f(z)$;
\item We also denote the partial derivative with respect to $ z_{i} $ by $ \partial_{e_i} $. \\
\end{itemize}

\begin{definition}[G Set]
    Let $G$ denote the set of continuous functions $\mathbb{R}^{2d} \rightarrow \mathbb{R}$ of at most polynomial growth, i.e.~$g \in G$ if there exists positive integers $\nu_1, \nu_2>0$ such that $|g(z)| \leq \nu_1\left(1+|z|^{2 \nu_2}\right)$, for all $z \in \mathbb{R}^{2d}$.
\end{definition}

The next results are inspired by Theorem 1 of \cite{li2017stochastic} and are derived under some regularity assumption on the function $f$.

\begin{mybox}{gray}
\begin{assumption}
Assume that the following conditions on $f, f_i$, and their gradients are satisfied:
\begin{itemize}
\item $\nabla f, \nabla f_i $ satisfy a Lipschitz condition: There exists $ L>0 $ such that
$$
|\nabla f(u)-\nabla f(v)|+\sum_{i=1}^N\left|\nabla f_i(u)-\nabla f_i(v)\right| \leq L|u-v|;
$$
\item $ f, f_i $ and its partial derivatives up to order 7 belong to $ G $;
\item $ \nabla f, \nabla f_i $ satisfy a growth condition: There exists $ M>0 $ such that
$$
|\nabla f(z)|+\sum_{i=1}^N\left|\nabla f_i(z)\right| \leq M(1+|z|).
$$
\end{itemize}
\label{ass:regularity_f}
\end{assumption}
\end{mybox}

 \begin{mybox}{gray}
\begin{lemma}[Lemma 1 \cite{li2017stochastic}] \label{lemma:li1}
Let $ 0<\eta<1 $. Consider a stochastic process $ Z_t, t \geq 0 $ satisfying the SDE
$$
d Z_t=b\left(Z_t\right)dt+\sqrt{\eta} \sigma\left(Z_t\right) d W_t
$$
with $ Z_0=z \in \mathbb{R}^{2 d}$ and $ b, \sigma $ together with their derivatives belong to $ G $. Define the one-step difference $ \Delta=Z_\eta-z $, and indicate the $i$-th component of $\Delta$ with $\Delta_i$. Then we have

\begin{enumerate}
\item $ \E \Delta_{i}=b_{i} \eta+\frac{1}{2}\left[\sum_{j=1}^d b_{j} \partial_{e_j} b_{i}\right] \eta^2+\mathcal{O}\left(\eta^3\right) \quad \forall i = 1, \ldots, 2 d$;
\item $ \E \Delta_{i} \Delta_{j}=\left[b_{i} b_{j}+\sigma \sigma_{(i j)}^T\right] \eta^2+\mathcal{O}\left(\eta^3\right) \quad \forall i,j = 1, \ldots, 2d$;
\item $ \E \prod_{j=1}^s \Delta_{\left(i_j\right)}=\mathcal{O}\left(\eta^3\right) $ for all $ s \geq 3, i_j=1, \ldots,2  d $.
\end{enumerate}
All functions above are evaluated at $ z $.
\end{lemma}
\end{mybox}

\begin{mybox}{gray}
\begin{theorem}[Theorem 2 and Lemma 5, \cite{mil1986weak}]\label{thm:mils}
Let Assumption \ref{ass:regularity_f} hold and let us define $\bar{\Delta}=z_1-z$ to be the increment in the discrete-time algorithm, and indicate the $i$-th component of $\bar{\Delta}$ with $\bar{\Delta}_i$. If in addition there exists $K_1, K_2, K_3, K_4 \in G$ so that

\begin{enumerate}
\item $\left|\E \Delta_{i}-\E \bar{\Delta}_{i}\right| \leq K_1(z) \eta^{2}, \quad \forall i = 1, \ldots,2d$;

\item $\left|\E \Delta_{i} \Delta_{j} - \E \bar{\Delta}_{i} \bar{\Delta}_{j}\right| \leq K_2(z) \eta^{2}, \quad \forall i,j = 1, \ldots, 2d$;
\item $\left|\E \prod_{j=1}^s \Delta_{i_j}-\E \prod_{j=1}^s \bar{\Delta}_{i_j}\right| \leq K_3(z) \eta^{2}, \quad \forall s \geq 3, \quad \forall i_j \in \{1, \ldots, 2d \}$;
\item $\E \prod_{j=1}^{ 3}\left|\bar{\Delta}_{i_j}\right| \leq K_4(z) \eta^{2}, \quad \forall i_j \in \{1, \ldots, 2d \}$.
\end{enumerate}
Then, there exists a constant $C$ so that for all $k=0,1, \ldots, N$ we have
$$
\left|\E g\left(Z_{k \eta}\right)-\E g\left(z_k\right)\right| \leq C \eta.
$$
\end{theorem}
\end{mybox}

\subsection{Formal Derivation - SGDA} \label{sec:formal_SGDA}
In this subsection, we provide the first formal derivation of an SDE model for SGDA. Let us consider the stochastic process $ Z_t \in \mathbb{R}^{2d} $ defined as the solution of
\begin{equation}\label{eq:SGDA_SDE}
d Z_t = - \eta_t \circ F \left( Z_t \right) dt + \sqrt{\eta} (\eta_t 1^{\top}) \circ \sqrt{\Sigma\left( Z_t \right)}d W_{t},
\end{equation}
where
\begin{align} \label{eq:SGDA_sigma_star}
 \Sigma(z):=\E \left[ \left(F \left(z\right) - F_{\gamma}\left(z\right) \right)\left( F\left(z\right) - F_{\gamma}\left(z\right) \right)^{\top} \right],
\end{align}
and the $\circ$ symbol represents the Hadamard product.
The following theorem guarantees that such a process is a $1$-order SDE of the discrete-time algorithm of SGDA
\begin{equation}\label{eq:SGDA_Discr_Update_ext}
    z_{k+1}=z_k-\eta  \eta_k \circ  F_{\gamma_k}\left(z_k \right)
\end{equation}
with $ z_0 := z = (x,y) \in \mathbb{R}^d \times \mathbb{R}^d $, which is an extension of \ref{eq:SGDA_Discr_Update}.
\begin{mybox}{gray}
\begin{theorem}[Stochastic modified equations] \label{thm:SGDA_SDE}
Let $0<\eta<1, T>0$ and set $N=\lfloor T / \eta\rfloor$. Let $ z_k \in \mathbb{R}^{2d}, 0 \leq k \leq N$ denote a sequence of SGDA iterations defined by Eq.~\eqref{eq:SGDA_Discr_Update_ext}. Consider the stochastic process $Z_t$ defined in Eq.~\eqref{eq:SGDA_SDE} and fix some test function $g \in G$ and suppose that $g$ and its partial derivatives up to order 6 belong to $G$.

Then, under Assumption~\ref{ass:regularity_f}, there exists a constant $ C>0 $ independent of $ \eta $ such that for all $ k=0,1, \ldots, N $, we have

$$
\left|\E g\left(Z_{k \eta}\right)-\E g\left(z_k\right)\right| \leq C \eta .
$$

That is, the SDE \eqref{eq:SGDA_SDE} is an order $ 1 $ weak approximation of the SGDA iterations \eqref{eq:SGDA_Discr_Update_ext}.
\end{theorem}
\end{mybox}

\begin{mybox}{gray}
\begin{lemma} \label{lemma:SGDA_SDE}
Under the assumptions of Theorem \ref{thm:SGDA_SDE}, let $ 0<\eta<1 $ and consider $ z_k, k \geq 0 $ satisfying the SGDA iterations
$$
z_{k+1}=z_k-\eta  \eta_k \circ  F_{\gamma_k}\left(z_k \right)
$$
with $ z_0 := z = (x,y) \in \mathbb{R}^d \times \mathbb{R}^d $. From the definition the one-step difference $ \bar{\Delta}=z_1-z $, then we have

\begin{enumerate}
\item $ \E \bar{\Delta}_{i}=- \eta^i_0 F_i \eta \quad \forall i = 1, \ldots,2d$;
\item $ \E \bar{\Delta}_{i} \bar{\Delta}_{j}= \eta^i_0 \eta^j_0 F_i  F_j  \eta^2 + \eta^i_0 \eta^j_0 \Sigma_{(i j)} \eta^2 \quad \forall i,j = 1, \ldots,2d$;
\item $\E \prod_{j=1}^s \Bar{\Delta}_{i_j} =\mathcal{O}\left(\eta^3\right) \quad \forall s \geq 3, \quad i_j \in \{ 1, \ldots, 2d\}$.
\end{enumerate}
All the functions above are evaluated at $ z $.
\end{lemma}
\end{mybox}

\begin{proof}[Proof of Lemma \ref{lemma:SGDA_SDE}]

First of all, we have that by definition

\begin{equation}
    \E \left[ z_1- z \right] = - \eta \eta_0 \circ F(z),
\end{equation}
which implies 
\begin{equation}
    \E \bar{\Delta}_{i}=- \eta^i_0 F_i \left( z \right) \eta \quad \forall i = 1, \ldots,2d.
\end{equation}

Second, we have that by definition

\begin{align}
    \E \left[ \left(z_1- z\right)\left(z_1- z\right)^\top \right] & = \eta^2 \eta_0 \circ F(z) F(z)^{\top} \circ \eta_0^{\top}  + \eta^2  \E \left[ \eta_0 \circ \left( F(z) - F_{\gamma}(z)\right)\left( F(z) - F_{\gamma}(z)\right)^{\top} \circ \eta_0^{\top} \right] \nonumber \\ 
    & = \eta^2 \eta_0 \circ F(z) F(z)^{\top}  \circ \eta_0^{\top}+ \eta^2 (\eta_0 1^{\top}) \circ \Sigma(z) \circ (\eta_0 1^{\top})^{\top},
\end{align}
which implies that

\begin{equation}
    \E \bar{\Delta}_{i} \bar{\Delta}_{j}=\eta^i_0 \eta^j_0 F_i \left( z \right) F_j \left( z \right) \eta^2 + \eta^i_0 \eta^j_0  \Sigma_{(i j)}\left( z \right) \eta^2 \quad \forall i,j = 1, \ldots,2d.
\end{equation}
Finally, by definition 

\begin{equation}
    \E \prod_{j=1}^s \Bar{\Delta}_{i_j} =\mathcal{O}\left(\eta^3\right) \quad \forall s \geq 3, \quad i_j \in \{ 1, \ldots, 2d\},
\end{equation}
which concludes our proof.
\end{proof}

\begin{proof}[Proof of Theorem \ref{thm:SGDA_SDE}] 
\label{proof:SGDA_SDE}
To prove this result, all we need to do is check the conditions in Theorem \ref{thm:mils}. As we apply Lemma \ref{lemma:li1}, we make the following choices:

\begin{itemize}
\item $b(z)= - \eta_t \circ F(z)$;
\item $\sigma(z) = (\eta_t 1^{\top}) \circ \sqrt{\Sigma(z)}$.
\end{itemize}
First of all, we notice that $\forall i = 1, \ldots, 2d$, it holds that

\begin{itemize}
\item $\E \bar{\Delta}_{i} \overset{\text{1. Lemma \ref{lemma:SGDA_SDE}}}{=}- \eta^i_0 F_i \left( z \right) \eta$;
\item $ \E \Delta_{i} \overset{\text{1. Lemma \ref{lemma:li1}}}{=} - \eta^i_0 F_i \left( z \right) \eta +\mathcal{O}\left(\eta^2\right)$.
\end{itemize}
Therefore, we have that for some $K_1(z) \in G$,
\begin{equation}\label{eq:SGDA_cond1}
\left|\E \Delta_{i}-\E \bar{\Delta}_{i}\right| \leq K_1(z) \eta^{2}, \quad \forall i = 1, \ldots,2d.
\end{equation}
Additionally, we notice that $\forall i,j = 1, \ldots, d$, it holds that

\begin{itemize}
\item $ \E \bar{\Delta}_{i} \bar{\Delta}_{j} \overset{\text{2. Lemma \ref{lemma:SGDA_SDE}}}{=}\eta^i_0 \eta^j_0 F_i \left( z \right) F_j \left( z \right) \eta^2 + \eta^i_0 \eta^j_0 \Sigma_{(i j)}\left( z \right) \eta^2$;
\item $ \E \Delta_{i} \Delta_{j} \overset{\text{2. Lemma \ref{lemma:li1}}}{=}\eta^i_0 \eta^j_0 F_i \left( z \right) F_j \left( z \right) \eta^2 + \eta^i_0 \eta^j_0 \Sigma_{(i j)}\left( z \right) \eta^2 + \mathcal{O}\left(\eta^3 \right)$.
\end{itemize}
Therefore, we have that for some $K_2(z) \in G$,
\begin{equation}\label{eq:SGDA_cond2}
\left|\E \Delta_{i} \Delta_{j} - \E \bar{\Delta}_{i} \bar{\Delta}_{j}\right| \leq K_2(z) \eta^{2}, \quad \forall i,j = 1, \ldots, 2d.
\end{equation}
Additionally, we notice that $\forall s \geq 3, \forall i_j \in \{1, \ldots, 2d \}$, it holds that

\begin{itemize}
\item $ \E \prod_{j=1}^s \bar{\Delta}_{i_j}\overset{\text{3. Lemma \ref{lemma:SGDA_SDE}}}{=}\mathcal{O}\left(\eta^3\right)$;
\item $ \E \prod_{j=1}^s \Delta_{i_j}\overset{\text{3. Lemma \ref{lemma:li1}}}{=}\mathcal{O}\left(\eta^3\right)$.
\end{itemize}
Therefore, we have that for some $K_3(z) \in G$,
\begin{equation}\label{eq:SGDA_cond3}
\left|\E \prod_{j=1}^s \Delta_{i_j}-\E \prod_{j=1}^s \bar{\Delta}_{i_j}\right| \leq K_3(z) \eta^{2}.
\end{equation}
Additionally, for some $K_4(z) \in G$, $\forall i_j \in \{1, \ldots, d \}$,
\begin{equation} \label{eq:SGDA_cond4}
\E \prod_{j=1}^{ 3}\left|\bar{\Delta}_{\left(i_j\right)}\right| \overset{\text{3. Lemma \ref{lemma:SGDA_SDE}}}{\leq}K_4(z) \eta^{2}.
\end{equation}
To conclude, Eq.~\eqref{eq:SGDA_cond1}, Eq.~\eqref{eq:SGDA_cond2}, Eq.~\eqref{eq:SGDA_cond3}, and Eq.~\eqref{eq:SGDA_cond4} allow us to conclude the proof.
\end{proof}

\begin{mybox}{gray}
\begin{corollary}\label{thm:SGDA_SDE_Simplified}
Let us take the same assumptions of Theorem~\ref{thm:SGDA_SDE}. Additionally, let us assume that stochastic gradients can be written as $\nabla_x f_{\gamma}(z) = \nabla_x f(z) + U^x$ and $\nabla_y f_{\gamma}(z) = \nabla_y f(z) + U^y$ such that $U^x$ and $U^y$ are independent noises that do not depend on $z$, whose expectation is $0$, and whose covariance matrix is $\Sigma$. For $\eta_t = \mathbf{1}$, the SDE \eqref{eq:SGDA_SDE} becomes
\begin{align}\label{eq:SGDA_SDE_Simplified}
& d X_t = - \nabla_{x} f \left( Z_t \right) dt + \sqrt{\eta\Sigma}d W^{x}_{t}, \\
& d Y_t = + \nabla_{y} f \left( Z_t \right) dt + \sqrt{\eta\Sigma} d W^{y}_{t}. \nonumber 
\end{align}
\end{corollary}
\end{mybox}

\begin{proof}[Proof of Corollary \ref{thm:SGDA_SDE_Simplified}]
It follows directly by the independence of the noise, the definition of the scheduler, and the definition of the covariance matrix.
\end{proof}

\subsection{Formal Derivation - SEG} \label{sec:formal_SEG}
In this subsection, we provide the first formal derivation of an SDE model for SEG.
Before presenting the proof, we introduce some notation. Let $\bgamma := (\gamma^1,\gamma^2)$, $\bar F_{\bgamma}(z) := \nabla F_{\gamma^1}(z) F_{\gamma^2}(z)$, and $\bar F(z):=\E[\bar F_{\bgamma}(z)]$ be its expectation. We denote the noise in $\bar F$ as $\bar \xi_{\bgamma}(z) := \bar F_{\bgamma}(z) -\bar F(z)$.

Let us consider the stochastic process $ Z_t \in \mathbb{R}^{2d} $ defined as the solution of
\begin{equation}\label{eq:SEG_SDE}
d Z_t = - F^{\seg}\left( Z_t \right) dt+ \sqrt{\eta\Sigma^{\seg}\left( Z_t \right)}d W_{t},
\end{equation}
with
\begin{align}
    & F^{\seg}(z):= F(z) - \rho \bar F(z),\\
    & \Sigma^{\seg}(z) := \Sigma(z) +\rho\left[\bar \Sigma(z) + \bar \Sigma(z)^\top\right],
\end{align}
where 
$\bar{\Sigma}(z)$ is defined as
\begin{align} \label{eq:SEG_sigma_star}
\E\left[\xi_{\gamma^1}(z)\bar \xi_{\bgamma}(z)^\top \right] = \E \left[ \left( F\left(z\right) - F_{\gamma^1}\left(z\right) \right) \left(\E \left[ \nabla F_{\gamma^1}(z) F_{\gamma^2}(z) \right] - \nabla F_{\gamma^1}(z) F_{\gamma^2}(z) \right)^{\top} \right].
\end{align}

\begin{mybox}{gray}
\begin{theorem}[Stochastic modified equations] \label{thm:SEG_SDE}
Let $0<\eta<1, T>0$ and set $N=\lfloor T / \eta\rfloor$. Let $ z_k \in \mathbb{R}^{2d}, 0 \leq k \leq N$ denote a sequence of SEG iterations defined by Eq.~\eqref{eq:SEG_Discr_Update}. Additionally, let us take 
\begin{equation}\label{eq:SEG_rho_theta_half}
\rho = \mathcal{O}\left(\sqrt{\eta}\right).
\end{equation}

Consider the stochastic process $Z_t$ defined in Eq.~\eqref{eq:SEG_SDE} and fix some test function $g \in G$ and suppose that $g$ and its partial derivatives up to order 6 belong to $G$.

Then, under Assumption~\ref{ass:regularity_f}, there exists a constant $ C>0 $ independent of $ \eta $ such that for all $ k=0,1, \ldots, N $, we have

$$
\left|\E g\left(Z_{k \eta}\right)-\E g\left(z_k\right)\right| \leq C \eta .
$$

That is, the SDE \eqref{eq:SEG_SDE} is an order $ 1 $ weak approximation of the SEG iterations \eqref{eq:SEG_Discr_Update}.
\end{theorem}
\end{mybox}

\begin{mybox}{gray}
\begin{lemma} \label{lemma:SEG_SDE}
Under the assumptions of Theorem \ref{thm:SEG_SDE}, let $ 0<\eta<1 $ and consider $ z_k, k \geq 0 $ satisfying the SEG iterations \eqref{eq:SEG_Discr_Update}
\begin{equation}
    \left[\begin{array}{l} x_{k+1} \\ y_{k+1} \end{array}\right] = \left[\begin{array}{l} x_{k} \\ y_{k} \end{array}\right] - \eta \left[\begin{array}{l} + \nabla_{x} f_{\gamma^1_k} \left( x_{k} - \rho \nabla_{x} f_{\gamma^2_k} \left(x_{k}, y_{k} \right), y_{k} + \rho \nabla_{y} f_{\gamma^2_k} \left(x_{k}, y_{k} \right) \right) \\ - \nabla_{y} f_{\gamma^1_k} \left( x_{k} - \rho \nabla_{x} f_{\gamma^2_k} \left(x_{k}, y_{k} \right), y_{k} + \rho \nabla_{y} f_{\gamma^2_k} \left(x_{k}, y_{k} \right) \right) \end{array}\right]
\end{equation}
with $ z_0 := z = (x,y) \in \mathbb{R}^d \times \mathbb{R}^d $. From the definition the one-step difference $ \bar{\Delta}=z_1-z $, then we have

\begin{enumerate}
\item $ \E \bar{\Delta}_{i}= - F^{\seg}_i \eta + \mathcal{O}\left( \eta^2\right)  \quad \forall i = 1, \ldots,2d$;
\item $ \E \bar{\Delta}_{i} \bar{\Delta}_{j}= F^{\seg}_i F^{\seg}_j \eta^2 + \Sigma^{\seg}_{(i j)} \eta^2 + \mathcal{O}\left( \eta^3\right)  \quad \forall i,j = 1, \ldots,2d$;
\item $\E \prod_{j=1}^s \Bar{\Delta}_{i_j} =\mathcal{O}\left(\eta^3\right) \quad \forall s \geq 3, \quad i_j \in \{ 1, \ldots, 2d\}$.
\end{enumerate}
All the functions above are evaluated at $ z $.
\end{lemma}
\end{mybox}

\begin{proof}[Proof of Lemma \ref{lemma:SEG_SDE}]

First of all, we have that by definition and using a Taylor expansion,
\begin{align}
    x_{k+1} & = x_k - \eta \nabla_{x} f_{\gamma^1_k} \left( x_{k} - \rho \nabla_{x} f_{\gamma^2_k} \left(x_{k}, y_{k} \right), y_{k} + \rho \nabla_{y} f_{\gamma^2_k} \left(x_{k}, y_{k} \right) \right)  \\
    & = x_k - \eta \nabla_{x} f_{\gamma^1_k} \left( x_{k}, y_{k} \right) + \eta \rho \nabla^2_{xx} f_{\gamma^1_k} \left( x_{k}, y_{k} \right) \nabla_{x} f_{\gamma^2_k} \left(x_{k}, y_{k} \right) - \eta \rho \nabla^2_{xy} f_{\gamma^1_k} \left( x_{k}, y_{k} \right) \nabla_{y} f_{\gamma^2_k} \left(x_{k}, y_{k} \right) + \mathcal{O}\left(\eta \rho^2 \right) \nonumber ,
\end{align}
and
\begin{align}
    y_{k+1} & = y_k + \eta \nabla_{y} f_{\gamma^1_k} \left( x_{k} -  \rho \nabla_{x} f_{\gamma^2_k} \left(x_{k}, y_{k} \right), y_{k} + \rho \nabla_{y} f_{\gamma^2_k} \left(x_{k}, y_{k} \right) \right) \\
    & = y_k + \eta \nabla_{y} f_{\gamma^1_k} \left( x_{k}, y_{k} \right) - \eta \rho \nabla^2_{xy} f_{\gamma^2_k} \left( x_{k}, y_{k} \right) \nabla_{x} f_{\gamma^2_k} \left(x_{k}, y_{k} \right) + \eta \rho \nabla^2_{yy} f_{\gamma_k} \left( x_{k}, y_{k} \right) \nabla_{y} f_{\gamma_k} \left(x_{k}, y_{k} \right) + \mathcal{O}\left(\eta \rho^2 \right) \nonumber .
\end{align}
Therefore
\begin{align}
    z_{1} & = z - \eta F_{\gamma^1}(z) + \eta \rho \nabla F_{\gamma^1}(z) F_{\gamma^2}(z) + \mathcal{O}\left(\eta^2 \right),
\end{align}
which implies that

\begin{align}
    \E \left[z_{1}-z \right] & = -\eta F(z) + \eta \rho \E \left[ \nabla F_{\gamma^1}(z) F_{\gamma^2}(z) \right] + \mathcal{O}\left(\eta^2 \right) \nonumber  \\
      & = z -\eta F^{\seg}\left( z \right)+ \mathcal{O}\left(\eta^2 \right),
\end{align}
where $F^{\seg}\left( z \right):=F(z) - \rho \E \left[ \nabla F_{\gamma^1}(z) F_{\gamma^2}(z) \right]$,
which in turn implies that
\begin{equation}
     \E \bar{\Delta}_{i}=-F^{\seg}_i\left( z \right) \eta + \mathcal{O}\left( \eta^2\right)  \quad \forall i = 1, \ldots,2d.
\end{equation}

Second, we have that 
\begin{align}
    \E \left[ \left(z_1- z\right)\left(z_1- z\right)^\top \right] & = \eta^2 \left[ \left(F^{\seg}(z) \right)\left(F^{\seg}(z) \right)^{\top} \right] \nonumber \\
    & + \eta^2\E \left[ \left( F(z) - F_{\gamma^1}(z)\right)\left( F(z) - F_{\gamma^1}(z)\right)^{\top} \right]  \nonumber \\
    & + \eta^2 \rho \left( \E \left[ \left( F\left(z\right) - F_{\gamma^1}\left(z\right) \right) \left(\E \left[ \nabla F_{\gamma^1}(z) F_{\gamma^2}(z) \right] - \nabla F_{\gamma^1}(z) F_{\gamma^2}(z) \right)^{\top} \right] \right)  \nonumber \\
    & + \eta^2 \rho \left( \E \left[ \left( F\left(z\right) - F_{\gamma^1}\left(z\right) \right) \left(\E \left[ \nabla F_{\gamma^1}(z) F_{\gamma^2}(z) \right] - \nabla F_{\gamma^1}(z) F_{\gamma^2}(z) \right)^{\top} \right] \right)^{\top} + \mathcal{O}\left( \eta^3\right) \nonumber  \\ 
    & = \eta^2 \left[ \left(F^{\seg}(z) \right)\left(F^{\seg}(z) \right)^{\top} \right] + \eta^2\Sigma^{\seg}(z) + \mathcal{O}\left( \eta^3\right),
\end{align}
which implies that

\begin{equation}
    \E \bar{\Delta}_{i} \bar{\Delta}_{j}= F^{\seg}_i\left( z \right) F^{\seg}_j\left( z \right) \eta^2 + \Sigma^{\seg}_{(i j)}\left( z \right) \eta^2 + \mathcal{O}\left( \eta^3\right)  \quad \forall i,j = 1, \ldots,2d.
\end{equation}

Finally, by definition
\begin{equation}
    \E \prod_{j=1}^s \Bar{\Delta}_{i_j} =\mathcal{O}\left(\eta^3\right) \quad \forall s \geq 3, \quad i_j \in \{ 1, \ldots, 2d\},
\end{equation}
which concludes our proof.
\end{proof}

\begin{proof}[Proof of Theorem \ref{thm:SEG_SDE}] 
\label{proof:SEG_SDE}
To prove this result, all we need to do is check the conditions in Theorem \ref{thm:mils}. As we apply Lemma \ref{lemma:li1}, we make the following choices:

\begin{itemize}
\item $b(z)= - F^{\seg}(z)$;
\item $\sigma(z) =\Sigma^{\seg}(z)^{\frac{1}{2}}$.
\end{itemize}
First of all, we notice that $\forall i = 1, \ldots, 2d$, it holds that

\begin{itemize}
\item $\E \bar{\Delta}_{i} \overset{\text{1. Lemma \ref{lemma:SEG_SDE}}}{=}- F^{\seg}_i\left( z \right)  \eta + \mathcal{O}\left( \eta^2\right )$;
\item $ \E \Delta_{i} \overset{\text{1. Lemma \ref{lemma:li1}}}{=} - F^{\seg}_i\left( z \right)  \eta + \mathcal{O}\left( \eta^2\right )$.
\end{itemize}
Therefore, we have that for some $K_1(z) \in G$,
\begin{equation}\label{eq:SEG_cond1}
\left|\E \Delta_{i}-\E \bar{\Delta}_{i}\right| \leq K_1(z) \eta^{2}, \quad \forall i = 1, \ldots,2d.
\end{equation}
Additionally, we notice that $\forall i,j = 1, \ldots, d$, it holds that
\begin{itemize}
\item $ \E \bar{\Delta}_{i} \bar{\Delta}_{j} \overset{\text{2. Lemma \ref{lemma:SEG_SDE}}}{=} F^{\seg}_i\left( z \right) F^{\seg}_j\left( z \right) \eta^2 + \Sigma^{\seg}_{(i j)}\left( z \right) \eta^2 + \mathcal{O}\left( \eta^3\right) $;
\item $ \E \Delta_{i} \Delta_{j} \overset{\text{2. Lemma \ref{lemma:li1}}}{=}F^{\seg}_i\left( z \right) F^{\seg}_j\left( z \right) \eta^2 + \Sigma^{\seg}_{(i j)}\left( z \right) \eta^2 + \mathcal{O}\left( \eta^3\right) $.
\end{itemize}
Therefore, we have that for some $K_2(z) \in G$,
\begin{equation}\label{eq:SEG_cond2}
\left|\E \Delta_{i} \Delta_{j} - \E \bar{\Delta}_{i} \bar{\Delta}_{j}\right| \leq K_2(z) \eta^{2}, \quad \forall i,j = 1, \ldots, 2d.
\end{equation}
Additionally, we notice that $\forall s \geq 3, \forall i_j \in \{1, \ldots, 2d \}$, it holds that

\begin{itemize}
\item $ \E \prod_{j=1}^s \bar{\Delta}_{i_j}\overset{\text{3. Lemma \ref{lemma:SEG_SDE}}}{=}\mathcal{O}\left(\eta^3\right)$;
\item $ \E \prod_{j=1}^s \Delta_{i_j}\overset{\text{3. Lemma \ref{lemma:li1}}}{=}\mathcal{O}\left(\eta^3\right)$.
\end{itemize}
Therefore, we have that for some $K_3(z) \in G$,
\begin{equation}\label{eq:SEG_cond3}
\left|\E \prod_{j=1}^s \Delta_{i_j}-\E \prod_{j=1}^s \bar{\Delta}_{i_j}\right| \leq K_3(z) \eta^{2}.
\end{equation}
Additionally, for some $K_4(z) \in G$, $\forall i_j \in \{1, \ldots, d \}$,
\begin{equation} \label{eq:SEG_cond4}
\E \prod_{j=1}^{ 3}\left|\bar{\Delta}_{\left(i_j\right)}\right| \overset{\text{3. Lemma \ref{lemma:SEG_SDE}}}{\leq}K_4(z) \eta^{2}.
\end{equation}
Finally, Eq.~\eqref{eq:SEG_cond1}, Eq.~\eqref{eq:SEG_cond2}, Eq.~\eqref{eq:SEG_cond3}, and Eq.~\eqref{eq:SEG_cond4} allow us to conclude the proof.

\end{proof}

\begin{mybox}{gray}
\begin{corollary}\label{thm:SEG_SDE_Simplified_SameSample}
Let us take the same assumptions of Theorem~\ref{thm:SEG_SDE}. Additionally, let us assume that $\gamma^1=\gamma^2=\gamma$, the stochastic gradients can be written as $\nabla_x f_{\gamma}(z) = \nabla_x f(z) + U^x$ and $\nabla_y f_{\gamma}(z) = \nabla_y f(z) + U^y$ such that $U^x$ and $U^y$ are independent noises that do not depend on $z$, whose expectation is $0$, and whose covariance matrix is $\Sigma$. Therefore, the SDE \eqref{eq:SEG_SDE} is
\begin{align}\label{eq:SEG_SDE_Simplified_SameSample}
d Z_{t} & = -F(Z_{t}) + \rho \nabla F(Z_{t}) F(Z_{t}) dt + \sqrt{\eta} \left(\mathbf{I}_{2d} - \rho \nabla F\left( Z_{t} \right) \right) \sqrt{\Sigma}  d W_t.
\end{align}
\end{corollary}
\end{mybox}

\begin{proof}[Proof of Corollary \ref{thm:SEG_SDE_Simplified_SameSample}]
First of all, we notice that
\begin{align}
 F(z) - \rho \E \left[ \nabla F_{\gamma}(z) F_{\gamma}(z) \right]
& = F(z) - \rho \E \left[ \nabla F(z) F_{\gamma}(z) \right] \nonumber  \\
& = F(z) - \rho \nabla F(z) F(z).
\end{align}
Second, based on our assumption of the noise structure, we can rewrite Eq. \eqref{eq:SEG_sigma_star} of the matrix $ \Sigma^{\seg}$ as
\begin{align}
   \Sigma^{\seg} & =  \eta^2\E \left[ \left( F(z) - F_{\gamma}(z)\right)\left( F(z) - F_{\gamma}(z)\right)^{\top} \right] \nonumber  \\
   & + \eta^2 \rho \left( \E \left[ \left( F\left(z\right) - F_{\gamma}\left(z\right) \right) \left(\E \left[ \nabla F_{\gamma}(z) F_{\gamma}(z) \right] - \nabla F_{\gamma}(z) F_{\gamma}(z) \right)^{\top} \right] \right) \nonumber  \\
    & + \eta^2 \rho \left( \E \left[ \left( F\left(z\right) - F_{\gamma}\left(z\right) \right) \left(\E \left[ \nabla F_{\gamma}(z) F_{\gamma}(z) \right] - \nabla F_{\gamma}(z) F_{\gamma}(z) \right)^{\top} \right] \right)^{\top} + \mathcal{O}\left( \eta^3\right) \nonumber  \\ 
    & =  \eta^2\E \left[ \left( F(z) - F_{\gamma}(z)\right)\left( F(z) - F_{\gamma}(z)\right)^{\top} \right] \nonumber  \\
   & + \eta^2 \rho \left( \E \left[ \left( F\left(z\right) - F_{\gamma}\left(z\right) \right) \left(\E \left[ \nabla F(z) F_{\gamma}(z) \right] - \nabla F(z) F_{\gamma}(z) \right)^{\top} \right] \right) \nonumber  \\
    & + \eta^2 \rho \left( \E \left[ \left( F\left(z\right) - F_{\gamma}\left(z\right) \right) \left(\E \left[ \nabla F(z) F_{\gamma}(z) \right] - \nabla F(z) F_{\gamma}(z) \right)^{\top} \right] \right)^{\top} + \mathcal{O}\left( \eta^3\right) \nonumber  \\ 
   &  =  \eta^2 \left(\Sigma  + \rho  \Sigma \nabla F\left( z \right)^{\top} + \rho \nabla F\left( z \right) \Sigma  \right) +  \mathcal{O}\left( \eta^3\right).
\end{align}
By observing that $(\mathbf{I}_{2d}-\rho \nabla F\left( z \right)) \sqrt{\Sigma}\sqrt{\Sigma} (\mathbf{I}_{2d}-\rho \nabla F\left( z \right)^{\top}) =\Sigma  + \rho \Sigma \nabla F\left( z \right)^{\top} + \rho \nabla F\left( z \right) \Sigma + \mathcal{O}\left( \eta\right)$, we conclude the proof.
\end{proof}
\begin{mybox}{gray}
\begin{corollary}\label{thm:SEG_SDE_Simplified_IndepSample}
Let us take the same assumptions of Theorem~\ref{thm:SEG_SDE}. Additionally, let us assume that $\gamma^1$ and $\gamma^2$, are independent and the stochastic gradients can be written as $\nabla_x f_{\gamma^i}(z) = \nabla_x f(z) + U_x^i$ and $\nabla_y f_{\gamma^i}(z) = \nabla_y f(z) + U_y^i$ such that $U_x^i$ and $U_y^i$ are independent noises that do not depend on $z$, whose expectation is $0$, and whose covariance matrix is $\Sigma$. Therefore, the SDE \eqref{eq:SEG_SDE} is
\begin{align}\label{eq:SEG_SDE_Simplified_IndepSample}
d Z_{t} & = -F(Z_{t}) + \rho \nabla F(Z_{t}) F(Z_{t}) dt + \sqrt{\eta\Sigma}  d W_t.
\end{align}
\end{corollary}
\end{mybox}
\begin{proof}[Proof of Corollary \ref{thm:SEG_SDE_Simplified_IndepSample}]
First of all, we notice that
\begin{align}
 F(z) - \rho \E \left[ \nabla F_{\gamma^1}(z) F_{\gamma^2}(z) \right]  = F(z) - \rho \nabla F(z) F(z).
\end{align}
Second, based on our assumption of the noise structure, we can rewrite Eq. \eqref{eq:SEG_sigma_star} of the matrix $ \Sigma^{\seg}$ as
\begin{align}
   \Sigma^{\seg} &  =  \eta^2\E \left[ \left( F(z) - F_{\gamma^1}(z)\right)\left( F(z) - F_{\gamma^1}(z)\right)^{\top} \right] \nonumber  \\
   & + \eta^2 \rho \left( \E \left[ \left( F\left(z\right) - F_{\gamma^1}\left(z\right) \right) \left(\nabla F(z) F(z) - \nabla F_{\gamma^1}(z) F_{\gamma^2}(z) \right)^{\top} \right] \right)  \nonumber \\
    & + \eta^2 \rho \left( \E \left[ \left( F\left(z\right) - F_{\gamma^1}\left(z\right) \right) \left(\nabla F(z) F(z) - \nabla F_{\gamma^1}(z) F_{\gamma^2}(z) \right)^{\top} \right] \right)^{\top} + \mathcal{O}\left( \eta^3\right) \nonumber  \\ 
    &  = \eta^2\E \left[ \left( F(z) - F_{\gamma^1}(z)\right)\left( F(z) - F_{\gamma^1}(z)\right)^{\top} \right] \nonumber  \\
   & + \eta^2 \rho \left( \E \left[ \left( F\left(z\right) - F_{\gamma^1}\left(z\right) \right) \left(\nabla F(z) F(z) - \nabla F(z) F_{\gamma^2}(z) \right)^{\top} \right] \right)  \nonumber \\
    & + \eta^2 \rho \left( \E \left[ \left( F\left(z\right) - F_{\gamma^1}\left(z\right) \right) \left(\nabla F(z) F(z) - \nabla F(z) F_{\gamma^2}(z) \right)^{\top} \right] \right)^{\top} + \mathcal{O}\left( \eta^3\right) \nonumber  \\ 
    & =  \eta^2 \Sigma +  \mathcal{O}\left( \eta^3\right),
\end{align}
which concludes the proof.

\end{proof}

\subsubsection{Continuous-time SEG is equivalent to SGDA if \texorpdfstring{$\rho = \mathcal{O}(\eta)$}{Lg}}
In this subsection, we provide a formal proof that if $\rho = \mathcal{O}(\eta)$, the first-order weak approximation of SEG is the same as that of SGDA. This is consistent with the ODE literature on ODEs for these models \citep{chavdarova2021last,lu2022sr}.

We will consider the stochastic process $ Z_t \in \mathbb{R}^{2d} $ defined as the solution of
\begin{equation}\label{eq:SEG_SDE_small_rho}
d Z_t = - F \left( Z_t \right) dt + \sqrt{\eta\Sigma}d W_{t}.
\end{equation}

\begin{mybox}{gray}
\begin{theorem}[Stochastic modified equations] \label{thm:SEG_SDE_small_rho}
Let $0<\eta<1, T>0$ and set $N=\lfloor T / \eta\rfloor$. Let $ z_k \in \mathbb{R}^{2d}, 0 \leq k \leq N$ denote a sequence of SEG iterations defined by Eq.~\eqref{eq:SEG_Discr_Update}. Additionally, let us take 
\begin{equation}\label{eq:SEG_rho_theta}
\rho = \mathcal{O}\left(\eta\right).
\end{equation}

Consider the stochastic process $Z_t$ defined in Eq.~\eqref{eq:SEG_SDE_small_rho} and fix some test function $g \in G$ and suppose that $g$ and its partial derivatives up to order 6 belong to $G$.

Then, under Assumption~\ref{ass:regularity_f}, there exists a constant $ C>0 $ independent of $ \eta $ such that for all $ k=0,1, \ldots, N $, we have

$$
\left|\E g\left(Z_{k \eta}\right)-\E g\left(z_k\right)\right| \leq C \eta .
$$

That is, the SDE \eqref{eq:SEG_SDE_small_rho} is an order $ 1 $ weak approximation of the SEG iterations \eqref{eq:SEG_Discr_Update}.
\end{theorem}
\end{mybox}

\begin{mybox}{gray}
\begin{lemma} \label{lemma:SEG_SDE_small_rho}
Under the assumptions of Theorem \ref{thm:SEG_SDE_small_rho}, let $ 0<\eta<1 $ and consider $ z_k, k \geq 0 $ satisfying the SEG iterations \eqref{eq:SEG_Discr_Update}
\begin{equation}
    \left[\begin{array}{l} x_{k+1} \\ y_{k+1} \end{array}\right] = \left[\begin{array}{l} x_{k} \\ y_{k} \end{array}\right] - \eta \left[\begin{array}{l} + \nabla_{x} f_{\gamma^1_k} \left( x_{k} - \rho \nabla_{x} f_{\gamma^2_k} \left(x_{k}, y_{k} \right), y_{k} + \rho \nabla_{y} f_{\gamma^2_k} \left(x_{k}, y_{k} \right) \right) \\ - \nabla_{y} f_{\gamma^1_k} \left( x_{k} - \rho \nabla_{x} f_{\gamma^2_k} \left(x_{k}, y_{k} \right), y_{k} + \rho \nabla_{y} f_{\gamma^2_k} \left(x_{k}, y_{k} \right) \right) \end{array}\right]
\end{equation}
with $ z_0 := z = (x,y) \in \mathbb{R}^d \times \mathbb{R}^d $. From the definition the one-step difference $ \bar{\Delta}=z_1-z $, then we have

\begin{enumerate}
\item $ \E \bar{\Delta}_{i}= - F_i\left( z \right) \eta + \mathcal{O}\left( \eta^2\right)  \quad \forall i = 1, \ldots,2d$;
\item $ \E \bar{\Delta}_{i} \bar{\Delta}_{j}= F_i\left( z \right) F_j\left( z \right) \eta^2 + \Sigma_{(i j)}\left( z \right) \eta^2 + \mathcal{O}\left( \eta^3\right)  \quad \forall i,j = 1, \ldots,2d$;
\item $\E \prod_{j=1}^s \Bar{\Delta}_{i_j} =\mathcal{O}\left(\eta^3\right) \quad \forall s \geq 3, \quad i_j \in \{ 1, \ldots, 2d\}.$
\end{enumerate}
All the functions above are evaluated at $ z $.
\end{lemma}
\end{mybox}

\begin{proof}[Proof of Lemma \ref{lemma:SEG_SDE_small_rho}]

First of all, we have that by definition and using a Taylor expansion,
\begin{align}
    x_{k+1} & = x_k - \eta \nabla_{x} f_{\gamma^1_k} \left( x_{k} - \rho \nabla_{x} f_{\gamma^2_k} \left(x_{k}, y_{k} \right), y_{k} + \rho \nabla_{y} f_{\gamma^2_k} \left(x_{k}, y_{k} \right) \right) \\
    & = x_k - \eta \nabla_{x} f_{\gamma^1_k} \left( x_{k}, y_{k} \right) + \eta \rho \nabla^2_{xx} f_{\gamma^1_k} \left( x_{k}, y_{k} \right) \nabla_{x} f_{\gamma^2_k} \left(x_{k}, y_{k} \right) - \eta \rho \nabla^2_{xy} f_{\gamma^1_k} \left( x_{k}, y_{k} \right) \nabla_{y} f_{\gamma^2_k} \left(x_{k}, y_{k} \right) + \mathcal{O}\left(\eta \rho^2 \right) \nonumber ,
\end{align}
and
\begin{align}
    y_{k+1} & = y_k + \eta \nabla_{y} f_{\gamma^1_k} \left( x_{k} -  \rho \nabla_{x} f_{\gamma^2_k} \left(x_{k}, y_{k} \right), y_{k} + \rho \nabla_{y} f_{\gamma^2_k} \left(x_{k}, y_{k} \right) \right) \\
    & = y_k + \eta \nabla_{y} f_{\gamma^1_k} \left( x_{k}, y_{k} \right) - \eta \rho \nabla^2_{xy} f_{\gamma^2_k} \left( x_{k}, y_{k} \right) \nabla_{x} f_{\gamma^2_k} \left(x_{k}, y_{k} \right) + \eta \rho \nabla^2_{yy} f_{\gamma_k} \left( x_{k}, y_{k} \right) \nabla_{y} f_{\gamma_k} \left(x_{k}, y_{k} \right) + \mathcal{O}\left(\eta \rho^2 \right) \nonumber .
\end{align}
Therefore
\begin{align}
    z_{1} & = z - \eta F_{\gamma^1}(z) + \mathcal{O}\left(\eta^2 \right),
\end{align}
which implies that
\begin{align}
    \E \left[z_{k+1}-z_{k} \right] & = -\eta F(z_k) + \mathcal{O}\left(\eta^2 \right),
\end{align}
which in turn implies that
\begin{equation}
     \E \bar{\Delta}_{i}=-F_i\left( z \right) \eta + \mathcal{O}\left( \eta^2\right)  \quad \forall i = 1, \ldots,2d.
\end{equation}
Second, we have that 
\begin{align}
    \E \left[ \left(z_1- z\right)\left(z_1- z\right)^\top \right] & = \eta^2 \left[ \left(F(z) \right)\left(F(z) \right)^{\top} \right] + \eta^2\Sigma(z) + \mathcal{O}\left( \eta^3\right),
\end{align}
which implies that
\begin{equation}
    \E \bar{\Delta}_{i} \bar{\Delta}_{j}= F_i\left( z \right) F_j\left( z \right) \eta^2 + \Sigma_{(i j)}\left( z \right) \eta^2 + \mathcal{O}\left( \eta^3\right)  \quad \forall i,j = 1, \ldots,2d.
\end{equation}
Finally, by definition,
\begin{equation}
    \E \prod_{j=1}^s \Bar{\Delta}_{i_j} =\mathcal{O}\left(\eta^3\right) \quad \forall s \geq 3, \quad i_j \in \{ 1, \ldots, 2d\},
\end{equation}
which concludes our proof.
\end{proof}

\begin{proof}[Proof of Theorem \ref{thm:SEG_SDE_small_rho}] 
\label{proof:SEG_SDE_small_rho}
To prove this result, all we need to do is check the conditions in Theorem \ref{thm:mils}. As we apply Lemma \ref{lemma:li1}, we make the following choices:

\begin{itemize}
\item $b(z)= - F(z)$;
\item $\sigma(z) =\Sigma(z)^{\frac{1}{2}}$.
\end{itemize}
First of all, we notice that $\forall i = 1, \ldots, 2d$, it holds that

\begin{itemize}
\item $\E \bar{\Delta}_{i} \overset{\text{1. Lemma \ref{lemma:SEG_SDE_small_rho}}}{=}- F_i\left( z \right)  \eta + \mathcal{O}\left( \eta^2\right )$;
\item $ \E \Delta_{i} \overset{\text{1. Lemma \ref{lemma:li1}}}{=} - F_i\left( z \right)  \eta + \mathcal{O}\left( \eta^2\right )$.
\end{itemize}
Therefore, we have that for some $K_1(z) \in G$,
\begin{equation}\label{eq:SEG_cond1_small_rho}
\left|\E \Delta_{i}-\E \bar{\Delta}_{i}\right| \leq K_1(z) \eta^{2}, \quad \forall i = 1, \ldots,2d.
\end{equation}
Additionally, we notice that $\forall i,j = 1, \ldots, d$, it holds that

\begin{itemize}
\item $ \E \bar{\Delta}_{i} \bar{\Delta}_{j} \overset{\text{2. Lemma \ref{lemma:SEG_SDE_small_rho}}}{=} F_i\left( z \right) F_j\left( z \right) \eta^2 + \Sigma_{(i j)}\left( z \right) \eta^2 + \mathcal{O}\left( \eta^3\right) $;
\item $ \E \Delta_{i} \Delta_{j} \overset{\text{2. Lemma \ref{lemma:li1}}}{=}F_i\left( z \right) F_j\left( z \right) \eta^2 + \Sigma_{(i j)}\left( z \right) \eta^2 + \mathcal{O}\left( \eta^3\right) $.
\end{itemize}
Therefore, we have that for some $K_2(z) \in G$,
\begin{equation}\label{eq:SEG_cond2_small_rho}
\left|\E \Delta_{i} \Delta_{j} - \E \bar{\Delta}_{i} \bar{\Delta}_{j}\right| \leq K_2(z) \eta^{2}, \quad \forall i,j = 1, \ldots, 2d.
\end{equation}
Additionally, we notice that $\forall s \geq 3, \forall i_j \in \{1, \ldots, 2d \}$, it holds that

\begin{itemize}
\item $ \E \prod_{j=1}^s \bar{\Delta}_{i_j}\overset{\text{3. Lemma \ref{lemma:SEG_SDE_small_rho}}}{=}\mathcal{O}\left(\eta^3\right)$;
\item $ \E \prod_{j=1}^s \Delta_{i_j}\overset{\text{3. Lemma \ref{lemma:li1}}}{=}\mathcal{O}\left(\eta^3\right)$.
\end{itemize}
Therefore, we have that for some $K_3(z) \in G$

\begin{equation}\label{eq:SEG_cond3_small_rho}
\left|\E \prod_{j=1}^s \Delta_{i_j}-\E \prod_{j=1}^s \bar{\Delta}_{i_j}\right| \leq K_3(z) \eta^{2}.
\end{equation}
Additionally, for some $K_4(z) \in G$, $\forall i_j \in \{1, \ldots, d \}$

\begin{equation} \label{eq:SEG_cond4_small_rho}
\E \prod_{j=1}^{ 3}\left|\bar{\Delta}_{\left(i_j\right)}\right| \overset{\text{3. Lemma \ref{lemma:SEG_SDE_small_rho}}}{\leq}K_4(z) \eta^{2}.
\end{equation}
Finally, Eq.~\eqref{eq:SEG_cond1_small_rho}, Eq.~\eqref{eq:SEG_cond2_small_rho}, Eq.~\eqref{eq:SEG_cond3_small_rho}, and Eq.~\eqref{eq:SEG_cond4_small_rho} allow us to conclude the proof.

\end{proof}

\subsection{Formal Derivation - SHGD} \label{sec:formal_SHGD}
In this subsection, we present the first formal derivation of an SDE model for SHGD. We will consider the stochastic process $ Z_t \in \mathbb{R}^d $ defined as the solution of
\begin{equation}\label{eq:SHGD_SDE}
d Z_t = - F^{\shgd} \left( Z_t \right) dt +\sqrt{\eta\Sigma^{\shgd}\left( Z_t \right)}d W_{t}.
\end{equation}
with 
\begin{align}
    &F^{\shgd}(z):=\nabla \E \left[ \mathcal{H}_{\bgamma} \left( z \right)  \right],\\
    &\Sigma^{\shgd}(z):=\E \left[ \hat \xi_\bgamma(z)\hat \xi_\bgamma(z)^\top\right], \\
    & \hat \xi_\bgamma(z) = F^{\shgd}(z) - \nabla \mathcal{H}_{\bgamma} \left( z \right).
\end{align}
We remind the following equalities that will come in handy in the subsequent proofs:
\begin{align} \label{eq:SHGD_sigma_star}
 \Sigma^{\shgd}(z)=\E \left[ \left( \nabla \E \left[ \mathcal{H}_{\gamma^1, \gamma^2} \left( z \right)  \right] - \nabla \mathcal{H}_{\gamma^1, \gamma^2} \left( z \right) \right) \left( \nabla  \E \left[ \mathcal{H}_{\gamma^1, \gamma^2} \left( z \right)  \right] - \nabla \mathcal{H}_{\gamma^1, \gamma^2} \left( z \right) \right)^{\top} \right],
\end{align}
\begin{equation}
     \E \left[ \mathcal{H}_{\gamma^1, \gamma^2} \left( z \right)  \right] = \E \left[ \frac{F_{\gamma^1}^{\top}\left( z \right)F_{\gamma^2}\left( z \right)}{2}\right],  \quad \text{and} \quad \E \left[ \nabla \mathcal{H}_{\gamma^1, \gamma^2} \left( z \right)  \right] = \E \left[ \frac{ F_{\gamma^1}^{\top}\left( z \right) \nabla F_{\gamma^2}\left( z \right) + F_{\gamma^2}^{\top}\left( z \right) \nabla F_{\gamma^1}\left( z \right)}{2}\right].
\end{equation}

\begin{mybox}{gray}
\begin{theorem}[Stochastic modified equations] \label{thm:SHGD_SDE}
Let $0<\eta<1, T>0$ and set $N=\lfloor T / \eta\rfloor$. Let $ z_k \in \mathbb{R}^{2d}, 0 \leq k \leq N$ denote a sequence of SHGD iterations defined by Eq.~\eqref{eq:SHGD_Discr_Update}. Consider the stochastic process $Z_t$ defined in Eq.~\eqref{eq:SHGD_SDE} and fix some test function $g \in G$ and suppose that $g$ and its partial derivatives up to order 6 belong to $G$.

Then, under Assumption~\ref{ass:regularity_f}, there exists a constant $ C>0 $ independent of $ \eta $ such that for all $ k=0,1, \ldots, N $, we have

$$
\left|\E g\left(Z_{k \eta}\right)-\E g\left(z_k\right)\right| \leq C \eta .
$$

That is, the SDE \eqref{eq:SHGD_SDE} is an order $ 1 $ weak approximation of the SHGD iterations \eqref{eq:SHGD_Discr_Update}.
\end{theorem}
\end{mybox}

\begin{mybox}{gray}
\begin{lemma} \label{lemma:SHGD_SDE}
Under the assumptions of Theorem \ref{thm:SHGD_SDE}, let $ 0<\eta<1 $ and consider $ z_k, k \geq 0 $ satisfying the SHGD iterations \eqref{eq:SHGD_Discr_Update}
$$
z_{k+1}=z_k-\eta \nabla \mathcal{H}_{\gamma^1_k\gamma^2_k}\left(z_k \right)
$$
with $ z_0 := z = (x,y) \in \mathbb{R}^d \times \mathbb{R}^d $. From the definition the one-step difference $ \bar{\Delta}=z_1-z $, then we have

\begin{enumerate}
\item $ \E \bar{\Delta}_{i}=-\partial_{e_i} \E \left[ \mathcal{H}_{\gamma^1, \gamma^2}\right] \eta \quad \forall i = 1, \ldots,2d$;
\item $ \E \bar{\Delta}_{i} \bar{\Delta}_{j}=\partial_{e_i} \E \left[ \mathcal{H}_{\gamma^1, \gamma^2} \right] \partial_{e_j} \E \left[ \mathcal{H}_{\gamma^1, \gamma^2} \right] \eta^2 + \Sigma^{\shgd}_{(i j)} \eta^2 \quad \forall i,j = 1, \ldots,2d$;
\item $\E \prod_{j=1}^s \Bar{\Delta}_{i_j} =\mathcal{O}\left(\eta^3\right) \quad \forall s \geq 3, \quad i_j \in \{ 1, \ldots, 2d\}.$
\end{enumerate}
All the functions above are evaluated at $ z $.
\end{lemma}
\end{mybox}

\begin{proof}[Proof of Lemma \ref{lemma:SHGD_SDE}]

First of all, we have that by definition
\begin{equation}
    \E \left[ z_1- z \right] = - \eta \nabla \E \left[ \mathcal{H}_{\gamma^1, \gamma^2} \left( z \right) \right],
\end{equation}
which implies 
\begin{equation}
    \E \bar{\Delta}_{i}=-\partial_{e_i} \E \left[ \mathcal{H}_{\gamma^1, \gamma^2} \left( z \right) \right] \eta \quad \forall i = 1, \ldots,2d.
\end{equation}

Second, we have that by definition
\begin{align}
    \E \left[ \left(z_1- z\right)\left(z_1- z\right)^\top \right] & = \eta^2 \left[   \nabla \E \left[ \mathcal{H}_{\gamma^1, \gamma^2} \left( z \right) \right]\nabla \E \left[ \mathcal{H}_{\gamma^1, \gamma^2} \left( z \right) \right]^\top \right] \nonumber  \\
    & + \eta^2\E \left[ \left( \nabla \E \left[ \mathcal{H}_{\gamma^1, \gamma^2} \left(z \right)  \right] - \nabla \mathcal{H}_{\gamma^1, \gamma^2} \left( z \right) \right) \left( \nabla\E \left[ \mathcal{H}_{\gamma^1, \gamma^2} \left( z \right)  \right] - \nabla\mathcal{H}_{\gamma^1, \gamma^2} \left( z \right) \right)^{\top} \right] \nonumber  \\
    & = \eta^2 \left[   \nabla \E \left[ \mathcal{H}_{\gamma^1, \gamma^2} \left( z \right) \right]\nabla \E \left[ \mathcal{H}_{\gamma^1, \gamma^2} \left( z \right) \right]^\top \right] + \eta^2\Sigma^{\shgd}(z),
\end{align}
which implies that
\begin{equation}
    \E \bar{\Delta}_{i} \bar{\Delta}_{j}=\partial_{e_i} \E \left[ \mathcal{H}_{\gamma^1, \gamma^2} \left( z \right) \right] \partial_{e_j} \E \left[ \mathcal{H}_{\gamma^1, \gamma^2} \left( z \right) \right] \eta^2 + \Sigma^{\shgd}_{(i j)}\left( z \right) \eta^2 \quad \forall i,j = 1, \ldots,2d.
\end{equation}
Finally, by definition 

\begin{equation}
    \E \prod_{j=1}^s \Bar{\Delta}_{i_j} =\mathcal{O}\left(\eta^3\right) \quad \forall s \geq 3, \quad i_j \in \{ 1, \ldots, 2d\},
\end{equation}
which concludes our proof.
\end{proof}

\begin{proof}[Proof of Theorem \ref{thm:SHGD_SDE}] 
\label{proof:SHGD_SDE}
To prove this result, all we need to do is check the conditions in Theorem \ref{thm:mils}. As we apply Lemma \ref{lemma:li1}, we make the following choices:

\begin{itemize}
\item $b(z)=- \nabla \E \left[ \mathcal{H}_{\gamma^1, \gamma^2} \left( z \right)  \right]$;
\item $\sigma(z) =\Sigma^{\shgd}(z)^{\frac{1}{2}}$.
\end{itemize}
First of all, we notice that $\forall i = 1, \ldots, 2d$, it holds that

\begin{itemize}
\item $\E \bar{\Delta}_{i} \overset{\text{1. Lemma \ref{lemma:SHGD_SDE}}}{=}-\partial_{e_i} \E \left[ \mathcal{H}_{\gamma^1, \gamma^2} \left( z \right)  \right]\eta$;
\item $ \E \Delta_{i} \overset{\text{1. Lemma \ref{lemma:li1}}}{=} -\partial_{e_i} \E \left[ \mathcal{H}_{\gamma^1, \gamma^2} \left( z \right)  \right]\eta +\mathcal{O}\left(\eta^2\right)$.
\end{itemize}
Therefore, we have that for some $K_1(z) \in G$,
\begin{equation}\label{eq:SHGD_cond1}
\left|\E \Delta_{i}-\E \bar{\Delta}_{i}\right| \leq K_1(z) \eta^{2}, \quad \forall i = 1, \ldots,2d.
\end{equation}
Additionally, we notice that $\forall i,j = 1, \ldots, d$, it holds that
\begin{itemize}
\item $ \E \bar{\Delta}_{i} \bar{\Delta}_{j} \overset{\text{2. Lemma \ref{lemma:SHGD_SDE}}}{=}\partial_{e_i} \E \left[ \mathcal{H}_{\gamma^1, \gamma^2}\left( z \right) \right] \partial_{e_j} \E \left[ \mathcal{H}_{\gamma^1, \gamma^2}\left( z \right)  \right] \eta^2+ \Sigma^{\shgd}_{(i j)}\left( z \right) \eta^2$;
\item $ \E \Delta_{i} \Delta_{j} \overset{\text{2. Lemma \ref{lemma:li1}}}{=}\partial_{e_i} \E \left[ \mathcal{H}_{\gamma^1, \gamma^2}\left( z \right)  \right] \partial_{e_j} \E \left[ \mathcal{H}_{\gamma^1, \gamma^2}\left( z \right)  \right] \eta^2+ \Sigma^{\shgd}_{(i j)}\left( z \right) \eta^2 + \mathcal{O}\left(\eta^3 \right)$.
\end{itemize}
Therefore, we have that for some $K_2(z) \in G$,
\begin{equation}\label{eq:SHGD_cond2}
\left|\E \Delta_{i} \Delta_{j} - \E \bar{\Delta}_{i} \bar{\Delta}_{j}\right| \leq K_2(z) \eta^{2}, \quad \forall i,j = 1, \ldots, 2d.
\end{equation}
Additionally, we notice that $\forall s \geq 3, \forall i_j \in \{1, \ldots, 2d \}$, it holds that
\begin{itemize}
\item $ \E \prod_{j=1}^s \bar{\Delta}_{i_j}\overset{\text{3. Lemma \ref{lemma:SHGD_SDE}}}{=}\mathcal{O}\left(\eta^3\right)$;
\item $ \E \prod_{j=1}^s \Delta_{i_j}\overset{\text{3. Lemma \ref{lemma:li1}}}{=}\mathcal{O}\left(\eta^3\right)$.
\end{itemize}
Therefore, we have that for some $K_3(z) \in G$,
\begin{equation}\label{eq:SHGD_cond3}
\left|\E \prod_{j=1}^s \Delta_{i_j}-\E \prod_{j=1}^s \bar{\Delta}_{i_j}\right| \leq K_3(z) \eta^{2}.
\end{equation}
Additionally, for some $K_4(z) \in G$, $\forall i_j \in \{1, \ldots, d \}$,
\begin{equation} \label{eq:SHGD_cond4}
\E \prod_{j=1}^{ 3}\left|\bar{\Delta}_{\left(i_j\right)}\right| \overset{\text{3. Lemma \ref{lemma:SHGD_SDE}}}{\leq}K_4(z) \eta^{2}.
\end{equation}
Finally, Eq.~\eqref{eq:SHGD_cond1}, Eq.~\eqref{eq:SHGD_cond2}, Eq.~\eqref{eq:SHGD_cond3}, and Eq.~\eqref{eq:SHGD_cond4} allow us to conclude the proof.

\end{proof}

\begin{mybox}{gray}
\begin{corollary}\label{thm:SHGD_SDE_Simplified_SameSample}
Under the assumptions of Theorem~\ref{thm:SHGD_SDE}. Additionally, let us assume that $\gamma^1=\gamma^2=\gamma$, the stochastic gradients are $\nabla_x f_{\gamma}(z) = \nabla_x f(z) + U^x$ and $\nabla_y f_{\gamma}(z) = \nabla_y f(z) + U^y$ such that $U^x$ and $U^y$ are independent noises that do not depend on $z$, whose expectation is $0$, and whose covariance matrix is $\Sigma$. Therefore, the SDE is:
\begin{equation}\label{eq:SHGD_SDE_Simplified_SameSample}
d Z_t = - \nabla \mathcal{H} \left( Z_t \right) dt + \sqrt{\eta} \nabla^2 f \left( Z_t \right) \sqrt{\Sigma} d W_{t}.
\end{equation}
\end{corollary}
\end{mybox}

\begin{proof}[Proof of Corollary \ref{thm:SHGD_SDE_Simplified_SameSample}]

First of all, we notice that
\begin{align}
\E \left[ \mathcal{H}_{\gamma} \left( Z_t \right)  \right] & = \E \left[ \frac{\lVert \nabla_x f_{\gamma}\left( Z_t \right) \rVert_2^2 + \lVert \nabla_y f_{\gamma}\left( Z_t \right) \rVert_2^2}{2}\right] = \E \left[ \frac{\lVert \nabla_x f\left( Z_t \right) \rVert_2^2 + \lVert \nabla_y f\left( Z_t \right) \rVert_2^2}{2}\right] \nonumber  \\
& + \frac{\E \left[ (U^x)(U^x)^{\top}\right]+\E \left[ (U^y)(U^y)^{\top}\right]}{2}  = \mathcal{H} \left( Z_t \right) +  \frac{\E \left[ (U^x)(U^x)^{\top}\right]+\E \left[ (U^y)(U^y)^{\top}\right]}{2}.
\end{align}
Since $ \frac{\E \left[ (U^x)(U^x)^{\top}\right]+\E \left[ (U^y)(U^y)^{\top}\right]}{2}$ is independent on $z$, we ignore it as its gradient is $0$.

Second, based on our assumption of the noise structure, we can rewrite Eq. \eqref{eq:SHGD_sigma_star} of the matrix $ \Sigma^{\shgd}\left( z \right)$ as
\begin{align}
\E \left[ \left( F^{\top}\left( z \right) \nabla F\left( z \right) - \frac{F_{\gamma^1}^{\top}\left( z \right) + F_{\gamma^2}^{\top}\left( z \right)}{2} \nabla F\left( z \right)\right) \left( F^{\top}\left( z \right) \nabla F\left( z \right) - \frac{F_{\gamma^1}^{\top}\left( z \right) + F_{\gamma^2}^{\top}\left( z \right)}{2} \nabla F\left( z \right)\right)^{\top} \right].
\end{align}
Since $\gamma^1 = \gamma^2 = \gamma$, and noticing that $F^{\top}\left( z \right) \nabla F\left( z \right) = \nabla^2 f\left( z \right) \nabla f\left( z \right)$, we have
\begin{equation}
    \Sigma^{\shgd}\left( z \right) = \nabla^2 f\left( z \right) \Sigma \nabla^2 f\left( z \right),
\end{equation}
which implies that
\begin{equation}
d Z_t = - \nabla \mathcal{H} \left( Z_t \right) dt + \sqrt{\eta} \nabla^2 f \left( Z_t \right) \sqrt{\Sigma} d W_{t}.
\end{equation}
\end{proof}

\begin{mybox}{gray}
\begin{corollary}\label{thm:SHGD_SDE_Simplified_IndepSample}
Under the assumptions of Theorem~\ref{thm:SHGD_SDE}. Additionally, let us assume that $\gamma^1$ and $\gamma^2$, are independent and the stochastic gradients can be written as $\nabla_x f_{\gamma^i}(z) = \nabla_x f(z) + U_x^i$ and $\nabla_y f_{\gamma^i}(z) = \nabla_y f(z) + U_y^i$ such that $U_x^i$ and $U_y^i$ are independent noises that do not depend on $z$. Therefore, the SDE is:
\begin{equation}\label{eq:SHGD_SDE_Simplified_IndepSample}
d Z_t = - \nabla \mathcal{H} \left( Z_t \right) dt + \sqrt{\frac{\eta}{2}} \nabla^2 f \left( Z_t \right) \sqrt{\Sigma} d W_{t}.
\end{equation}
\end{corollary}
\end{mybox}

\begin{proof}[Proof of Corollary \ref{thm:SHGD_SDE_Simplified_IndepSample}]

First of all, we notice that
\begin{equation}
    \E \left[ \nabla \mathcal{H}_{\gamma^1, \gamma^2} \left( z \right)  \right] := \E \left[ \frac{ F_{\gamma^1}^{\top}\left( z \right) \nabla F_{\gamma^2}\left( z \right) + F_{\gamma^2}^{\top}\left( z \right) \nabla F_{\gamma^1}\left( z \right)}{2}\right] = F^{\top}\left( z \right) \nabla F\left( z \right) = \nabla \mathcal{H}\left( z \right).
\end{equation}
Second, based on our assumption of the noise structure, we can rewrite Eq. \eqref{eq:SHGD_sigma_star} of the matrix $ \Sigma^{\shgd}\left( z \right)$ as
\begin{align}
\E \left[ \left( F^{\top}\left( z \right) \nabla F\left( z \right) - \frac{F_{\gamma^1}^{\top}\left( z \right) + F_{\gamma^2}^{\top}\left( z \right)}{2} \nabla F\left( z \right)\right) \left( F^{\top}\left( z \right) \nabla F\left( z \right) - \frac{F_{\gamma^1}^{\top}\left( z \right) + F_{\gamma^2}^{\top}\left( z \right)}{2} \nabla F\left( z \right)\right)^{\top} \right].
\end{align}
Since $\gamma^1$ and $\gamma^2$ are independent, and noticing that $F^{\top}\left( z \right) \nabla F\left( z \right) = \nabla^2 f\left( z \right) \nabla f\left( z \right)$, we have

\begin{equation}
    \Sigma^{\shgd}\left( z \right) = \frac{1}{2} \nabla^2 f\left( z \right) \Sigma \nabla^2 f\left( z \right),
\end{equation}
which implies that
\begin{equation}
d Z_t = - \nabla \mathcal{H} \left( Z_t \right) dt + \sqrt{\frac{\eta}{2}} \nabla^2 f \left( Z_t \right) \sqrt{\Sigma} d W_{t}.
\end{equation}
\end{proof}

\section{BILINEAR GAMES - INSIGHTS}\label{sec:PIBG_insights}
In this section, we study Bilinear Games of the form $f(x,y) = x^{\top} \mathbf{\Lambda} y$,  where $\mathbf{\Lambda}$ is a square, diagonal, and positive semidefinite matrix.
\subsection{SEG}
\begin{theorem}[Exact Dynamics of SEG] \label{thm:SEG_Dynamic_PIBG}
Under the assumptions of Corollary \ref{thm:SEG_SDE_Simplified_SameSample}, for $f(x,y) = x^{\top} \mathbf{\Lambda} y$ and noise covariance matrices equal to $\sigma\mathbf{I}_d$, we have that
    \begin{equation}
    Z_t = \mathbf{\mathbf{\Tilde{E}}}(t) \mathbf{\mathbf{\Tilde{R}}}(t) \left( z + \sqrt{\eta} \sigma \int_{0}^t \mathbf{\mathbf{\Tilde{E}}}(-s) \mathbf{\mathbf{\Tilde{R}}}(-s)  \mathbf{M} d W_s\right),
\end{equation}
    with $\mathbf{\mathbf{\Tilde{E}}}(t)= \left[\begin{array}{ll} \mathbf{E}(t) & \mathbf{0}_d  \\  \mathbf{0}_d  & \mathbf{E}(t) \end{array}\right],\mathbf{\mathbf{\Tilde{R}}}(t)=\left[\begin{array}{ll} \mathbf{C}(t) & -\mathbf{S}(t)  \\  \mathbf{S}(t)  & \mathbf{C}(t) \end{array}\right]$, and $\mathbf{M}=\left[\begin{array}{ll} \mathbf{I}_d & - \rho \mathbf{\Lambda}  \\ \rho \mathbf{\Lambda}  & \mathbf{I}_d \end{array}\right] $,
where 
\begin{equation}
    \mathbf{E}(t):= \diag{\left( e^{-\rho \lambda_{1}^{2} t}, \cdots, e^{-\rho \lambda_{d}^{2} t} \right)},
\end{equation}
\begin{equation}
    \mathbf{C}(t):= \diag{\left( \cos{(\lambda_1 t)}, \cdots, \cos{(\lambda_d t)}\right)},
\end{equation}
and 
\begin{equation}
    \mathbf{S}(t):= \diag{\left( \sin{(\lambda_1 t)}, \cdots, \sin{(\lambda_d t)}\right)}.
\end{equation}
In particular, we have that

\begin{enumerate}
    \item $\E \left[ Z_t\right] = \mathbf{\mathbf{\Tilde{E}}}(t) \mathbf{\mathbf{\Tilde{R}}}(t) z \overset{t \rightarrow \infty}{=} 0$;
    \item The covariance matrix of is equal to \begin{equation}
    \eta \sigma^2 \left[\begin{array}{ll} \mathbf{I}_d - \mathbf{E}(2t) & \mathbf{0}_d  \\  \mathbf{0}_d  &  \mathbf{I}_d - \mathbf{E}(2t) \end{array}\right] \bar{\Sigma} \overset{t \rightarrow \infty}{=} \eta \sigma^2 \bar{\Sigma},
\end{equation}
where
\begin{equation}
    \bar{\Sigma} := \left[\begin{array}{ll} \mathbf{B}&  \mathbf{0}_d  \\  \mathbf{0}_d   & \mathbf{B}  \end{array}\right],
\end{equation}
and $\mathbf{B}:= \diag{\left( \frac{1 + \rho^2 \lambda_1^2}{2 \rho \lambda_1^2}, \cdots, \frac{1 + \rho^2 \lambda_d^2}{2 \rho \lambda_d^2} \right)}$;
\item If $\rho=0$, SGDA would indeed diverge.
\end{enumerate}
\end{theorem}

\begin{proof}
The SDEs of SEG are:
\begin{align}
    d X_t = - \mathbf{\Lambda} Y_t dt - \rho\mathbf{\Lambda} ^2 X_t dt + \sqrt{\eta}\mathbf{I}_d \sigma d W^x_t - \sqrt{\eta} \sigma \rho\mathbf{\Lambda}  d W^y_t
\end{align}
and
\begin{align}
    d Y_t = +\mathbf{\Lambda} X_t dt - \rho\mathbf{\Lambda}^2 Yt dt + \sqrt{\eta} \sigma\mathbf{I}_d W^y_t + \sqrt{\eta} \sigma \rho\mathbf{\Lambda} d W^x_t,
\end{align}
which can be rewritten as
\begin{align}
    d Z_t = \mathbf{A} Z_t dt + \sqrt{\eta} \sigma \mathbf{B} d W_t,
\end{align}
where
\begin{equation}
     \mathbf{A} = \left[\begin{array}{ll} - \rho\mathbf{\Lambda}^2 & -\mathbf{\Lambda} \\ \mathbf{\Lambda}  & - \rho\mathbf{\Lambda}^2  \end{array}\right] \quad \text{ and } \quad \mathbf{B} = \left[\begin{array}{ll}\mathbf{I}_d & - \rho\mathbf{\Lambda}  \\ \rho\mathbf{\Lambda}  &\mathbf{I}_d  \end{array}\right].
\end{equation}
Therefore, the solution is
\begin{equation}
    Z_t = e^{\mathbf{A} t} \left( z + \sqrt{\eta} \sigma \int_{0}^{t} e^{-\mathbf{A} s} \mathbf{B} d W_s \right).
\end{equation}
We observe that 
$\mathbf{A} = \mathbf{A}_1 + \mathbf{A}_2$ s.t.
\begin{equation}
    \mathbf{A}_1= \left[\begin{array}{ll} - \rho\mathbf{\Lambda}^2 & \mathbf{0}_d  \\  \mathbf{0}_d  & - \rho\mathbf{\Lambda}^2  \end{array}\right] \quad \text{ and } \quad \mathbf{A}_2= \left[\begin{array}{ll} \mathbf{0}_d & -\mathbf{\Lambda} \\ \mathbf{\Lambda}  & \mathbf{0}_d \end{array}\right],
\end{equation}
and that since these two matrix commute, $e^{\mathbf{A}t} = e^{\mathbf{A}_1 t} e^{\mathbf{A}_2 t}$. Clearly, we have that
\begin{equation}
    \mathbf{\mathbf{\Tilde{E}}}(t):=e^{\mathbf{A}_1 t} = \diag{\left( e^{-\rho \lambda_{1}^{2} t}, \cdots, e^{-\rho \lambda_{d}^{2} t}, e^{-\rho \lambda_{1}^{2} t}, \cdots, e^{-\rho \lambda_{d}^{2} t}  \right)} = \left[\begin{array}{ll} \mathbf{E}(t) & \mathbf{0}_d  \\  \mathbf{0}_d  & \mathbf{E}(t) \end{array}\right],
\end{equation}
where $\mathbf{E}(t):= \diag{\left( e^{-\rho \lambda_{1}^{2} t}, \cdots, e^{-\rho \lambda_{d}^{2} t} \right)}$.

Regarding $\mathbf{A}_2$, we observe that

\begin{equation}
    (\mathbf{A}_2 t)^{2k} = \diag{\left( (\lambda_1 t)^{2k} (-1)^{k}, \cdots, (\lambda_d t)^{2k} (-1)^{k}, (\lambda_1 t)^{2k} (-1)^{k}, \cdots, (\lambda_d t)^{2k} (-1)^{k} \right)}
\end{equation}
and that
\begin{equation}
    (\mathbf{A}_2 t)^{2k+1} = \left[\begin{array}{ll} \mathbf{0_d} & \mathbf{P}  \\  \mathbf{Q}  & \mathbf{0_d}  \end{array}\right] ,
\end{equation}
where
\begin{equation}
    \mathbf{P}:= \diag{\left( (\lambda_1 t)^{2k+1} (-1)^{k+1}, \cdots, (\lambda_d t)^{2k+1} (-1)^{k+1} \right)}
\end{equation}
and
\begin{equation}
    \mathbf{Q}:= \diag{\left( (\lambda_1 t)^{2k+1} (-1)^{k}, \cdots, (\lambda_d t)^{2k+1} (-1)^{k} \right)}.
\end{equation}
Therefore,
\begin{align}
    \mathbf{\mathbf{\Tilde{R}}}(t):= e^{\mathbf{A}_2 t} & = \sum_{k=0}^{\infty} \frac{(\mathbf{A}_2 t)^{2k}}{(2k)!} + \sum_{k=0}^{\infty} \frac{(\mathbf{A}_2 t)^{2k+1}}{(2k+1)!}  \nonumber \\ 
    & = \left[\begin{array}{ll} \mathbf{C}(t) & \mathbf{0}_d  \\  \mathbf{0}_d  & \mathbf{C}(t) \end{array}\right] + \left[\begin{array}{ll} \mathbf{0}_d & -\mathbf{S}(t)  \\  \mathbf{S}(t)  & \mathbf{0}_d \end{array}\right] = \left[\begin{array}{ll} \mathbf{C}(t) & -\mathbf{S}(t)  \\  \mathbf{S}(t)  & \mathbf{C}(t) \end{array}\right],
\end{align}
where
\begin{equation}
    \mathbf{C}(t):= \diag{\left( \cos{(\lambda_1 t)}, \cdots, \cos{(\lambda_d t)}\right)}
\end{equation}
and
\begin{equation}
    \mathbf{S}(t):= \diag{\left( \sin{(\lambda_1 t)}, \cdots, \sin{(\lambda_d t)}\right)}.
\end{equation}
Automatically, we get that
\begin{equation}
    e^{-\mathbf{A}_1 s} = \left[\begin{array}{ll} \mathbf{E}(-s) & \mathbf{0}_d  \\  \mathbf{0}_d  & \mathbf{E}(-s) \end{array}\right]
\end{equation}
and 
\begin{align}
    e^{-\mathbf{A}_2 s} = \left[\begin{array}{ll} \mathbf{C}(s) & \mathbf{S}(s)  \\  -\mathbf{S}(s)  & \mathbf{C}(s) \end{array}\right],
\end{align}
which imply that
\begin{align}
    Z_t & = \left[\begin{array}{ll} \mathbf{E}(t) & \mathbf{0}_d  \\  \mathbf{0}_d  & \mathbf{E}(t) \end{array}\right] \left[\begin{array}{ll} \mathbf{C}(t) & -\mathbf{S}(t)  \\  \mathbf{S}(t)  & \mathbf{C}(t) \end{array}\right] \bigg( z \nonumber \\
    &  + \sqrt{\eta} \sigma \int_{0}^t \left[\begin{array}{ll} \mathbf{E}(-s) & \mathbf{0}_d  \\  \mathbf{0}_d  & \mathbf{E}(-s) \end{array}\right] \left[\begin{array}{ll} \mathbf{C}(s) & \mathbf{S}(s)  \\  -\mathbf{S}(s)  & \mathbf{C}(s) \end{array}\right]  \left[\begin{array}{ll}\mathbf{I}_d & - \rho\mathbf{\Lambda}  \\ \rho\mathbf{\Lambda}  &\mathbf{I}_d  \end{array}\right] \left[\begin{array}{l} d W^x_s \\ d W^y_s  \end{array}\right] \bigg).
\end{align}
To conclude, we have that
    \begin{equation}
    Z_t = \mathbf{\mathbf{\Tilde{E}}}(t) \mathbf{\mathbf{\Tilde{R}}}(t) \left( z + \sqrt{\eta} \sigma \int_{0}^t \mathbf{\mathbf{\Tilde{E}}}(-s) \mathbf{\mathbf{\Tilde{R}}}(-s)  \mathbf{M} d W_s\right),
\end{equation}
where $\mathbf{M}=\left[\begin{array}{ll}\mathbf{I}_d & - \rho \mathbf{\Lambda}  \\ \rho \mathbf{\Lambda}  &\mathbf{I}_d  \end{array}\right] $.

We observe that since the expected value of the noise terms is $0$, we have that
\begin{equation}
    \E[Z_t] = \left[\begin{array}{ll} \mathbf{E}(t) & \mathbf{0}_d  \\  \mathbf{0}_d  & \mathbf{E}(t) \end{array}\right] \left[\begin{array}{ll} \mathbf{C}(t) & -\mathbf{S}(t)  \\  \mathbf{S}(t)  & \mathbf{C}(t) \end{array}\right] z.
\end{equation}
Therefore, the expectation of $Z_t$ converges to $0$ exponentially fast, while spiraling around the origin. We observe that larger values of $\rho$ encourage a faster convergence of $\E[Z_t]$ to $0$.

Let us now have a look at the covariance matrix of this process:

$Var (Z_t) = \eta \sigma^2 \left[\begin{array}{ll} \mathbf{E}(2t) & \mathbf{0}_d  \\  \mathbf{0}_d  & \mathbf{E}(2t) \end{array}\right] \mathbf{\mathbf{\Tilde{R}}}(t) Var(V_t)  \mathbf{\mathbf{\Tilde{R}}}(t)^{\top}$, where

\begin{align}
    V_t & := \int_{0}^t \left[\begin{array}{ll} \mathbf{E}(-s) & \mathbf{0}_d  \\  \mathbf{0}_d  & \mathbf{E}(-s) \end{array}\right] \left[\begin{array}{ll} \mathbf{C}(s) & \mathbf{S}(s)  \\  -\mathbf{S}(s)  & \mathbf{C}(s) \end{array}\right]  \left[\begin{array}{ll}\mathbf{I}_d & - \rho\mathbf{\Lambda}  \\ \rho\mathbf{\Lambda}  & \mathbf{I}_d  \end{array}\right] \left[\begin{array}{l} d W^x_s \\ d W^y_s  \end{array}\right]  \nonumber  \\
     & = \int_{0}^t \left[\begin{array}{ll} \mathbf{E}(-s) (\mathbf{C}(s)  + \rho\mathbf{\Lambda} \mathbf{S}(s) ) & \mathbf{E}(-s) (\mathbf{S}(s)  - \rho\mathbf{\Lambda} \mathbf{C}(s))   \\  \mathbf{E}(-s) (\rho\mathbf{\Lambda} \mathbf{C}(s)-\mathbf{S}(s) )  &  \mathbf{E}(-s) (\mathbf{C}(s)  + \rho\mathbf{\Lambda} \mathbf{S}(s) ) \end{array}\right]  \left[\begin{array}{l} d W^x_s \\ d W^y_s  \end{array}\right]  \nonumber \\
    & = \left[\begin{array}{l} \int_{0}^t  e^{\rho \lambda_1^2 s }  ( \cos{(\lambda_1 s)} + \rho \lambda \sin{(\lambda_1 s)} ) d W^{x_1}_s + \int_{0}^t  e^{\rho \lambda_1^2 s }  ( \sin{(\lambda_1 s)} - \rho \lambda_1 \cos{(\lambda_1 s)} ) d W^{y_1}_s \\ \vdots \\ \int_{0}^t  e^{\rho \lambda_d^2 s }  ( \cos{(\lambda_d s)} + \rho \lambda \sin{(\lambda_d s)} ) d W^{x_d}_s + \int_{0}^t  e^{\rho \lambda_d^2 s }  ( \sin{(\lambda_d s)} - \rho \lambda_d \cos{(\lambda_d s)} ) d W^{y_d}_s  \\ \int_{0}^t  e^{\rho \lambda_1^2 s }  (- \sin{(\lambda_1 s)} + \rho \lambda \cos{(\lambda_1 s)} ) d W^{x_1}_s + \int_{0}^t  e^{\rho \lambda_1^2 s }  (\cos{(\lambda_1 s)} + \rho \lambda_1 \sin{(\lambda_1 s)} ) d W^{y_1}_s \\ \vdots \\ \int_{0}^t  e^{\rho \lambda_d^2 s }  (- \sin{(\lambda_d s)} + \rho \lambda \cos{(\lambda_d s)} ) d W^{x_d}_s + \int_{0}^t  e^{\rho \lambda_d^2 s }  (\cos{(\lambda_d s)} + \rho \lambda_d \sin{(\lambda_d s)} ) d W^{y_d}_s \end{array}\right] \nonumber  \\
    &  =:  \left[\begin{array}{l} a^{x_1}_1(t) + a^{y_1}_2(t) \\ \vdots \\ a^{x_d}_1(t) + a^{y_d}_2(t) \\ a^{x_1}_3(t) + a^{y_1}_4(t) \\ \vdots \\ a^{x_d}_3(t) + a^{y_d}_4(t)  \end{array}\right].
\end{align}
Therefore,
\begin{equation}
    Var(V_t) = \left[\begin{array}{ll} \mathbf{V}^{1,2}(t)  &  \mathbf{C}^{1,2,3,4}(t) \\  \mathbf{C}^{1,2,3,4}(t)  & \mathbf{V}^{3,4}(t)  \end{array}\right],
\end{equation}
such that
\begin{equation}
    \mathbf{V}^{1,2}_{i,i}(t) = Var(a^{x_i}_1(t)) + Var(a^{y_i}_2(t)), \quad \forall i \in \{1, \cdots, d \},
\end{equation}
\begin{equation}
    \mathbf{V}^{3,4}_{i,i}(t) = Var(a^{x_i}_3(t)) + Var(a^{y_i}_4(t)), \quad \forall i \in \{1, \cdots, d \},
\end{equation}
and 
\begin{equation}
    \mathbf{C}^{1,2,3,4}_{i,i}(t) = Cov(a^{x_i}_1(t),a^{x_i}_3(t)) + Cov(a^{y_i}_1(t),a^{y_i}_3(t)), \quad \forall i \in \{1, \cdots, d \}.
\end{equation}
Using the well-known Itô Isometry:
$$\E\left[\left(\int_0^t H_s d W_s\right)^2\right]=\E\left[\int_0^t H_s^2 d s\right],$$
we get that
\begin{align}
    & Var(a^{x_i}_1(t)) + Var(a^{y_i}_2(t)) =  \int_{0}^t  e^{2 \rho \lambda_i^2 s }  ( \cos{(\lambda_i s)} + \rho \lambda_i \sin{(\lambda_i s)} )^2 d s + \int_{0}^t  e^{2\rho \lambda_i^2 s }  ( \sin{(\lambda_i s)} - \rho \lambda_i \cos{(\lambda_i s)} )^2 d s \nonumber  \\
    & =  \int_{0}^t  e^{2 \rho \lambda_i^2 s } \left[  ( \cos{(\lambda_i s)} + \rho \lambda_i \sin{(\lambda_i s)} )^2 + ( \sin{(\lambda_i s)} - \rho \lambda_i \cos{(\lambda_i s)} )^2 \right] d s \nonumber  \\
    & = \int_{0}^t  e^{2 \rho \lambda_i^2 s } \left[  1 + \rho^2 \lambda_i^2 \right] d s \nonumber  \\
    & = \frac{1 + \rho^2 \lambda_i^2}{2 \rho \lambda_i^2} \left( e^{2 \rho \lambda_i^2 t }  -1\right).
\end{align}
We observe that if $\rho = 0$, this quantity is equal to $t$.

Then, we do a similar calculation:
\begin{align}
    & Var(a^{x_i}_3(t)) + Var(a^{y_i}_4(t)) =  \int_{0}^t  e^{2 \rho \lambda_i^2 s }  (- \sin{(\lambda_i s)} + \rho \lambda_i \cos{(\lambda_i s)} )^2 d s + \int_{0}^t  e^{2\rho \lambda_i^2 s }  (\cos{(\lambda_i s)} + \rho \lambda_i \sin{(\lambda_i s)} )^2 d s \nonumber  \\
    & =  \int_{0}^t  e^{2 \rho \lambda_i^2 s } \left[ (- \sin{(\lambda_i s)} + \rho \lambda_i \cos{(\lambda_i s)} )^2 + (\cos{(\lambda_i s)} + \rho \lambda_i \sin{(\lambda_i s)} )^2 \right] d s \nonumber  \\
    & = \int_{0}^t  e^{2 \rho \lambda_i^2 s } \left[  1 + \rho^2 \lambda_i^2 \right] d s \nonumber  \\
    & = \frac{1 + \rho^2 \lambda_i^2}{2 \rho \lambda_i^2} \left( e^{2 \rho \lambda_i^2 t }  -1\right).
\end{align}
We observe that if $\rho = 0$, also this quantity is equal to $t$.

Remembering now that
$$\E\left[\left(\int_0^t X_s d W_s\right)\left(\int_0^t Y_s d W_s\right)\right]=\E\left[\int_0^t X_s Y_s d s\right],$$
we have that 
\begin{align}
    Cov(a^{x_i}_1(t),a^{x_i}_3(t)) +Cov(a^{y_i}_2(t),a^{y_i}_4(t)) & =  \int_{0}^t  e^{2 \rho \lambda_i^2 s }  (\cos{(\lambda_i s)} + \rho \lambda_i \sin{(\lambda_i s)} )(- \sin{(\lambda_i s)} + \rho \lambda_i \cos{(\lambda_i s)} ) d s \nonumber  \\
    & + \int_{0}^t  e^{2\rho \lambda_i^2 s }  (\sin{(\lambda_i s)} - \rho \lambda_i \cos{(\lambda_i s)} )(\cos{(\lambda_i s)} + \rho \lambda_i \sin{(\lambda_i s)} ) d s = 0.
\end{align}

To conclude, the covariance matrix of $Z_t$ is
\begin{equation}
    Var(Z_t) = \eta \sigma^2 \left[\begin{array}{ll} \mathbf{I}_d - \mathbf{E}(2t) & \mathbf{0}_d  \\  \mathbf{0}_d  &  \mathbf{I}_d - \mathbf{E}(2t) \end{array}\right] \bar{\Sigma} \overset{t \rightarrow \infty}{=} \eta \sigma^2 \bar{\Sigma},
\end{equation}
where
\begin{equation}
    \bar{\Sigma} := \left[\begin{array}{ll} \mathbf{B}&  \mathbf{0}_d  \\  \mathbf{0}_d   & \mathbf{B}  \end{array}\right]
\end{equation}
and $\mathbf{B}:= \diag{\left( \frac{1 + \rho^2 \lambda_1^2}{2 \rho \lambda_1^2}, \cdots, \frac{1 + \rho^2 \lambda_d^2}{2 \rho \lambda_d^2} \right)}$.

Of course, if $\rho =0$ the covariance matrix is actually $\eta \sigma^2  t \mathbf{I}_d$, meaning that the variance of SGDA diverges.
\end{proof}

\begin{lemma}
    Let us define the variance $\mathbf{B}_{i,i}(\rho) = \frac{1 + \rho^2 \lambda_i^2}{2 \rho \lambda_i^2} $ and consider it as a function of $\rho$. The following hold:
    \begin{enumerate}
        \item $\lim_{\rho \rightarrow 0} \mathbf{B}_{i,i}(\rho) = \infty$;
        \item $\lim_{\rho \rightarrow \infty} \mathbf{B}_{i,i}(\rho) = \infty$;
        \item $\mathbf{B}_{i,i}(\rho)$ is convex in $\rho$;
        \item $\rho = \frac{1}{\lambda_i}$ realizes the minimum and $\mathbf{B}_{i,i}\left(\frac{1}{\lambda_i}\right) = \frac{1}{\lambda_i}$;
        \item The trace of $\bar{\Sigma}$ is minimized by $\rho = \sqrt{\frac{\sum \frac{1}{\lambda_i^2}}{d}}$.
    \end{enumerate}
\end{lemma}
\begin{proof}
    The first four points are obvious while we spell out the last one. We observe that the trace of $\mathbf{B}$ is convex as the sum of convex functions and its derivative w.r.t. $\rho$ is
    \begin{equation}
        \frac{d}{d \rho} \left(\sum_{i=1}^{d} \frac{1 + \rho^2 \lambda_i^2}{2 \rho \lambda_i^2} \right) = \sum_{i=1}^{d} \frac{1}{2} - \frac{1}{2 \rho^2 \lambda_i^2},
    \end{equation}
    which implies that the optimal $\rho$ is indeed  $\rho = \sqrt{\frac{\sum \frac{1}{\lambda_i^2}}{d}}$.
\end{proof}

\paragraph{Insights - The trade-off in selecting $\rho$}
The curvature of the landscape influences the speed of convergence. Indeed, larger values of $\lambda_i$, which correspond to stronger interaction, speed up the exponential decay in the expected value of the iterates.
Additionally, $\rho$ impacts the convergence speed in expectation as larger values boost such a decay.
However, the peculiar way in which the noise and the landscape interact implies that larger values of $\rho$ might actually result in larger asymptotic variance.
One observes that both $\rho \rightarrow 0$ and $\rho \rightarrow \infty$ result in infinite asymptotic variance. On the bright side, $\rho_i=\frac{1}{\lambda_i}$ is the optimal choice to reduce the variance along the $i$-th dimension. Unfortunately, this could possibly be very small and thus slow down the convergence. Finally, if one can only select a single $\rho$ across all parameters, then one might want to minimize the trace of the covariance matrix using $\rho= \sqrt{\frac{\sum \frac{1}{\lambda_i^2}}{d}}$.

\subsection{SHGD}
\begin{theorem}[Exact Dynamics of SHGD] \label{thm:SHGD_Dynamic_PIBG}
Under the assumptions of Corollary \ref{thm:SHGD_SDE_Simplified_SameSample}, for $f(x,y) = x^{\top} \mathbf{\Lambda} y$ and noise covariance matrices equal to $\sigma\mathbf{I}_d$, we have that
    \begin{equation}
    Z_t = \mathbf{\Tilde{E}}(t) \left( z + \sqrt{\eta} \sigma \int_{0}^t \mathbf{\Tilde{E}}(-s) \mathbf{M} d W_s\right),
\end{equation}
    $\mathbf{\Tilde{E}}(t)= \left[\begin{array}{ll} \mathbf{E}(t) & \mathbf{0}_d  \\  \mathbf{0}_d  & \mathbf{E}(t) \end{array}\right]$, $\mathbf{M}=\left[\begin{array}{ll} \mathbf{0}_d &  \mathbf{\Lambda}  \\ \mathbf{\Lambda}  & \mathbf{0}_d  \end{array}\right] $,
where 
\begin{equation}
    \mathbf{E}(t):= \diag{\left( e^{-\lambda_{1}^{2} t}, \cdots, e^{-\lambda_{d}^{2} t} \right)}.
\end{equation}
In particular, we have that
\begin{enumerate}
    \item $\E \left[ Z_t\right] = \mathbf{\Tilde{E}}(t) z \overset{t \rightarrow \infty}{=} 0$;
    \item The covariance matrix of $Z_t$ is equal to \begin{equation}
    \eta \frac{\sigma^2 }{2}\left[\begin{array}{ll} \mathbf{I}_d - \mathbf{E}(2t) & \mathbf{0}_d  \\  \mathbf{0}_d  &  \mathbf{I}_d - \mathbf{E}(2t) \end{array}\right] \bar{\Sigma} \overset{t \rightarrow \infty}{=} \eta \sigma^2 \bar{\Sigma},
\end{equation}
where
\begin{equation}
    \bar{\Sigma} := \left[\begin{array}{ll} \mathbf{I}_d &  \mathbf{0}_d  \\  \mathbf{0}_d   & \mathbf{I}_d  \end{array}\right].
\end{equation}
\end{enumerate}
\end{theorem}

\begin{proof}

The SDEs of SHGD are:
\begin{align}
    d X_t = -\mathbf{\Lambda} ^2 X_t dt + \sqrt{\eta} \sigma\mathbf{\Lambda}  d W^y_t
\end{align}
and
\begin{align}
    d Y_t = -\mathbf{\Lambda}^2 Yt dt + \sqrt{\eta} \sigma\mathbf{\Lambda}  d W^x_t,
\end{align}
which can be rewritten as
\begin{align}
    d Z_t = \mathbf{A} Z_t dt + \sqrt{\eta} \sigma \mathbf{B} d W_t,
\end{align}
where
\begin{equation}
    \mathbf{A}= \left[\begin{array}{ll} -\mathbf{\Lambda}^2 & \mathbf{0}_d  \\  \mathbf{0}_d  &  -\mathbf{\Lambda}^2  \end{array}\right] \quad \text{ and } \quad \mathbf{B} = \left[\begin{array}{ll} \mathbf{0}_d &\mathbf{\Lambda}  \\\mathbf{\Lambda}  & \mathbf{0}_d  \end{array}\right].
\end{equation}
Therefore, the solution is
\begin{equation}
    Z_t = e^{\mathbf{A} t} \left( z + \sqrt{\eta} \sigma \int_{0}^{t} e^{-\mathbf{A} s} \mathbf{B} d W_s \right).
\end{equation}
Clearly, we have that
\begin{equation}
    \mathbf{\mathbf{\Tilde{E}}}(t):=e^{\mathbf{A} t} = \diag{\left( e^{-\lambda_{1}^{2} t}, \cdots, e^{- \lambda_{d}^{2} t}, e^{- \lambda_{1}^{2} t}, \cdots, e^{- \lambda_{d}^{2} t}  \right)} = \left[\begin{array}{ll} \mathbf{E}(t) & \mathbf{0}_d  \\  \mathbf{0}_d  & \mathbf{E}(t) \end{array}\right],
\end{equation}
where $\mathbf{E}(t):= \diag{\left( e^{- \lambda_{1}^{2} t}, \cdots, e^{- \lambda_{d}^{2} t} \right)}$.
Automatically, we get that
\begin{equation}
    e^{-\mathbf{A}_1 s} = \left[\begin{array}{ll} \mathbf{E}(-s) & \mathbf{0}_d  \\  \mathbf{0}_d  & \mathbf{E}(-s) \end{array}\right],
\end{equation}
which implies that
\begin{equation}
    Z_t = \left[\begin{array}{ll} \mathbf{E}(t) & \mathbf{0}_d  \\  \mathbf{0}_d  & \mathbf{E}(t) \end{array}\right] \left( z + \sqrt{\eta} \sigma \int_{0}^t \left[\begin{array}{ll} \mathbf{E}(-s) & \mathbf{0}_d  \\  \mathbf{0}_d  & \mathbf{E}(-s) \end{array}\right]   \left[\begin{array}{ll} \mathbf{0}_d  & \mathbf{\Lambda}  \\ \mathbf{\Lambda}  & \mathbf{0}_d   \end{array}\right] \left[\begin{array}{l} d W^x_s \\ d W^y_s  \end{array}\right] \right).
\end{equation}
To conclude, we have that
    \begin{equation}
    Z_t = \mathbf{\mathbf{\Tilde{E}}}(t) \left( z + \sqrt{\eta} \sigma \int_{0}^t \mathbf{\mathbf{\Tilde{E}}}(-s)  \mathbf{M} d W_s\right),
\end{equation}
where $\mathbf{M}=\left[\begin{array}{ll} \mathbf{0}_d  & \mathbf{\Lambda}  \\ \mathbf{\Lambda}  & \mathbf{0}_d   \end{array}\right] $.
We observe that

\begin{equation}
    \E[Z_t] = \left[\begin{array}{ll} \mathbf{E}(t) & \mathbf{0}_d  \\  \mathbf{0}_d  & \mathbf{E}(t) \end{array}\right] z
\end{equation}
because the parts dependent on $dW$ are martingales. Therefore, $\E[Z_t]$ converges to $0$ exponentially fast.

Let us now have a look at the covariance matrix of this process:

$Var (Z_t) = \eta \sigma^2 \left[\begin{array}{ll} \mathbf{E}(2t) & \mathbf{0}_d  \\  \mathbf{0}_d  & \mathbf{E}(2t) \end{array}\right]  Var(V_t)$, where

\begin{align}
    V_t & := \int_{0}^t \left[\begin{array}{ll} \mathbf{E}(-s) & \mathbf{0}_d  \\  \mathbf{0}_d  & \mathbf{E}(-s) \end{array}\right] \left[\begin{array}{ll} \mathbf{0}_d  & \mathbf{\Lambda}  \\ \mathbf{\Lambda}  & \mathbf{0}_d   \end{array}\right] \left[\begin{array}{l} d W^x_s \\ d W^y_s  \end{array}\right]  \nonumber \\
    & = \left[\begin{array}{l}  \int_{0}^t \lambda_1 e^{ \lambda_1^2 s }  d W^{y_1}_s \\ \vdots \\  \int_{0}^t \lambda_d e^{ \lambda_d^2 s }  d W^{y_d}_s  \\ \int_{0}^t \lambda_1 e^{ \lambda_1^2 s }  d W^{x_1}_s \\ \vdots \\  \int_{0}^t \lambda_d e^{ \lambda_d^2 s }  d W^{x_d}_s \end{array}\right]   =:  \left[\begin{array}{l} a^{y_1}(t) \\ \vdots \\ a^{y_d}(t)\\ a^{x_1}(t) \\ \vdots \\ a^{x_d}(t)  \end{array}\right].
\end{align}
Therefore,
\begin{equation}
    Var(V_t) = \left[\begin{array}{ll} \mathbf{V}^{1,y}(t)  &  \mathbf{0}_d \\  \mathbf{0}_d  & \mathbf{V}^{1,x}(t)  \end{array}\right],
\end{equation}
such that
\begin{equation}
    \mathbf{V}^{1,y}_{i,i}(t) = Var(a^{y_i}(t)), \quad \forall i \in \{1, \cdots, d \},
\end{equation}
and
\begin{equation}
    \mathbf{V}^{1,x}_{i,i}(t) = Var(a^{x_i}(t)), \quad \forall i \in \{1, \cdots, d \}.
\end{equation}
Using the well-known Itô Isometry
$$\E\left[\left(\int_0^t H_s d W_s\right)^2\right]=\E\left[\int_0^t H_s^2 d s\right],$$
we get that
\begin{align}
    &  Var(a^{y_i}(t)) =  \int_{0}^t  e^{2 \lambda_i^2 s } \lambda_i^2 d s = \frac{1}{2} \left( e^{2 \lambda_i^2 t }  -1\right).
\end{align}
Similarly, we get that
\begin{align}
    &  Var(a^{x_i}(t)) =  \int_{0}^t  e^{2 \lambda_i^2 s } \lambda_i^2 d s = \frac{1}{2} \left( e^{2 \lambda_i^2 t }  -1\right).
\end{align}
Therefore, we conclude that the covariance matrix of $Z_t$ is
\begin{equation}
    Var(Z_t) = \eta \sigma^2 \left[\begin{array}{ll} \mathbf{I}_d - \mathbf{E}(2t) & \mathbf{0}_d  \\  \mathbf{0}_d  &  \mathbf{I}_d - \mathbf{E}(2t) \end{array}\right] \bar{\Sigma} \overset{t \rightarrow \infty}{=} \frac{\eta \sigma^2}{2} \bar{\Sigma},
\end{equation}
where
\begin{equation}
    \bar{\Sigma} := \left[\begin{array}{ll} \mathbf{I}_d&  \mathbf{0}_d  \\  \mathbf{0}_d   & \mathbf{I}_d  \end{array}\right].
\end{equation}
\end{proof}

\paragraph{Insights}
Just like for SEG, the curvature influences the speed of convergence of the algorithm. On the other, the asymptotic variance is independent of the curvature.

\paragraph{SEG vs SHGD}
We notice that if $\rho=1$, the exponential decay of SEG and SHGD is the same. However, in such a case, the asymptotic variance of SEG along the $i$-th dimension is $\eta \sigma^2\left(\frac{1}{2} +\frac{1}{2 \lambda_i}  \right)$ which is larger than that of SHGD which attains $\frac{\eta \sigma^2}{2}$.
Differently, one can select $\rho_i^{\text{V}}=\frac{1}{\lambda_i}$, which realizes the minimum variance of SEG along the $i$-th dimension. Thus, the resulting variance is $\frac{\eta \sigma^2}{2 \lambda_i}$ which is smaller than $\frac{\eta \sigma^2}{2}$ if and only if $\lambda_i>1$. In such a case, SEG can be more optimal than SHGD, but since $\rho_i = \frac{1}{\lambda_i}<1$, it will converge more slowly than SHGD.

Therefore, selecting the size of $\rho$ or $\rho_i$ leads to a trade-off between the speed of convergence and the asymptotic variance. To conclude, there is no clear winner between the two methods as their performance depends on the curvature of the landscape.

\section{QUADRATIC GAMES - INSIGHTS}\label{sec:FBG_insights}
In this section, we study Quadratic Games of the form $f(x,y) = \frac{x^{\top}\mathbf{A}x}{2} + x^{\top} \mathbf{\Lambda} y - \frac{y^{\top}\mathbf{A}y}{2}$,  where $\mathbf{\Lambda}$ and $\mathbf{A}$ are square, diagonal and positive semidefinite matrices. We notice that if $\mathbf{A}=\mathbf{0}$, these are classic Bilinear Games.
\subsection{SEG}
\begin{theorem}[Exact Dynamics of SEG] \label{thm:SEG_Dynamic_FBG}
Under the assumptions of Corollary \ref{thm:SEG_SDE_Simplified_SameSample}, for $f(x,y) = \frac{x^{\top}\mathbf{A}x}{2} + x^{\top} \mathbf{\Lambda} y - \frac{y^{\top}\mathbf{A}y}{2}$ and noise covariance matrices equal to $\sigma\mathbf{I}_d$, we have that
    \begin{equation}
    Z_t = \mathbf{\Tilde{E}}(t) \mathbf{\Tilde{R}}(t) \left( z + \sqrt{\eta} \sigma \int_{0}^t \mathbf{\Tilde{E}}(-s) \mathbf{\Tilde{R}}(-s) \mathbf{M}d W_s\right),
\end{equation}
    $\mathbf{\Tilde{E}}(t)= \left[\begin{array}{ll} \mathbf{E}(t) & \mathbf{0}_d  \\  \mathbf{0}_d  & \mathbf{E}(t) \end{array}\right],\mathbf{\Tilde{R}}(t)=\left[\begin{array}{ll} \mathbf{C}(t) & -\mathbf{S}(t)  \\  \mathbf{S}(t)  & \mathbf{C}(t) \end{array}\right]$, and $M=\left[\begin{array}{ll} \mathbf{I}_d - \rho\mathbf{A}& - \rho \mathbf{\Lambda}  \\ \rho \mathbf{\Lambda}  & \mathbf{I}_d - \rho\mathbf{A} \end{array}\right] $,
where 
\begin{equation}
    \mathbf{E}(t):= \diag{\left( e^{\rho \left(a_{1}^{2} - \lambda_1^2\right)t - a_1 t}, \cdots, e^{\rho \left(a_{d}^{2} - \lambda_d^2\right)t - a_d t} \right)},
\end{equation}
\begin{equation}
    \mathbf{C}(t):= \diag{\left( \cos{(\hat{\lambda}_1 t)}, \cdots, \cos{(\hat{\lambda}_d t)}\right)},
\end{equation} 
\begin{equation}
    \mathbf{S}(t):= \diag{\left( \sin{(\hat{\lambda}_d t)}, \cdots, \sin{(\hat{\lambda}_d t)}\right)},
\end{equation}
and $\hat{\lambda}_i:= \lambda_i (1-2 \rho a_i)$.
In particular, if $\rho \left(a_{i}^{2} - \lambda_1^2\right) - a_i <0$:

\begin{enumerate}
    \item $\E \left[ Z_t\right] = \mathbf{\Tilde{E}}(t) \mathbf{\Tilde{R}}(t) z \overset{t \rightarrow \infty}{=} 0$;
    \item The covariance matrix of $Z_t$ is equal to \begin{equation}
    \eta \sigma^2 \left[\begin{array}{ll} \mathbf{I}_d - \mathbf{E}(2t) & \mathbf{0}_d  \\  \mathbf{0}_d  &  \mathbf{I}_d - \mathbf{E}(2t) \end{array}\right] \bar{\Sigma} \overset{t \rightarrow \infty}{=} \eta \sigma^2 \bar{\Sigma},
\end{equation}
where
\begin{equation}
    \bar{\Sigma} := \left[\begin{array}{ll} \mathbf{B} &  \mathbf{0}_d  \\  \mathbf{0}_d   & \mathbf{B}  \end{array}\right],
\end{equation}
and $\mathbf{B}:= \diag{\left(  \frac{(1-\rho a_1)^2 + \rho^2 \lambda_1^2}{2(a_1 + \rho( \lambda_{1}^{2}-a_1^2))}, \cdots,  \frac{(1-\rho a_d)^2 + \rho^2 \lambda_d^2}{2(a_d + \rho( \lambda_{d}^{2}-a_d^2))} \right)}$;
\item If $\rho=0$, SGDA would indeed always converge.
\end{enumerate}
\end{theorem}

\begin{proof}

The SDE is
\begin{align}
    d Z_t = \mathbf{D} Z_t dt + \sqrt{\eta} \sigma \mathbf{B} d W_t,
\end{align}
where
\begin{equation}
     \mathbf{D} = \left[\begin{array}{ll} \rho(\mathbf{A}^2 -\mathbf{\Lambda}^2)-\mathbf{A} & -\Lambda(\mathbf{I}_d-2 \rho \mathbf{A})  \\ \mathbf{\Lambda}(\mathbf{I}_d-2 \rho \mathbf{A})  & \rho(\mathbf{A}^2 -\mathbf{\Lambda}^2)-\mathbf{A}  \end{array}\right] \quad \text{ and } \quad \mathbf{B} = \left[\begin{array}{ll} \mathbf{I}_d - \rho\mathbf{A}& - \rho\mathbf{\Lambda}  \\ \rho\mathbf{\Lambda}  & \mathbf{I}_d  - \rho\mathbf{A}\end{array}\right].
\end{equation}
Therefore, the solution is

\begin{equation}
    Z_t = e^{\mathbf{D} t} \left( z + \sqrt{\eta} \sigma \int_{0}^{t} e^{-\mathbf{D} s} \mathbf{B} d W_s \right).
\end{equation}
We observe that 
$\mathbf{D} = \mathbf{\mathbf{D}_1} + \mathbf{\mathbf{D}_2} $ s.t.

\begin{equation}
    \mathbf{D}_1= \left[\begin{array}{ll} \rho(\mathbf{A}^2 -\mathbf{\Lambda}^2)-\mathbf{A} & \mathbf{0}_d  \\  \mathbf{0}_d  & \rho(\mathbf{A}^2 -\mathbf{\Lambda}^2)-\mathbf{A}  \end{array}\right] \quad \text{ and } \quad \mathbf{\mathbf{D}_2} = \left[\begin{array}{ll} \mathbf{0}_d & -\Lambda(\mathbf{I}_d-2 \rho \mathbf{A})  \\ \mathbf{\Lambda}(\mathbf{I}_d-2 \rho \mathbf{A})  & \mathbf{0}_d \end{array}\right]
\end{equation}
and that since these two matrix commute, $e^{\mathbf{D}t} = e^{\mathbf{D}_1 t} e^{\mathbf{D}_2 t}$.
Clearly, we have that
\begin{equation}
    \mathbf{\mathbf{\Tilde{E}}}(t):=e^{\mathbf{D}_1 t} = \left[\begin{array}{ll} \mathbf{E}(t) & \mathbf{0}_d  \\  \mathbf{0}_d  & \mathbf{E}(t) \end{array}\right],
\end{equation}
where $\mathbf{E}(t):= \diag{\left( e^{\rho(a_1^2- \lambda_{1}^{2})t-a_1 t}, \cdots, e^{\rho(a_d^2- \lambda_{1}^{2})t-a_d t} \right)}$.

Regarding $\mathbf{D}_2$, we observe that $(\mathbf{D}_2 t)^{2k}$ is equal to
\begin{equation}
    \small{\diag{\left( (\lambda_1 (1-2 \rho a_1) t)^{2k} (-1)^{k}, \cdots, (\lambda_d (1-2 \rho a_d) t)^{2k} (-1)^{k}, (\lambda_1 (1-2 \rho a_1) t)^{2k} (-1)^{k}, \cdots, (\lambda_d (1-2 \rho a_d) t)^{2k} (-1)^{k} \right)}}
\end{equation}
and
\begin{equation}
    (\mathbf{D}_2 t)^{2k+1} = \left[\begin{array}{ll} \mathbf{0_d} & \mathbf{P}  \\  \mathbf{Q}  & \mathbf{0_d}  \end{array}\right],
\end{equation}
with
\begin{equation}
    \mathbf{P}:= \diag{\left( (\lambda_1 (1-2 \rho a_1) t)^{2k+1} (-1)^{k+1}, \cdots, (\lambda_d (1-2 \rho a_d) t)^{2k+1} (-1)^{k+1} \right)}
\end{equation}
and
\begin{equation}
    \mathbf{Q}:= \diag{\left( (\lambda_1 (1-2 \rho a_1) t)^{2k+1} (-1)^{k}, \cdots, (\lambda_d (1-2 \rho a_d) t)^{2k+1} (-1)^{k} \right)}.
\end{equation}
Therefore,
\begin{align}
   \mathbf{\mathbf{\Tilde{R}}}(t):= e^{\mathbf{D}_2 t} & = \sum_{k=0}^{\infty} \frac{(\mathbf{D}_2 t)^{2k}}{(2k)!} + \sum_{k=0}^{\infty} \frac{(\mathbf{D}_2 t)^{2k+1}}{(2k+1)!} \nonumber \\
    & = \left[\begin{array}{ll} \mathbf{C}(t) & \mathbf{0}_d  \\  \mathbf{0}_d  & \mathbf{C}(t) \end{array}\right] + \left[\begin{array}{ll} \mathbf{0}_d & -\mathbf{S}(t)  \\  \mathbf{S}(t)  & \mathbf{0}_d \end{array}\right] = \left[\begin{array}{ll} \mathbf{C}(t) & -\mathbf{S}(t)  \\  \mathbf{S}(t)  & \mathbf{C}(t) \end{array}\right],
\end{align}
where
\begin{equation}
    \mathbf{C}(t):= \diag{\left( \cos{(\lambda_1 (1-2 \rho a_1) t)}, \cdots, \cos{(\lambda_d (1-2 \rho a_d) t)}\right)},
\end{equation}
and
\begin{equation}
    \mathbf{S}(t):= \diag{\left( \sin{(\lambda_1 (1-2 \rho a_1) t)}, \cdots, \sin{(\lambda_d (1-2 \rho a_d) t)}\right)}.
\end{equation}
Automatically, we get that
\begin{equation}
    e^{-\mathbf{D}_1 s} = \left[\begin{array}{ll} \mathbf{E}(-s) & \mathbf{0}_d  \\  \mathbf{0}_d  & \mathbf{E}(-s) \end{array}\right]
\end{equation}
and 
\begin{align}
    e^{-\mathbf{D}_2 s} = \left[\begin{array}{ll} \mathbf{C}(s) & \mathbf{S}(s)  \\  -\mathbf{S}(s)  & \mathbf{C}(s) \end{array}\right],
\end{align}
which implies that
\begin{align}
    Z_t & = \left[\begin{array}{ll} \mathbf{E}(t) & \mathbf{0}_d  \\  \mathbf{0}_d  & \mathbf{E}(t) \end{array}\right] \left[\begin{array}{ll} \mathbf{C}(t) & -\mathbf{S}(t)  \\  \mathbf{S}(t)  & \mathbf{C}(t) \end{array}\right] \bigg( z  \nonumber  \\
    & + \sqrt{\eta} \sigma \int_{0}^t \left[\begin{array}{ll} \mathbf{E}(-s) & \mathbf{0}_d  \\  \mathbf{0}_d  & \mathbf{E}(-s) \end{array}\right] \left[\begin{array}{ll} \mathbf{C}(s) & \mathbf{S}(s)  \\  -\mathbf{S}(s)  & \mathbf{C}(s) \end{array}\right]  \left[\begin{array}{ll} \mathbf{I}_d - \rho\mathbf{A}& - \rho\mathbf{\Lambda}  \\ \rho\mathbf{\Lambda}  & \mathbf{I}_d  - \rho\mathbf{A}\end{array}\right] \left[\begin{array}{l} d W^x_s \\ d W^y_s  \end{array}\right] \bigg).
\end{align}
We observe that since the expected value of the noise terms is $0$, we have
\begin{equation}
    \E[Z_t] = \left[\begin{array}{ll} \mathbf{E}(t) & \mathbf{0}_d  \\  \mathbf{0}_d  & \mathbf{E}(t) \end{array}\right] \left[\begin{array}{ll} \mathbf{C}(t) & -\mathbf{S}(t)  \\  \mathbf{S}(t)  & \mathbf{C}(t) \end{array}\right] z.
\end{equation}
Therefore if $\rho(a_i^2- \lambda_{i}^{2})-a_i <0$, $\E[Z_t]$ converges to $0$ exponentially fast, while spiraling around the origin.

Let us now have a look at the variance of this process:

$Var (Z_t) = \eta \sigma^2 \left[\begin{array}{ll} \mathbf{E}(2t) & \mathbf{0}_d  \\  \mathbf{0}_d  & \mathbf{E}(2t) \end{array}\right] \mathbf{\mathbf{\Tilde{R}}}(t) Var(V_t)  \mathbf{\mathbf{\Tilde{R}}}(t)^{\top}$, where

\begin{align}
    V_t & := \int_{0}^t \left[\begin{array}{ll} \mathbf{E}(-s) & \mathbf{0}_d  \\  \mathbf{0}_d  & \mathbf{E}(-s) \end{array}\right] \left[\begin{array}{ll} \mathbf{C}(s) & \mathbf{S}(s)  \\  -\mathbf{S}(s)  & \mathbf{C}(s) \end{array}\right]  \left[\begin{array}{ll} \mathbf{I}_d - \rho\mathbf{A}& - \rho\mathbf{\Lambda}  \nonumber  \\  \rho\mathbf{\Lambda}  & \mathbf{I}_d - \rho\mathbf{A} \end{array}\right] \left[\begin{array}{l} d W^x_s \\ d W^y_s  \end{array}\right]  \nonumber  \\
     & = \int_{0}^t \left[\begin{array}{ll} \mathbf{E}(-s) (\mathbf{C}(s) (\mathbf{I}_d - \rho \mathbf{A})  + \rho\mathbf{\Lambda} \mathbf{S}(s) ) & \mathbf{E}(-s) (\mathbf{S}(s)(\mathbf{I}_d - \rho \mathbf{A})  - \rho\mathbf{\Lambda} \mathbf{C}(s))  \nonumber  \\  \mathbf{E}(-s) (\rho\mathbf{\Lambda} \mathbf{C}(s)-\mathbf{S}(s)(\mathbf{I}_d - \rho \mathbf{A}) )  &  \mathbf{E}(-s) (\mathbf{C}(s)(\mathbf{I}_d - \rho \mathbf{A})  + \rho\mathbf{\Lambda} \mathbf{S}(s) ) \end{array}\right]  \left[\begin{array}{l} d W^x_s \\ d W^y_s  \end{array}\right] \nonumber  \\
    & = {\scriptsize\left[\begin{array}{l} \int_{0}^t  e^{a_1 s - \rho(a_1^2- \lambda_{1}^{2})s}  ( \cos{(\hat{\lambda}_1 s)}(1-\rho a_1) + \rho \lambda_1 \sin{(\hat{\lambda}_1 s)} ) d W^{x_1}_s + \int_{0}^t  e^{a_1 s - \rho(a_1^2- \lambda_{1}^{2})s}  ( \sin{(\hat{\lambda}_1 s)}(1-\rho a_1) - \rho \lambda_1 \cos{(\hat{\lambda}_1 s)} ) d W^{y_1}_s \\ \vdots \\ \int_{0}^t  e^{a_d s - \rho(a_d^2- \lambda_{d}^{2})s}  ( \cos{(\hat{\lambda}_d s)}(1-\rho a_d) + \rho \lambda_d \sin{(\hat{\lambda}_d s)} ) d W^{x_d}_s + \int_{0}^t  e^{a_d s - \rho(a_d^2- \lambda_{d}^{2})s}  ( \sin{(\hat{\lambda}_d s)}(1-\rho a_d) - \rho \lambda_d \cos{(\hat{\lambda}_d s)} ) d W^{y_d}_s  \\ \int_{0}^t  e^{a_1 s - \rho(a_1^2- \lambda_{1}^{2})s}  (\rho \lambda_1 \cos{(\hat{\lambda}_1 s)} - \sin{(\hat{\lambda}_1 s)}(1-\rho a_1)  ) d W^{x_1}_s + \int_{0}^t  e^{a_1 s - \rho(a_1^2- \lambda_{1}^{2})s}  (\cos{(\hat{\lambda}_1 s)}(1-\rho a_1) + \rho \lambda_1 \sin{(\hat{\lambda}_1 s)} ) d W^{y_1}_s \\ \vdots \\ \int_{0}^t  e^{a_d s - \rho(a_d^2- \lambda_{1}^{2})s}  (\rho \lambda_d \cos{(\hat{\lambda}_d s)} - \sin{(\hat{\lambda}_d s)}(1-\rho a_d)  ) d W^{x_d}_s + \int_{0}^t  e^{a_d s - \rho(a_d^2- \lambda_{1}^{2})s}  (\cos{(\hat{\lambda}_d s)}(1-\rho a_d) + \rho \lambda_d \sin{(\hat{\lambda}_d s)} ) d W^{y_d}_s \end{array}\right]}  \nonumber \\
    &  =:  \left[\begin{array}{l} a^{x_1}_1(t) + a^{y_1}_2(t) \\ \vdots \\ a^{x_d}_1(t) + a^{y_d}_2(t) \\ a^{x_1}_3(t) + a^{y_1}_4(t) \\ \vdots \\ a^{x_d}_3(t) + a^{y_d}_4(t)  \end{array}\right] \text{ and } \hat{\lambda}_i:= \lambda_i (1-2 \rho a_i).
\end{align}
Therefore,
\begin{equation}
    Var(V_t) = \left[\begin{array}{ll} \mathbf{V}^{1,2}(t)  &  \mathbf{C}^{1,2,3,4}(t) \\  \mathbf{C}^{1,2,3,4}(t)  & \mathbf{V}^{3,4}(t)  \end{array}\right],
\end{equation}
such that
\begin{equation}
    \mathbf{V}^{1,2}_{i,i}(t) = Var(a^{x_i}_1(t)) + Var(a^{y_i}_2(t)), \quad \forall i \in \{1, \cdots, d \},
\end{equation}
\begin{equation}
    \mathbf{V}^{3,4}_{i,i}(t) = Var(a^{x_i}_3(t)) + Var(a^{y_i}_4(t)), \quad \forall i \in \{1, \cdots, d \},
\end{equation}
and 
\begin{equation}
    \mathbf{C}^{1,2,3,4}_{i,i}(t) = Cov(a^{x_i}_1(t),a^{x_i}_3(t)) + Cov(a^{y_i}_1(t),a^{y_i}_3(t)), \quad \forall i \in \{1, \cdots, d \}.
\end{equation}
Using the well-known Itô Isometry,
$$\E\left[\left(\int_0^t H_s d W_s\right)^2\right]=\E\left[\int_0^t H_s^2 d s\right],$$
we get that
\begin{align}
    Var(a^{x_i}_1(t)) + Var(a^{y_i}_2(t)) & =  \int_{0}^t  e^{2(a_i - \rho(a_i^2- \lambda_{i}^{2}))s}  ( \cos{(\hat{\lambda}_i s)}(1-\rho a_i) + \rho \lambda_i \sin{(\hat{\lambda}_i s)} )^2 d s \nonumber \\
    & + \int_{0}^t  e^{2(a_i  - \rho(a_i^2- \lambda_{i}^{2}))s}  ( \sin{(\hat{\lambda}_i s)}(1-\rho a_i) - \rho \lambda_i \cos{(\hat{\lambda}_i s)} )^2 d s \nonumber  \\
    & = \int_{0}^t  e^{2(a_i - \rho(a_i^2- \lambda_{i}^{2}))s} \left[  (1-\rho a_i)^2 + \rho^2 \lambda_i^2 \right] d s \nonumber  \\
    & = \frac{(1-\rho a_i)^2 + \rho^2 \lambda_i^2}{2(a_i + \rho( \lambda_{i}^{2}-a_i^2))} \left( e^{2(a_i  + \rho(\lambda_{i}^{2}-a_i^2))t} -1\right).
\end{align}
Then, we do a similar calculation and find that
\begin{align}
    & Var(a^{x_i}_3(t)) + Var(a^{y_i}_4(t)) =  \frac{(1-\rho a_i)^2 + \rho^2 \lambda_i^2}{2(a_i + \rho( \lambda_{i}^{2}-a_i^2))} \left( e^{2(a_i  + \rho(\lambda_{i}^{2}-a_i^2))t} -1\right).
\end{align}
Remembering now that
$$\E\left[\left(\int_0^t X_s d W_s\right)\left(\int_0^t Y_s d W_s\right)\right]=\E\left[\int_0^t X_s Y_s d s\right],$$
we have that
\begin{align}
    Cov(a^{x_i}_1(t),a^{x_i}_3(t)) +Cov(a^{y_i}_2(t),a^{y_i}_4(t)) & = \int_{0}^t  e^{2 \rho \lambda_i^2 s }  (\cos{(\lambda_i s)} + \rho \lambda_i \sin{(\lambda_i s)} )(- \sin{(\lambda_i s)} + \rho \lambda_i \cos{(\lambda_i s)} ) d s \nonumber  \\
    & + \int_{0}^t  e^{2\rho \lambda_i^2 s }  (\sin{(\lambda_i s)} - \rho \lambda_i \cos{(\lambda_i s)} )(\cos{(\lambda_i s)} + \rho \lambda_i \sin{(\lambda_i s)} ) d s = 0.
\end{align}
Therefore, we conclude that the covariance matrix of $Z_t$ is
\begin{equation}
    Var(Z_t) = \eta \sigma^2 \left[\begin{array}{ll} \mathbf{I}_d -  \mathbf{E}(2t) & \mathbf{0}_d  \\  \mathbf{0}_d  &  \mathbf{I}_d -  \mathbf{E}(2t) \end{array}\right] \bar{\Sigma} \overset{t \rightarrow \infty}{=} \eta \sigma^2 \bar{\Sigma},
\end{equation}
with
\begin{equation}
    \bar{\Sigma} := \left[\begin{array}{ll} \mathbf{B}  & \mathbf{0}_d \\  \mathbf{0}_d  & \mathbf{B}  \end{array}\right],
\end{equation}
where
\begin{equation}\label{eq:varFBG_app}
    \mathbf{B}:= \diag{\left(  \frac{(1-\rho a_1)^2 + \rho^2 \lambda_1^2}{2(a_1 + \rho( \lambda_{1}^{2}-a_1^2))}, \cdots,  \frac{(1-\rho a_d)^2 + \rho^2 \lambda_d^2}{2(a_d + \rho( \lambda_{d}^{2}-a_d^2))} \right)}.
\end{equation}
Empirical validation of Eq. \eqref{eq:varFBG_app} is provided in Figure \ref{fig:Variance}.

\end{proof}

\begin{lemma}
Let us define the variance $\mathbf{B}_{i,i}(\rho) = \frac{(1-\rho a_i)^2 + \rho^2 \lambda_i^2}{2(a_i + \rho( \lambda_{i}^{2}-a_i^2))} $ and consider it as a function of $\rho$. The following hold:
\begin{enumerate}
    \item  $\rho_i \left(a_{i}^{2} - \lambda_i^2\right)<0$ is necessary to converge faster than SGDA;
    \item If $\lambda_i>a_i$ and $\lim_{\rho \rightarrow \infty} \mathbf{B}_{i,i}(\rho) = \infty$;
    \item If $\lambda_i<a_i$ and $\lim_{\rho \rightarrow -\infty} \mathbf{B}_{i,i}(\rho) = \infty$;
    \item $\lim_{\rho \rightarrow \frac{-a_i}{\lambda_{i}^{2}-a_i^2}} \mathbf{B}_{i,i}(\rho) = \infty$;
    \item $\rho = \frac{1}{\lambda_i+ a_i}$ realizes the minimum of $\mathbf{B}_{i,i}$ and $\mathbf{B}_{i,i}\left(\frac{1}{\lambda_i+a_i}\right) = \frac{\eta\sigma^2}{2}\frac{\lambda_i}{(a_i+\lambda_i)^2}$;
    \item The trace of $\mathbf{B}$ is  is strictly convex in $\rho$, meaning that there is a unique minimizer.
\end{enumerate}
\end{lemma}
\begin{proof}
    All points above can be proven easily and are left an exercise for the reader. 
\end{proof}

\paragraph{Insights - The trade-off in selecting $\rho$} \label{parag:roleofrho_1}
The curvature of the landscape influences the speed of convergence. Indeed, larger values of $a_i$, which correspond to stronger convexity/concavity, speed up the exponential decay in the expected value of the iterates. Differently, the relative size of $\lambda_i$ and $a_i$ influences the convergence depending on the sign of $\rho$. First of all, if $\rho_i \left(a_{i}^{2} - \lambda_i^2\right)>0$, SEG is slower than SGDA at converging and $\rho_i \left(a_{i}^{2} - \lambda_i^2\right)<0$ is necessary to converge faster than SGDA. This means that \textbf{negative} $\rho_i$ might be convenient if $a_{i}>\lambda_i$. Therefore, if $\rho_i$ has the correct sign, a \textbf{larger} absolute value implies \textbf{faster convergence}. However, we also have that the asymptotic variance along the $i$-th coordinate $\mathbf{B}_{i,i}(\rho_i)$ \textbf{explodes} if $|\rho_i|$ is too \textbf{large} or if $ \rho_i \rightarrow \frac{-a_i}{\lambda_i^2-a_i^2}$. On the bright side, $B_{i,i}(\rho_i)$ is a convex function of $\rho_i$ whose minimum is realized at $\rho^{\text{V}}_i=\frac{1}{a_i + \lambda_i}$. However, if $\rho^{\text{V}}_i$ is \textbf{small}, it \textbf{slows down} the convergence. Finally, if one has to choose a single value of $\rho$, one has to carefully select it as it will (de)accelerate different coordinates based on its sign. Fortunately, the trace of $\mathbf{B}$ is a convex function of $\rho$, meaning that there is an optimal $\rho^{*}$ that minimizes it.

\subsection{SHGD}
\begin{theorem}[Exact Dynamics of SHGD] \label{thm:SHGD_Dynamic_FBG}
Under the assumptions of Corollary \ref{thm:SHGD_SDE_Simplified_SameSample}, for $f(x,y) = \frac{x^{\top}\mathbf{A}x}{2} + x^{\top} \mathbf{\Lambda} y - \frac{y^{\top}\mathbf{A}y}{2}$ and noise covariance matrices equal to $\sigma\mathbf{I}_d$, we have that
    \begin{equation}
    Z_t = \mathbf{\Tilde{E}}(t) \left( z + \sqrt{\eta} \sigma \int_{0}^t \mathbf{\Tilde{E}}(-s)\mathbf{M}d W_s\right),
\end{equation}
with $\mathbf{\Tilde{E}}(t)= \left[\begin{array}{ll} \mathbf{E}(t) & \mathbf{0}_d  \\  \mathbf{0}_d  & \mathbf{E}(t) \end{array}\right]$, $M=\left[\begin{array}{ll}\mathbf{A}&  \mathbf{\Lambda}  \\ \mathbf{\Lambda}  & -\mathbf{A}  \end{array}\right] $,
where 
\begin{equation}
    \mathbf{E}(t):= \diag{\left( e^{-\left(\lambda_{1}^{2}+a_{1}^{2}  \right)t}, \cdots, e^{-\left(\lambda_{d}^{2}+a_{d}^{2}  \right)t} \right)}.
\end{equation}
In particular, we have that
\begin{enumerate}
    \item $\E \left[ Z_t\right] = \mathbf{\Tilde{E}}(t) z \overset{t \rightarrow \infty}{=} 0$;
    \item The covariance matrix of $Z_t$ is equal to \begin{equation}
    \eta \frac{\sigma^2 }{2}\left[\begin{array}{ll} \mathbf{I}_d - \mathbf{E}(2t) & \mathbf{0}_d  \\  \mathbf{0}_d  &  \mathbf{I}_d - \mathbf{E}(2t) \end{array}\right] \bar{\Sigma} \overset{t \rightarrow \infty}{=} \eta \sigma^2 \bar{\Sigma},
\end{equation}
where
\begin{equation}
    \bar{\Sigma} := \left[\begin{array}{ll} \mathbf{I}_d &  \mathbf{0}_d  \\  \mathbf{0}_d   & \mathbf{I}_d  \end{array}\right].
\end{equation}
\end{enumerate}
\end{theorem}

\begin{proof}
The SDE is
\begin{align}
    d Z_t = \mathbf{D} Z_t dt + \sqrt{\eta} \sigma \mathbf{B} d W_t,
\end{align}
where
\begin{equation}
     \mathbf{D} = \left[\begin{array}{ll} - (\mathbf{\Lambda}^2 + \mathbf{A}^2) & \mathbf{0}_d  \\  \mathbf{0}_d  &  -(\mathbf{\Lambda}^2 + \mathbf{A}^2)  \end{array}\right] \quad \text{ and } \quad \mathbf{B} = \left[\begin{array}{ll}\mathbf{A}&\mathbf{\Lambda}  \\ \mathbf{\Lambda}  & -\mathbf{A}  \end{array}\right].
\end{equation}
Therefore, the solution is
\begin{equation}
    Z_t = e^{\mathbf{D} t} \left( z + \sqrt{\eta} \sigma \int_{0}^{t} e^{-\mathbf{D} s} \mathbf{B} d W_s \right).
\end{equation}
Clearly, we have that
\begin{equation}
    \mathbf{\mathbf{\Tilde{E}}}(t):=e^{\mathbf{D} t} = \diag{\left( e^{- (\lambda_{1}^{2}+a_{1}^{2}) t}, \cdots, e^{- (\lambda_{d}^{2}+a_{d}^{2}) t}, e^{- (\lambda_{1}^{2}+a_{1}^{2}) t}, \cdots, e^{- (\lambda_{d}^{2}+a_{d}^{2}) t}  \right)} = \left[\begin{array}{ll} \mathbf{E}(t) & \mathbf{0}_d  \\  \mathbf{0}_d  & \mathbf{E}(t) \end{array}\right],
\end{equation}
where $\mathbf{E}(t):= \diag{\left( e^{- (\lambda_{1}^{2}+a_{1}^{2}) t}, \cdots, e^{- (\lambda_{d}^{2}+a_{d}^{2}) t} \right)}$.
Automatically, we get that
\begin{equation}
    e^{-\mathbf{D} s} = \left[\begin{array}{ll} \mathbf{E}(-s) & \mathbf{0}_d  \\  \mathbf{0}_d  & \mathbf{E}(-s) \end{array}\right],
\end{equation}
which implies that
\begin{equation}
    Z_t = \left[\begin{array}{ll} \mathbf{E}(t) & \mathbf{0}_d  \\  \mathbf{0}_d  & \mathbf{E}(t) \end{array}\right] \left( z + \sqrt{\eta} \sigma \int_{0}^t \left[\begin{array}{ll} \mathbf{E}(-s) & \mathbf{0}_d  \\  \mathbf{0}_d  & \mathbf{E}(-s) \end{array}\right]   \left[\begin{array}{ll}\mathbf{A}& \mathbf{\Lambda}  \\ \mathbf{\Lambda}  & -\mathbf{A}   \end{array}\right] \left[\begin{array}{l} d W^x_s \\ d W^y_s  \end{array}\right] \right).
\end{equation}
To conclude, we have that
    \begin{equation}
    Z_t = \mathbf{\mathbf{\Tilde{E}}}(t) \left( z + \sqrt{\eta} \sigma \int_{0}^t \mathbf{\mathbf{\Tilde{E}}}(-s)  \mathbf{M} d W_s\right),
\end{equation}
where $\mathbf{M}=\left[\begin{array}{ll}\mathbf{A} & \mathbf{\Lambda}  \\ \mathbf{\Lambda}  & -\mathbf{A}   \end{array}\right] $.

We observe that since the expected value of the noise is $0$,
\begin{equation}
    \E[Z_t] = \left[\begin{array}{ll} \mathbf{E}(t) & \mathbf{0}_d  \\  \mathbf{0}_d  & \mathbf{E}(t) \end{array}\right] z.
\end{equation}
Therefore, it converges to $0$ exponentially fast.

Let us now have a look at the covariance matrix of this process:

$Var (Z_t) = \eta \sigma^2 \left[\begin{array}{ll} \mathbf{E}(2t) & \mathbf{0}_d  \\  \mathbf{0}_d  & \mathbf{E}(2t) \end{array}\right]  Var(V_t)$, where,
\begin{align}
    V_t & := \int_{0}^t \left[\begin{array}{ll} \mathbf{E}(-s) & \mathbf{0}_d  \\  \mathbf{0}_d  & \mathbf{E}(-s) \end{array}\right] \left[\begin{array}{ll}\mathbf{A}& \mathbf{\Lambda}  \\ \mathbf{\Lambda}  & -\mathbf{A}   \end{array}\right] \left[\begin{array}{l} d W^x_s \\ d W^y_s  \end{array}\right]  \nonumber  \\
    & = \left[\begin{array}{l}  \int_{0}^t \lambda_1 e^{ (\lambda_{1}^{2}+a_{1}^{2}) s }  d W^{y_1}_s + \int_{0}^t a_1 e^{ (\lambda_{1}^{2}+a_{1}^{2}) s }  d W^{x_1}_s \\ \vdots \\ \int_{0}^t \lambda_d e^{ (\lambda_{d}^{2}+a_{d}^{2}) s }  d W^{y_d}_s + \int_{0}^t a_d e^{ (\lambda_{d}^{2}+a_{d}^{2}) s }  d W^{x_d}_s  \\ \int_{0}^t \lambda_1 e^{ (\lambda_{1}^{2}+a_{1}^{2}) s }  d W^{x_1}_s + \int_{0}^t a_1 e^{ (\lambda_{1}^{2}+a_{1}^{2}) s }  d W^{y_1}_s \\ \vdots \\ \int_{0}^t \lambda_d e^{ (\lambda_{d}^{2}+a_{d}^{2}) s }  d W^{x_d}_s + \int_{0}^t a_d e^{ (\lambda_{d}^{2}+a_{d}^{2}) s }  d W^{y_d}_s \end{array}\right]   =:  \left[\begin{array}{l} a_1^{x_1}(t) + a_2^{y_1}(t) \\ \vdots \\ a_1^{x_d}(t) + a_2^{y_d}(t)\\ a_3^{x_1}(t) + a_4^{y_1}(t) \\ \vdots \\ a_3^{x_d}(t) + a_4^{y_d}(t)  \end{array}\right].
\end{align}
Therefore,
\begin{equation}
    Var(V_t) = \left[\begin{array}{ll} \mathbf{V}^{1,2}(t)  &  \mathbf{0}_d \\  \mathbf{0}_d  & \mathbf{V}^{3,4}(t)  \end{array}\right],
\end{equation}
such that
\begin{equation}
    \mathbf{V}^{1,2}_{i,i}(t) = Var(a_1^{x_i}(t))+Var(a_2^{y_i}(t)), \quad \forall i \in \{1, \cdots, d \},
\end{equation}
and
\begin{equation}
    \mathbf{V}^{3,4}_{i,i}(t) = Var(a_3^{x_i}(t))+Var(a_4^{y_i}(t)), \quad \forall i \in \{1, \cdots, d \}.
\end{equation}
Using the well-known Itô Isometry
$$\E\left[\left(\int_0^t H_s d W_s\right)^2\right]=\E\left[\int_0^t H_s^2 d s\right],$$
we get that
\begin{align}
    &  Var(a_1^{x_i}(t))+Var(a_2^{y_i}(t)) =  \int_{0}^t  e^{2 (\lambda_{i}^{2}+a_{i}^{2}) s } \lambda_i^2 d s + \int_{0}^t  e^{2 (\lambda_{i}^{2}+a_{i}^{2}) s } a_i^2 d s = \frac{1}{2} \left( e^{2 (\lambda_{i}^{2}+a_{i}^{2}) t }  -1 \right).
\end{align}
Similarly, we get that
\begin{align}
    &  Var(a_3^{x_i}(t))+Var(a_4^{y_i}(t)) =  \frac{1}{2} \left( e^{2 (\lambda_{i}^{2}+a_{i}^{2}) t }  -1 \right).
\end{align}
Therefore, we conclude that the covariance matrix of $Z_t$ is
\begin{equation}
    Var(Z_t) = \eta \sigma^2 \left[\begin{array}{ll} \mathbf{I}_d - \mathbf{E}(2t) & \mathbf{0}_d  \\  \mathbf{0}_d  &  \mathbf{I}_d - \mathbf{E}(2t) \end{array}\right] \bar{\Sigma} \overset{t \rightarrow \infty}{=} \frac{\eta \sigma^2}{2} \bar{\Sigma},
\end{equation}
where
\begin{equation}
    \bar{\Sigma} := \left[\begin{array}{ll} \mathbf{I}_d&  \mathbf{0}_d  \\  \mathbf{0}_d   & \mathbf{I}_d  \end{array}\right].
\end{equation}
\end{proof}

\paragraph{SEG vs SHGD}  \label{parag:roleofrho_2}
Interestingly, for both algorithms, the curvature of the landscape influences the speed of convergence. However, the asymptotic covariance matrix of SHGD is not influenced by it. Much differently, the speed of convergence of the expected value of $Z_t$ is strongly influenced by the values of $a_i$, $\lambda_i$, and for SEG also by the sign and magnitude of $\rho$. For example, if $(\lambda_i^2-a_i^2)\rho^{H}_i > a_i^2 + \lambda_i^2 - a_i$, SEG exponentially decays \textbf{faster} than SHGD. However, this results in SEG having a \textbf{larger} asymptotic variance. Another interesting choice is to reduce the asymptotic variance of SEG by selecting $\rho^{V}_i = \frac{1}{\lambda_i + a_i}$. In this case, SEG attains its lowest asymptotic variance $\frac{\eta \sigma^2\lambda_i}{2(a_i + \lambda_i)^2}$, which is smaller than $\frac{\eta \sigma^2}{2}$ reached by SHGD only if $a_i^2+\lambda_i^2-\lambda_i>0$.
Finally, if $a_i<0$, $Z=0$ is a \textit{bad saddle} that one wishes to escape. While SEG can escape it, SHGD is pulled towards it.

Therefore, selecting the size of $\rho$ or $\rho_i$ leads to a trade-off between the speed of convergence and the asymptotic variance. To conclude, there is no clear winner between the two methods as the entire dynamics depend on the curvature of the landscape.

\section{CONVERGENCE GUARANTEES}

In this section, we provide a complete and precise characterization of the dynamics of the Hamiltonian under the dynamics of SEG and SHGD. We then use the latter to derive convergence bounds and establish conditions under which stepsize schedulers guarantee asymptotic convergence. For simplicity, we write $\bigO(\text{Noise})$ for the terms that depend on $dW_t$ as they vanish once we take an expectation.

In this section, we will often make use of the following shorthand:

\begin{enumerate}
    \item $H_t = \E_\bgamma[\mathcal{H}_{\bgamma}(Z_t)] = \E_{\gamma^1, \gamma^2} \left[ \mathcal{H}_{\gamma^1, \gamma^2}(Z_t)\right]= \E_\bgamma \left[ \frac{F_{\gamma^1}^{\top}(Z_t)F_{\gamma^2}(Z_t)}{2} \right]$;
    \item $\Sigma^{\shgd}_t:= \Sigma^{\shgd}(Z_t)$ and $\Sigma^{\seg}_t:= \Sigma^{\seg}(Z_t)$;
    \item $F_t = F(Z_t)$, $\nabla F_t = \nabla F(Z_t)$,  $\nabla f_t = \nabla f(Z_t)$, and $\nabla^2 f_t = \nabla^2 f(Z_t)$.
\end{enumerate}

\subsection{SHGD}
\begin{lemma}[Dynamics of Hamiltonian] \label{thm:SHGD_DynHamil}
Let $Z_t$ be the solution of the SDE of SHGD. Then,
$$\E\left[\dot H_t\right] = - \E\left[\|\nabla H_t\|^2\right] + \frac{\eta}{2} \tr\left(\E\left[\Sigma^{\shgd}_t\nabla^2 H_t\right]\right).$$
\end{lemma}

\begin{proof}
The SDE for SHGD is:

\begin{equation}
    d Z_t = - \nabla \E_{\bgamma} \left[ \mathcal{H}_{\bgamma} \left( Z_t \right)  \right] dt + \sqrt{\eta} \sqrt{\Sigma^{\shgd}_t}d W_{t}.
\end{equation}
Therefore, by Itô's Lemma we have 
\begin{align}
     d\E_{\bgamma} \left[ \mathcal{H}_{\bgamma}(Z_t) \right] = &  - \lVert \nabla \E_{\bgamma} \left[ \mathcal{H}_{\bgamma}(Z_t) \right] \rVert_2^2  dt  + \frac{\eta}{2} \tr \left(  \sqrt{\Sigma^{\shgd}_t}  \nabla^2 \E_{\bgamma} \left[ \mathcal{H}_{\bgamma}(Z_t) \right]  \sqrt{\Sigma^{\shgd}_t}  \right)dt + \bigO(\text{Noise}).
\end{align}
Therefore,
\begin{align}
  d \E \left[ \E_{\bgamma} \left[ \mathcal{H}_{\bgamma}(Z_t) \right] \right]= &  - \E \left[\lVert \nabla \E_{\bgamma} \left[ \mathcal{H}_{\bgamma}(Z_t) \right] \rVert_2^2 \right] dt  + \frac{\eta}{2} \tr \left( \E \left[ \sqrt{\Sigma^{\shgd}_t}  \nabla^2 \E_{\bgamma} \left[ \mathcal{H}_{\bgamma}(Z_t) \right]  \sqrt{\Sigma^{\shgd}_t}  \right] \right)dt,
\end{align}
which implies that
\begin{equation}
    \E\left[\dot H_t\right] = - \E\left[\|\nabla H_t\|^2\right] + \frac{\eta}{2} \tr\left(\E\left[\Sigma^{\shgd}_t\nabla^2 H_t\right]\right).
\end{equation}
\end{proof}

\begin{theorem}[SHGD General Convergence] \label{theorem:SHGD_GenConv}Consider the solution $Z_t$ of the SHGD SDE with $\gamma^1\ne\gamma^2$.
Let $v_t := \E \left[\E_{\bgamma}  \left[ \lVert \nabla \mathcal{H} \left( Z_t \right) -  \nabla \mathcal{H}_{\bgamma} \left( Z_t \right) \rVert_2^2 \right]\right]$ measure the error in $ \nabla \mathcal{H}$, in expectation over the whole randomness up to time $t$. Suppose that:
\begin{enumerate}
    \item The smallest eigenvalue (in absolute value) $\mu$ of $\nabla^2 f(z)$ is non-zero;
    \item $ \lVert \nabla^2 H(z)\rVert_{\text{op}} < \mathcal{L}_{\mathcal{T}}$, for all $z\in\R^{2d}$.
\end{enumerate}
Then,
\begin{equation}
    \E \left[ H_t \right] \leq  e^{-2\mu^2t} \left[ H_0 + \frac{\eta\mathcal{L}_{\mathcal{T}}}{2} \int_0^t v_s  e^{2\mu^2 s} ds \right].
\end{equation}
\end{theorem}

\begin{proof}
Under these assumptions,
\begin{align}
    d H_t = & - \lVert \nabla H_t \rVert^2 dt + \frac{\eta}{2} \tr \left( \nabla^2 H_t \Sigma^{\shgd}_t \right) dt + \bigO(\text{Noise})   \nonumber \\
    = & - \lVert F^{\top}_t \nabla F_t \rVert^2 dt + \frac{\eta}{2} \tr \left( \nabla^2 H_t \Sigma^{\shgd}_t \right) dt + \bigO(\text{Noise}) \nonumber  \\
    = & - \lVert \nabla^2 f_t \nabla f_t \rVert^2 dt + \frac{\eta}{2} \tr \left( \nabla^2 H_t \Sigma^{\shgd}_t \right) dt + \bigO(\text{Noise})  \nonumber  \\
    \leq & - \mu^2  \lVert  \nabla f_t \rVert^2 dt + \frac{\eta}{2} \lVert \nabla^2 H_t \rVert_{\text{op}} \tr \left( \Sigma^{\shgd}_t \right) dt + \bigO(\text{Noise})  \nonumber \\
    \leq & - \mu^2  \lVert  \nabla f_t \rVert^2 dt + \frac{\eta}{2} \mathcal{L}_{\mathcal{T}} \tr \left( \Sigma^{\shgd}_t \right) dt + \bigO(\text{Noise}),
\end{align}
where we used that $\tr(\mathbf{A} \mathbf{B}) \leq \tr(\mathbf{A}) \lVert \mathbf{B} \rVert_{\text{op}}$ if $\mathbf{A}$ is a real positive semi-definite matrix and $\mathbf{B}$ is of the same size. Then, we observe that
\begin{align}
    \E \left[\tr\left(\Sigma^{\shgd}_t \right)\right] & = \E \left[\E_{\bgamma}  \left[ \lVert \nabla \mathcal{H} \left( Z_t \right) -  \nabla \mathcal{H}_{\bgamma} \left( Z_t \right) \rVert_2^2 \right] \right] = v_t,
\end{align}
which implies that 
\begin{align}
    d \E \left[H_t \right] & \leq - \mu^2  \E \left[\lVert  \nabla f_t \rVert^2  \right]dt + \frac{\eta}{2} \mathcal{L}_{\mathcal{T}} v_t dt  \nonumber \\
    & = -2 \mu^2  \E \left[ H_t \right]  dt + \frac{\eta}{2} \mathcal{L}_{\mathcal{T}} v_t dt,
\end{align}
which in turn implies that
\begin{equation}
    \E \left[ H_t \right] \leq  e^{-2\mu^2t} \left[ H_0 + \frac{\eta\mathcal{L}_{\mathcal{T}}}{2} \int_0^t v_s  e^{2\mu^2s} ds \right].
\end{equation}
\end{proof}

\begin{corollary} \label{cor:SHGD_Conv}
Under the assumptions of Theorem \ref{theorem:SHGD_GenConv}, if for $\mathcal{L}_{\mathcal{V}}>0$
\begin{equation}\label{eq:Scaling_Vt_app}
    v_t \leq \mathcal{L}_{\mathcal{V}} \E \left[ H_t \right],
\end{equation}
the solution is
\begin{equation}
    \E \left[H_t \right] \leq  H_0  e^{\left(-2 \mu^2 + \eta \mathcal{L}_{\mathcal{V}} \mathcal{L}_{\mathcal{T}} \right)t}.
\end{equation}
If instead
\begin{equation}\label{eq:Bounded_Vt_app}
    v_t \leq \mathcal{L}_{\mathcal{V}},
\end{equation}
we have
\begin{equation}
    \E \left[ H_t \right] \leq H_0 e^{-2 \mu^2 t} + \left( 1 - e^{-2 \mu^2 t}  \right) \frac{\eta \mathcal{L}_{\mathcal{V}} \mathcal{L}_{\mathcal{T}}}{2 \mu^2}.
\end{equation}
In more generality, if
\begin{equation}
    v_t \leq \mathcal{L}_{\mathcal{V}} \E \left[ H_t \right]^{\alpha},  \quad  \alpha \in [0,1) \cup (1, \infty),
\end{equation}
the solution is even more interesting as:
\begin{enumerate}
    \item If $\alpha>1$, $\E \left[ H_t \right] \rightarrow 0$ as $e^{-2 \mu^2t}$;
    \item If $\alpha<1$, $\E \left[ H_t \right] \rightarrow \left(\frac{\eta \mathcal{L}_{\mathcal{T}}\mathcal{L}_{\mathcal{V}}}{2 \mu^2}\right)^{\frac{1}{1-\alpha}}$.
\end{enumerate}
\end{corollary}

\begin{proof}
Let us first consider the case where, for some $\mathcal{L}_{\mathcal{V}}>0$,
\begin{equation}
    v_t \leq \mathcal{L}_{\mathcal{V}} \E \left[ H_t \right].
\end{equation}
This implies that 
\begin{equation}
    d \E \left[ H_t \right] \leq -2 \mu^2  \E \left[ H_t \right] dt + \frac{\eta}{2} \mathcal{L}_{\mathcal{T}} \mathcal{L}_{\mathcal{V}} \E \left[ H_t \right]dt.
\end{equation}
By renaming $r_t :=\E \left[ H_t \right]$, we have
\begin{equation}
    d r_t \leq -2 \mu^2 r_t dt + \eta \mathcal{L}_{\mathcal{T}}\mathcal{L}_{\mathcal{V}}r_t dt,
\end{equation}
which results in
\begin{equation}
    \E \left[ H_t \right] \leq  H_0  e^{\left(-2 \mu^2 + \eta \mathcal{L}_{\mathcal{V}} \mathcal{L}_{\mathcal{T}} \right)t}.
\end{equation}
If instead
\begin{equation}
    v_t \leq \mathcal{L}_{\mathcal{V}},
\end{equation}
we automatically have
\begin{equation}
    \E \left[ H (Z_t) \right] \leq H (Z_0) e^{-2 \mu^2 t} + \left( 1 - e^{-2 \mu^2 t}  \right) \frac{\eta \mathcal{L}_{\mathcal{V}} \mathcal{L}_{\mathcal{T}}}{2 \mu^2}.
\end{equation}
More in general, if we assume that for some $\mathcal{L}_{\mathcal{V}}>0$,
\begin{equation}
    v_t \leq \mathcal{L}_{\mathcal{V}} \E \left[ H_t \right]^{\alpha},
\end{equation}
we have
\begin{equation}
    d r_t \leq -2 \mu^2 r_t dt + \eta \mathcal{L}_{\mathcal{T}}\mathcal{L}_{\mathcal{V}}r_t^{\alpha}dt,
\end{equation}
which implies that for $\alpha \neq 1$:
\begin{equation}
    r_t\leq \sqrt[1-\alpha]{\frac{r_0^{-\alpha} e^{(\alpha-1) 2 \mu^2 t}\left(r_0 2 \mu^2-\eta \mathcal{L}_{\mathcal{T}}\mathcal{L}_{\mathcal{V}}r_0^\alpha\right)+\eta \mathcal{L}_{\mathcal{T}}\mathcal{L}_{\mathcal{V}}}{2 \mu^2}} =: b_t
\end{equation}
In these cases, we have that the bound $b_t$ converges or diverges depending on the magnitude of $\alpha$:
\begin{enumerate}
    \item If $\alpha>1$, $b_t \rightarrow 0$ as $e^{-2 \mu^2t}$;
    \item If $\alpha<1$, $b_t \rightarrow \left(\frac{\eta \mathcal{L}_{\mathcal{T}}\mathcal{L}_{\mathcal{V}}}{2 \mu^2}\right)^{\frac{1}{1-\alpha}}$.
\end{enumerate}
\end{proof}

\begin{corollary}
$\beta$-Error Bound on $F$ and $L$-Lipschitzianity on $\nabla \mathcal{H}_{\gamma^1,\gamma^2}$ and $\nabla \mathcal{H}$, implies that $v_t \leq 8 L^2 \beta^2 \E \left[ H_t \right]$.
\end{corollary}
\begin{proof}
    \begin{align} \label{eq:SHGD_Lip}
    v_t = & \E \left[ \E_{\bgamma}  [\lVert \nabla \mathcal{H}_{\bgamma} \left( Z_t \right) - \nabla \mathcal{H} \left( Z_t \right) \rVert^2_2] \right] = \E \left[ \E_{\bgamma}  [\lVert \nabla \mathcal{H}_{\bgamma} \left( Z_t \right) -\mathcal{H}_{\bgamma} \left( Z^{*} \right) + \nabla \mathcal{H} \left(Z^{*} \right) - \nabla \mathcal{H} \left( Z_t \right) \rVert^2_2] \right]  \nonumber \\
    \leq & 2 \E \left[ \E_{\bgamma}  [\lVert \nabla \mathcal{H}_{\bgamma} \left( Z_t \right) -\mathcal{H}_{\bgamma} \left( Z^{*} \right) \rVert^2_2] \right] + 2 \E \left[\E_{\bgamma}  [\lVert  \nabla \mathcal{H} \left(Z^{*} \right) - \nabla \mathcal{H} \left( Z_t \right) \rVert^2_2] \right]  \leq 4L^2 \E \left[ \lVert Z_t - Z^{*} \rVert^2_2 \right] \nonumber \\
    \leq &  4\beta^2 L^2  \E \left[\lVert F(Z_t) \rVert^2_2 \right]= 8 L^2 \beta^2 \E \left[ H_t \right].
\end{align}
\end{proof}

The following corollary exemplifies the case where $v_t$ is bounded and we achieve convergence only up to a certain ball.
\begin{corollary}
Under the assumptions of Theorem \ref{thm:SHGD_Dynamic_PIBG}, for $f(x,y) := x^{\top}\mathbf{\Lambda}y$, we have:

\begin{equation}
    \frac{\E \left[\lVert Z_t \rVert^2 \right]}{2} \overset{t \rightarrow \infty}{=} \frac{\eta}{2} \sum_{i=1}^{d} \sigma_i^2 >0.
\end{equation}
\end{corollary}

\begin{proof}
It is easy to see that
\begin{equation}
    \frac{\lVert Z_t \rVert^2}{2} = \sum_{i=1}^{d} \frac{\lVert Z^{i}_t \rVert^2}{2},
\end{equation}
where $Z^i:=(X^i,Y^i)$, and that
\begin{equation}
    d \left(\frac{\lVert Z^{i}_t \rVert^2}{2} \right) = - 2 \lambda_i^2 \frac{\lVert Z^{i}_t \rVert^2}{2} dt + \eta \sigma_i^2 \lambda_i^2 dt + \bigO(\text{Noise}).
\end{equation}
This implies that 
\begin{equation}
   d \left( \frac{\E \left[\lVert Z^{i}_t \rVert^2\right]}{2} \right) = - 2 \lambda_i^2 \frac{\E \left[\lVert Z^{i}_t \rVert^2\right]}{2} dt + \eta\sigma_i^2 \lambda_i^2 dt,
\end{equation}
which implies that
\begin{equation}
    \frac{\E \left[\lVert Z^{i}_t \rVert^2\right]}{2} = \frac{\lVert Z^{i}_0 \rVert^2}{2} e^{- 2 \lambda_i^2t}  + (1-e^{- 2 \lambda_i^2t} ) \frac{\eta \sigma_i^2}{2} \overset{t \rightarrow \infty}{\rightarrow} \frac{\eta \sigma_i^2}{2}.
\end{equation}
\end{proof}
Interestingly, one can recover convergence by allowing stepsize schedulers. In the following result, we derive a necessary and sufficient condition to craft such schedulers. Then, we provide two concrete examples.
\begin{corollary}[SHGD Insights] \label{thm:SHGD_Convergence_PIBG_NoG}
Under the assumptions of Theorem \ref{thm:SHGD_Dynamic_PIBG}, for $f(x,y) := x^{\top}\mathbf{\Lambda}y$, for any positive scheduler $\eta_t$ we have
\begin{equation}
     \frac{\E \left[\lVert Z_t \rVert^2 \right]}{2}= \sum_{i=1}^{d} e^{-2 \lambda^2_i \int_0^t \eta_s ds} \left(\frac{\lVert Z^i_0 \rVert^2}{2} + \eta \sigma_i^2 \lambda_i^2 \int_0^t e^{2 \lambda^2_i \int_0^s\eta_r dr} \eta^2_s ds\right).
\end{equation}
Therefore,
\begin{equation}
    \frac{\E \left[\lVert Z_t \rVert^2 \right]}{2} \overset{t \rightarrow \infty}{\rightarrow}  0 \iff  \int_0^\infty \eta_s ds = \infty \text{ and } \lim_{t \rightarrow \infty} \eta_t = 0.
\end{equation}
In particular,
\begin{enumerate}
    \item $\eta_t=1$ implies that \begin{equation}
     \frac{\E \left[\lVert Z_t \rVert^2 \right]}{2} \overset{t \rightarrow \infty}{=} \frac{\eta}{2} \sum_{i=1}^{d} \sigma_i^2 >0;
     \end{equation}
     \item $\eta_t = \frac{1}{(t+1)^{\gamma}}$ for $\gamma \in \{0.5, 1\}$, $\frac{\E \left[\lVert Z_t \rVert^2 \right]}{2} \rightarrow 0$;
     \item $\eta_t = \frac{1}{(t+1)^2}$, $\frac{\E \left[\lVert Z_t \rVert^2 \right]}{2} \nrightarrow 0$.
\end{enumerate}
\end{corollary}

\begin{proof}
For $f(x,y):= x^{\top}\mathbf{\Lambda}y$, with the noise on gradient assumption, the SDE when we include a scheduler $\eta_t$ is
    \begin{align}
    d Z_t =\mathbf{A}Z_t \eta_t dt + \sqrt{\eta} \eta_t \sigma \mathbf{B} d W_t,
    \end{align}
where
\begin{equation}
    \mathbf{A}= \left[\begin{array}{ll} -\mathbf{\Lambda}^2 & \mathbf{0}_d  \\  \mathbf{0}_d  &  -\mathbf{\Lambda}^2  \end{array}\right] \quad \text{ and } \quad \mathbf{B} = \left[\begin{array}{ll} \mathbf{0}_d &\mathbf{\Lambda}  \\\mathbf{\Lambda}  & \mathbf{0}_d  \end{array}\right].
\end{equation}

Therefore, 
\begin{equation}
        d \left( \frac{\lVert Z^{i}_t \rVert^2}{2} \right)  = - 2 \lambda_i^2 \eta_t \frac{\lVert Z^{i}_t \rVert^2}{2} dt + \eta \sigma_i^2 \lambda_i^2 \eta_t^2 dt + \bigO(\text{Noise}),
\end{equation}
which implies that 
\begin{equation}
        d \left(\frac{\E \left[\lVert Z^{i}_t \rVert^2\right]}{2} \right) = - 2 \lambda_i^2 \frac{\E \left[\lVert Z^{i}_t \rVert^2\right]}{2} \eta_t dt + \eta\sigma_i^2 \lambda_i^2 \eta_t^2dt,
\end{equation}
which ultimately implies that
\begin{equation}
     \frac{\E \left[\lVert Z^{i}_t \rVert^2\right]}{2}=  e^{-2 \lambda^2_i \int_0^t \eta_s ds} \left(\frac{\lVert Z^i_0 \rVert^2}{2} + \eta \sigma_i^2 \lambda_i^2 \int_0^t e^{2 \lambda^2_i \int_0^s\eta_r dr} \eta^2_s ds\right).
\end{equation}

Let us write out the necessary conditions for this quantity to go to 0:
\begin{enumerate}
    \item For the first part, $ e^{-2 \lambda^2_i \int_0^t \eta_s ds} \frac{\lVert Z^i_0 \rVert^2}{2}$ goes to $0$, if and only if $\int_0^\infty \eta_s ds  = \infty$;
    \item For the second part, we need $\eta \sigma_i^2\lambda_i^2  e^{-2 \lambda^2_i \int_0^t \eta_s ds} \int_0^t e^{2 \lambda^2_i\int_0^s\eta_r dr} \eta^2_s ds$ to go to $0$ as well.
\end{enumerate}

Let us rewrite the second condition in a more convenient way:
\begin{align}
     & \eta \sigma_i^2\lambda_i^2  e^{-2 \lambda^2_i \int_0^t \eta_s ds} \int_0^t e^{2 \lambda^2_i\int_0^s\eta_r dr} \eta^2_s ds =  \eta \sigma_i^2\lambda_i^2 \frac{\int_0^t e^{2 \lambda^2_i\int_0^s\eta_r dr} \eta^2_s ds}{e^{2 \lambda^2_i \int_0^t \eta_s ds}}.
\end{align}
Since $\int_0^\infty \eta_s ds  = \infty$, both the numerator and denominator diverge. Therefore, we can use L'Hôpital's rule:
\begin{align}
    \lim_{t \rightarrow \infty} \eta \sigma_i^2\lambda_i^2 \frac{\int_0^t e^{2 \lambda^2_i\int_0^s\eta_r dr} \eta^2_s ds}{e^{2 \lambda^2_i \int_0^t \eta_s ds}} = \lim_{t \rightarrow \infty} \eta \sigma_i^2\lambda_i^2 \frac{e^{2 \lambda^2_i\int_0^t\eta_s ds} \eta^2_t}{e^{2 \lambda^2_i \int_0^t \eta_s ds} 2 \lambda^2_i \eta_t} = \lim_{t \rightarrow \infty} \frac{\eta \sigma_i^2}{2} \eta_t,
\end{align}
which converges to $0$, if and only if $ \lim_{t \rightarrow \infty} \eta_t =0$.

Note that, if the first condition is violated, the first component does not go to $0$. If the second condition is not satisfied, the second component does not go to $0$.

In particular,
\begin{enumerate}
    \item $\eta_t = \frac{1}{t+1} \text{ and } 2 \lambda_i^2\neq 1 \implies \frac{\E \left[\lVert Z^{i}_t \rVert^2\right]}{2} = \frac{(t+1)^{-2 \lambda_i^2-1}\left(\frac{\lVert Z^{i}_0 \rVert^2}{2}(2 \lambda_i^2-1)(t+1)+\eta \sigma_i^2\lambda_i^2\left((t+1)^{2 \lambda_i^2}-t-1\right)\right)}{2 \lambda_i^2-1} \overset{t \rightarrow \infty}{\rightarrow} 0$;
    \item $\eta_t = \frac{1}{t+1} \text{ and } 2 \lambda_i^2 = 1 \implies \frac{\E \left[\lVert Z^{i}_t \rVert^2\right]}{2} = \frac{\frac{\lVert Z^{i}_0 \rVert^2}{2}+\eta \sigma_i^2\lambda_i^2 \log (t+1)}{t+1} \overset{t \rightarrow \infty}{\rightarrow} 0$;
    \item $\eta_t = \frac{1}{\sqrt{t+1}} \implies \frac{\E \left[\lVert Z^{i}_t \rVert^2\right]}{2} = e^{-4 \lambda_i^2 \sqrt{t+1}} \left( \frac{\lVert Z^{i}_0 \rVert^2}{2} e^{4 \lambda_i^2 }  + 2 \eta \sigma_i^2\lambda_i^2 \left( Ei(4 \lambda_i^2 \sqrt{t+1}) - Ei(4 \lambda_i^2) \right) \right) \overset{t \rightarrow \infty}{\rightarrow} 0$;
    \item $\eta_t = \frac{1}{(t+1)^2} \implies \frac{\E \left[\lVert Z^{i}_t \rVert^2\right]}{2} = \frac{\lVert Z^{i}_0 \rVert^2}{2} e^{-\frac{2 \lambda_i^2 t}{t+1}}-\eta \sigma_i^2\lambda_i^2 \left(\frac{\left(4 \lambda_i^4+4 \lambda_i^2+2\right) e^{-\frac{2 \lambda_i^2 t}{t+1}}}{8 \lambda_i^6}+\frac{\left(4 \lambda_i^4+4 \lambda_i^2(t+1)+2(t+1)^2\right)}{8 \lambda_i^6(t+1)^2} \right) \overset{t \nrightarrow \infty}{\rightarrow} 0$.
\end{enumerate}
\end{proof}

Now we study a case where the noise structure itself is enough to guarantee the convergence. In this case, $v_t$ scales with $\E[H_t]$.
\begin{corollary}[Noise on Data] \label{thm:SHGD_Convergence_PIBG_NoM}
For $f(x,y) = x^{\top} \E_{\xi}  \left[ \mathbf{\Lambda}_{\xi} \right] y$ such that $\mathbf{\Lambda}_{\xi}$ is diagonal, we have
\begin{enumerate}
    \item $\E \left[ Z_t\right] = \mathbf{\Tilde{E}}(t) z \overset{t \rightarrow \infty}{=} 0$;
    \item $\frac{\E \left[\lVert Z_t \rVert^2 \right]}{2} = \sum_{i=1}^{d} \frac{\lVert Z^i_0 \rVert^2}{2} e^{- \left(2 \lambda^2_i - \eta \sigma_i^2 \left(2 \lambda_1^2 + \sigma_i^2\right) \right) t}$.
\end{enumerate}
In particular, $\frac{\E \left[\lVert Z_t \rVert^2 \right]}{2} \rightarrow 0 $ if $\eta < \frac{2 \lambda^2_i}{\sigma_i^2 \left(2 \lambda_1^2 + \sigma_i^2\right)}, \quad \forall i$.
\end{corollary}

\begin{proof}
The derivation of the SDE is straightforward and the formula is
\begin{align}
    d Z_t =\mathbf{A}Z_t dt + \sqrt{\eta} \mathbf{B} d W_t,
\end{align}
where
\begin{equation}
    \mathbf{A}= \left[\begin{array}{ll} -\mathbf{\Lambda}^2 & \mathbf{0}_d  \\  \mathbf{0}_d  &  -\mathbf{\Lambda}^2 \end{array}\right] \quad \text{and} \quad \mathbf{B}\mathbf{B}^{\top} = \left[\begin{array}{ll} 2\mathbf{\Lambda}^2 \circ  \Sigma^2 + \Sigma^4 & 2\mathbf{\Lambda}^2 \circ  \Sigma^2 + \Sigma^4  \\ 2\mathbf{\Lambda}^2 \circ  \Sigma^2 + \Sigma^4  & 2\mathbf{\Lambda}^2 \circ  \Sigma^2 + \Sigma^4  \end{array}\right] \circ \left[\begin{array}{ll} \diag(X_t \circ X_t) & \diag(X_t \circ Y_t)  \\ \diag(X_t \circ Y_t)  & \diag(Y_t \circ Y_t)  \end{array}\right].
\end{equation}
It is easy to see that
\begin{equation}
    \frac{\lVert Z_t \rVert^2}{2} = \sum_{i=1}^{d} \frac{\lVert Z^{i}_t \rVert^2}{2},
\end{equation}
where $Z^i:=(X^i,Y^i)$, and that
\begin{equation}
   d \left( \frac{\E \left[\lVert Z^{i}_t \rVert^2\right]}{2} \right)= - 2 \lambda_i^2 \frac{\E \left[\lVert Z^{i}_t \rVert^2\right]}{2} dt + \eta \sigma_i^2 (2 \lambda_i^2 + \sigma_i^2)  \frac{\E \left[\lVert Z^{i}_t \rVert^2\right]}{2} dt,
\end{equation}
which ultimately implies that 
\begin{equation}
     \frac{\E \left[\lVert Z^{i}_t \rVert^2\right]}{2} = \frac{\lVert Z^i_0 \rVert^2}{2} e^{- \left(2 \lambda^2_i - \eta \sigma_i^2 \left(2 \lambda_1^2 + \sigma_i^2\right) \right) t}.
\end{equation}
Empirical validation of this result is provided in Figure \ref{fig:NoM}.
\end{proof}

\subsection{SEG}
\begin{lemma}[Dynamics of Hamiltonian] \label{thm:SEG_DynHamil}
Let $Z_t$ be the solution of the SEG SDE. Then,
\begin{equation}
    \E\left[\dot H_t\right] = - \E\left[\nabla H_t^{\top} F^{\seg}_t\right] + \frac{\eta}{2} \tr\left(\E\left[\Sigma^{\seg}_t\nabla^2 H_t\right]\right).
\end{equation}
Under the additional assumption that $\E_{\bgamma} \left[ \nabla F_{\bgamma}(Z_t)F_{\bgamma}(Z_t) \right] = \nabla F(Z_t) F(Z_t) $, the formula simplifies and is
    \begin{align}
      d \E \left[ H_t \right]  = &  - \E \left[F_t^{\top} (\nabla F_t - \rho (\nabla F_t)^2 ) F_t \right] dt + \frac{\eta}{2} \tr\left(\E\left[\Sigma^{\seg}_t\nabla^2 H_t\right]\right)dt.
    \end{align}
\end{lemma}

\begin{proof}
The SDE for SEG is
\begin{align}
    d Z_t = - F^{\seg}\left( Z_t \right) dt  + \sqrt{\eta} \sqrt{\Sigma^{\seg}_t}d W_{t}.
\end{align}
Therefore, by Itô's Lemma we have 
\begin{align}
      d   \E_{\bgamma} \left[ \mathcal{H}_{\bgamma}\left( Z_t \right) \right] = &  - \nabla \E_{\bgamma} \left[ \mathcal{H}_{\bgamma} \left( Z_t \right)\right]^{\top} F^{\seg}  dt  + \frac{\eta}{2} \tr \left(  \sqrt{\Sigma^{\seg}_t}  \nabla^2 \E_{\bgamma} \left[ \mathcal{H}_{\bgamma} \left( Z_t \right)\right]  \sqrt{\Sigma^{\seg}_t}  \right)dt + \bigO(\text{Noise}).
\end{align}
Therefore,
\begin{align}
  \E\left[\dot H_t\right] = - \E\left[\nabla H_t^{\top} F^{\seg}_t\right] + \frac{\eta}{2} \tr\left(\E\left[\Sigma^{\seg}_t\nabla^2 H_t\right]\right).
\end{align}

\end{proof}

\begin{theorem}[SEG General Convergence] \label{theorem:SEG_GenConv}
Consider the solution $Z_t$ of the SEG SDE with $\gamma^1\ne\gamma^2$.
Let $v_t := \E \left[\E_{\bgamma}  \left[\lVert  F^{\seg}\left( Z_t \right) - F^{\seg}_{\bgamma}\left( Z_t \right) \rVert_2^2 \right]\right]$ measure the error in $ F^{\seg}$, in expectation over the whole randomness up to time $t$. Suppose that:
\begin{enumerate}
    \item The smallest eigenvalue (in absolute value) $\mu_{\rho}$ of $\mathbf{M}(z)$ is non-zero, where $\mathbf{M}(z) = \diag(\mathbf{M}_{1,1}(z),\mathbf{M}_{2,2}(z))$, with $\mathbf{M}_{1,1}(z):=  \nabla^2 f_{xx}(z)+\rho \left( \nabla^2 f_{xy}(z) \nabla^2 f_{xy}(z)^T  - (\nabla^2 f_{xx}(z))^2  \right)$, and $\mathbf{M}_{2,2}(z):= -\nabla^2 f_{yy}(z)+\rho \left( \nabla^2 f_{xy}(z) \nabla^2 f_{xy}(z)^T  - (\nabla^2 f_{yy}(z))^2  \right)$;
    \item $\lVert \nabla^2 H (z) \rVert_{\text{op}} < \mathcal{L}_{\mathcal{T}}$, for all $z\in\R^{2d}$.
\end{enumerate}
Then,
\begin{equation}
    \E \left[ H_t\right] \leq  e^{-2\mu_{\rho}^2t} \left[ H_0 + \frac{\eta\mathcal{L}_{\mathcal{T}}}{2} \int_0^t v_s  e^{2\mu_{\rho}^2 s} ds \right].
\end{equation}
\end{theorem}

\begin{proof}
From the previous theorem, we have that
\begin{align}
  d  \E \left[ H_t \right]  = &  - \E \left[F_t^{\top} (\nabla F_t - \rho (\nabla F_t)^2 ) F_t \right] dt + \frac{\eta}{2} \tr\left(\E\left[\Sigma^{\seg}_t\nabla^2 H_t\right]\right)dt.
\end{align}
After observing that
\begin{equation}
    F^{\top}(z) (\nabla F(z) - \rho (\nabla F(z))^2 ) F(z)  = F(z)^{\top}\mathbf{M}(z) F(z).
\end{equation}
We have that 
\begin{align}
    d  \E \left[H_t \right]  = &  - \E \left[F^{\top}_t (\nabla F_t - \rho (\nabla F_t)^2 ) F_t \right] dt + \frac{\eta}{2} \tr\left(\E\left[\Sigma^{\seg}_t\nabla^2 H_t\right]\right)dt \nonumber \\
   = & - \E \left[F^{\top}_t\mathbf{M}_t F_t \right] dt + \frac{\eta}{2} \tr\left(\E\left[\Sigma^{\seg}_t\nabla^2 H_t\right]\right)dt\nonumber \\
   \leq &  - \mu_{\rho}^2 \E \left[ \lVert F_t \rVert_2^2 \right] dt + \frac{\eta}{2} \E \left[\lVert \nabla^2 H_t \rVert_{\text{op}} \tr\left(\Sigma^{\seg}_t \right) \right] dt \nonumber \\
    \leq &  - \mu_{\rho}^2 \E \left[ \lVert F_t \rVert_2^2 \right] dt  + \frac{\eta \mathcal{L}_{\mathcal{T}}}{2} \E \left[\tr\left(\Sigma^{\seg}_t \right) \right] dt,
\end{align}
where we used that $\tr(\mathbf{A} \mathbf{B}) \leq \tr(\mathbf{A}) \lVert \mathbf{B} \rVert_{\text{op}}$ if $\mathbf{A}$ is a real positive semi-definite matrix and $\mathbf{B}$ is of the same size. Then, we observe that
\begin{align}
    \E \left[\tr\left(\Sigma^{\seg}_t \right) \right] & =\E \left[\E_{\bgamma} \left[ \lVert  (F(Z_t) - F_{\gamma^1}(Z_t) - \rho \left( \nabla F(Z_t) F(Z_t) - \nabla F_{\gamma^1}(Z_t) F_{\gamma^2}(Z_t) \right) ) \rVert^2 \right]\right] = v_t,
\end{align}
which implies that 
\begin{align}
    d  \E \left[ H_t \right] & \leq - \mu_{\rho}^2 \E \left[ \lVert  F_t \rVert^2 \right] dt + \frac{\eta}{2} \mathcal{L}_{\mathcal{T}} v_t dt \nonumber \\
    & = -2 \mu_{\rho}^2  \E \left[ H_t \right]  dt + \frac{\eta}{2} \mathcal{L}_{\mathcal{T}} v_t dt,
\end{align}
which in turn implies that
\begin{equation}
    \E \left[ H_t \right] \leq  e^{-2\mu_{\rho}^2t} \left[ H_0 + \frac{\eta\mathcal{L}_{\mathcal{T}}}{2} \int_0^t v_s  e^{2\mu_{\rho}^2s} ds \right].
\end{equation}
\end{proof}

\begin{corollary} \label{cor:SEG_Conv}
Under the assumptions of Theorem \ref{theorem:SEG_GenConv}, if for $\mathcal{L}_{\mathcal{V}}>0$
\begin{equation}\label{eq:SEG_Scaling_Vt_app}
    v_t \leq \mathcal{L}_{\mathcal{V}} \E \left[ H_t \right],
\end{equation}
the solution is
\begin{equation}
    \E \left[H_t \right] \leq  H_0  e^{\left(-2 \mu_{\rho}^2 + \eta \mathcal{L}_{\mathcal{V}} \mathcal{L}_{\mathcal{T}} \right)t}.
\end{equation}
If instead
\begin{equation}\label{eq:SEG_Bounded_Vt_app}
    v_t \leq \mathcal{L}_{\mathcal{V}},
\end{equation}
we have
\begin{equation}
    \E \left[ H_t \right] \leq H_0 e^{-2 \mu_{\rho}^2 t} + \left( 1 - e^{-2 \mu_{\rho}^2 t}  \right) \frac{\eta \mathcal{L}_{\mathcal{V}} \mathcal{L}_{\mathcal{T}}}{2 \mu_{\rho}^2}.
\end{equation}
More in general, if
\begin{equation}
    v_t \leq \mathcal{L}_{\mathcal{V}}\E \left[ H_t \right]^{\alpha}, \quad \mathcal{L}_{\mathcal{V}}>0,  \quad  \alpha \in [0,1) \cup (1, \infty).
\end{equation}
The solution is even more interesting:
\begin{enumerate}
    \item If $\alpha>1$, $\E \left[ H_t \right] \rightarrow 0$ as $e^{-2 \mu_{\rho}^2t}$;
    \item If $\alpha<1$, $\E \left[ H_t \right] \rightarrow \left(\frac{\eta \mathcal{L}_{\mathcal{T}}\mathcal{L}_{\mathcal{V}}}{2 \mu_{\rho}^2}\right)^{\frac{1}{1-\alpha}}$.
\end{enumerate}
\end{corollary}
\begin{proof}
The proof is the same as Corollary \ref{cor:SHGD_Conv} where we substitute $\mu$ with $\mu_{\rho}$.
\end{proof}

\begin{corollary}
$\kappa_1$-Lipschitzianity on $F_{\gamma^1}$, $\kappa_2$-Lipschitzianity on $\nabla F_{\gamma^1} F_{\gamma^2}$, and $\beta$-Error Bound on $F$, implies that $v_t \leq 16 \beta^2 (\kappa_1^2 + \rho^2 \kappa_2^2) \E\left[ H_t \right]$.
\end{corollary}
\begin{proof}
\begin{align}
    v_t  & = \E \left[ \E_{\bgamma} \left[\lVert F\left( Z_t \right)  - F_{\gamma^1}\left( Z_t \right)  - \rho \left( \nabla F\left( Z_t \right)  F\left( Z_t \right)  - \nabla F_{\gamma^1}\left( Z_t \right)  F_{\gamma^2}\left( Z_t \right)  \right)  \rVert^2\right]\right]\nonumber \\
    &  \leq 2\E \left[ \E_{\gamma^1} \left[\lVert  F\left( Z_t \right) - F_{\gamma^1}\left( Z_t \right) \rVert^2 \right] \right] + 2 \rho^2 \E \left[ \E_{\bgamma} \left[\lVert   \nabla F\left( Z_t \right) F\left( Z_t \right) - \nabla F_{\gamma^1}\left( Z_t \right) F_{\gamma^2}\left( Z_t \right)  \rVert^2 \right] \right] \nonumber \\
    & \leq 8(\kappa_1^2 + \rho^2 \kappa_2^2) \E [\lVert Z_t - Z^{*}  \rVert^2_2] \leq 8\beta^2(\kappa_1^2 + \rho^2 \kappa_2^2) \E [\lVert F(Z_t) \rVert^2_2] = 16 \beta^2 (\kappa_1^2 + \rho^2 \kappa_2^2) \E\left[ H_t \right].
\end{align}
\end{proof}
The following corollary exemplifies the case where $v_t$ is bounded and we achieve convergence only up to a certain ball, and that selecting a proper $\rho$ has a crucial role: If it is too large, it might increase the suboptimality of the algorithm.
\begin{corollary}
Under the assumptions of Theorem \ref{thm:SEG_Dynamic_PIBG}, for $f(x,y) := x^{\top}\mathbf{\Lambda}y$, we have:
\begin{equation}
    \frac{\E \left[\lVert Z_t \rVert^2 \right]}{2} \overset{t \rightarrow \infty}{=} \eta \sum_{i=1}^{d} \sigma_i^2 \frac{1 + \rho^2 \lambda_i^2}{2\rho \lambda_i^2}>0.
\end{equation}
\end{corollary}
\begin{proof}
It is easy to see that:
\begin{equation}
    \frac{\lVert Z_t \rVert^2}{2} = \sum_{i=1}^{d} \frac{\lVert Z^{i}_t \rVert^2}{2},
\end{equation}
where $Z^i:=(X^i,Y^i)$, and that
\begin{equation}
    d \left( \frac{\lVert Z^{i}_t \rVert^2}{2} \right)  = - 2 \rho \lambda_i^2 \frac{\lVert Z^{i}_t \rVert^2}{2} dt + \eta \sigma_i^2 (1 + \rho^2\lambda_i^2) dt + \bigO(\text{Noise}).
\end{equation}
This implies that 
\begin{equation}
    d \left( \frac{\E \left[\lVert Z^{i}_t \rVert^2\right]}{2} \right) = - 2 \rho \lambda_i^2 \frac{\E \left[\lVert Z^{i}_t \rVert^2\right]}{2} dt + \eta\sigma_i^2 (1 + \rho^2\lambda_i^2) dt.
\end{equation}
Which implies that
\begin{equation}
    \frac{\E \left[\lVert Z^{i}_t \rVert^2\right]}{2} = \frac{\lVert Z^{i}_0 \rVert^2}{2} e^{- 2 \rho \lambda_i^2t}  + (1-e^{- 2 \lambda_i^2t} ) \frac{\eta \sigma_i^2}{2} \frac{(1 + \rho^2\lambda_i^2)}{\rho \lambda_i^2}\overset{t \rightarrow \infty}{\rightarrow} \eta \sigma_i^2 \frac{(1 + \rho^2\lambda_i^2)}{2\rho \lambda_i^2}.
\end{equation}
\end{proof}

Interestingly, one can recover convergence by allowing stepsize schedulers. In the following result, we derive a necessary and sufficient condition to craft such schedulers. Then, we provide two concrete examples.
\begin{corollary}[SEG Insights] \label{thm:SEG_Convergence_PIBG_NoG}
Under the assumptions of Theorem \ref{thm:SEG_Dynamic_PIBG}, for $f(x,y) := x^{\top}\mathbf{\Lambda}y$, for any positive schedulers $\eta_t$ and $\rho_t$ we have
\begin{equation}
     \frac{\E \left[\lVert Z_t \rVert^2 \right]}{2}= \sum_{i=1}^{d} e^{-2 \lambda^2_i \rho \int_0^t \eta_s \rho_s ds} \left(\frac{\lVert Z^i_0 \rVert^2}{2} + \eta \sigma_i^2  \int_0^t e^{2 \lambda^2_i \rho\int_0^s\eta_r \rho_r dr} \eta^2_s (1 + \lambda_i^2 \rho^2 \rho^2_s) ds\right).
\end{equation}
Therefore,
\begin{equation}
    \frac{\E \left[\lVert Z_t \rVert^2 \right]}{2} \overset{t \rightarrow \infty}{\rightarrow}  0 \iff  \int_0^\infty \eta_s \rho_s ds = \infty \text{ and } \lim_{t \rightarrow \infty} \eta_t \rho_t = \lim_{t \rightarrow \infty} \frac{\eta_t}{\rho_t} = 0.
\end{equation}
In particular, consistently with \citep{hsieh2020explore},
\begin{enumerate}
    \item $\eta_t=\rho_t=1$ implies that \begin{equation}
      \frac{\E \left[\lVert Z_t \rVert^2 \right]}{2} \overset{t \rightarrow \infty}{=} \eta \sum_{i=1}^{d} \sigma_i^2 \frac{1 + \rho^2 \lambda_i^2}{2\rho \lambda_i^2}>0;
     \end{equation}
     \item $\eta_t = \frac{1}{(t+1)^{\gamma}}$ and $\rho_t=1$, $\gamma \in \{0.5, 1\}$, $\frac{\E \left[\lVert Z_t \rVert^2 \right]}{2} \rightarrow 0$;
     \item $\eta_t = \frac{1}{(t+1)^2}$ and $\rho_t=1$, $\frac{\E \left[\lVert Z_t \rVert^2 \right]}{2} \nrightarrow 0$.
\end{enumerate}
\end{corollary}

\begin{proof}
In this case, the SDE when we include the schedulers $\eta_t$ and $\rho_t$ is
\begin{align}
    d Z_t =\mathbf{A}Z_t \eta_t dt + \sqrt{\eta} \eta_t \sigma \mathbf{B} d W_t
\end{align}
where
\begin{equation}
    \mathbf{A}= \left[\begin{array}{ll} - \rho \rho_t\mathbf{\Lambda}^2 & -\mathbf{\Lambda}  \\ \mathbf{\Lambda}  &  -\rho \rho_t\mathbf{\Lambda}^2  \end{array}\right] \quad \text{ and } \quad \mathbf{B} = \left[\begin{array}{ll} \mathbf{I}_d & - \rho \rho_t\mathbf{\Lambda}  \\ \rho \rho_t\mathbf{\Lambda}  & \mathbf{I}_d  \end{array}\right].
\end{equation}
Therefore, 
\begin{equation}
    d \left( \frac{\lVert Z^{i}_t \rVert^2}{2}\right)  = - 2 \rho \lambda_i^2 \rho_t \eta_t \frac{\lVert Z^{i}_t \rVert^2}{2} dt + \eta \sigma_i^2 \left( 1 + \lambda_i^2 \rho^2 \rho_t^2 \right) \eta_t^2 dt + \bigO(\text{Noise}),
\end{equation}
which implies that 
\begin{equation}
    d \left( \frac{\E \left[\lVert Z^{i}_t \rVert^2\right]}{2} \right) = - 2 \rho \lambda_i^2 \rho_t \eta_t  \frac{\E \left[\lVert Z^{i}_t \rVert^2\right]}{2} dt + \eta \sigma_i^2 \left( 1 + \lambda_i^2 \rho^2 \rho_t^2 \right) \eta_t^2 dt,
\end{equation}
which implies that
\begin{equation}
     \frac{\E \left[\lVert Z^{i}_t \rVert^2\right]}{2}=  e^{-2 \lambda^2_i \rho \int_0^t \eta_s \rho_s ds} \left(\frac{\lVert Z^i_0 \rVert^2}{2} + \eta \sigma_i^2  \int_0^t e^{2 \lambda^2_i \rho\int_0^s\eta_r \rho_r dr} \eta^2_s (1 + \lambda_i^2 \rho^2 \rho^2_s) ds\right).
\end{equation}
With arguments similar to Corollary \ref{thm:SHGD_Convergence_PIBG_NoG}, the necessary and sufficient conditions for convergence are the following:
\begin{enumerate}
    \item For the first part $ e^{-2 \lambda^2_i \int_0^t \rho_s \eta_s ds} \frac{\lVert Z^i_0 \rVert^2}{2}$ to go to $0$, we need $\int_0^\infty \eta_s \rho_s ds  =\infty$;
    \item For the second part to go to $0$, we need both $\frac{\eta_t}{\rho_t}$ and $\eta_t \rho_t$ to go to 0.
\end{enumerate}
For the schedulers above, the proofs of their convergence or divergence are the same as Corollary \ref{thm:SHGD_Convergence_PIBG_NoG} with different constants.
\end{proof}

Now we study a case where the noise structure itself is enough to guarantee the convergence. In this case, $v_t$ scales with $H_t$.
\begin{corollary}[SEG Insights] \label{thm:SEG_Convergence_PIBG_NoM}
For $f(x,y) = x^{\top} \E_{\xi}  \left[ \mathbf{\Lambda}_{\xi} \right] y$ such that $\mathbf{\Lambda}_{\xi}$ is diagonal, we have
\begin{enumerate}
    \item $\E \left[ Z_t\right] = \mathbf{\Tilde{E}}(t) \mathbf{\Tilde{R}}(t) z \overset{t \rightarrow \infty}{=} 0$;
    \item $\frac{\E \left[\lVert Z_t \rVert^2 \right]}{2} = \sum_{i=1}^{d} \frac{\lVert Z^i_0 \rVert^2}{2} e^{- \left(2 \rho \lambda^2_i - \eta \sigma_i^2 \left( 1 + \rho^2\left(2 \lambda_1^2 + \sigma_i^2\right)\right) \right) t}$.
\end{enumerate}
In particular, $\frac{\E \left[\lVert Z_t \rVert^2 \right]}{2} \rightarrow 0 $ if $2 \rho \lambda^2_i - \eta \sigma_i^2 \left( 1 + \rho^2\left(2 \lambda_1^2 + \sigma_i^2\right)\right)>0, \forall i$.
\end{corollary}
\begin{proof}
The derivation of the SDE is straightforward and is
\begin{align}
    d Z_t =\mathbf{A}Z_t dt + \sqrt{\eta} \mathbf{B} d W_t,
\end{align}
where
\begin{equation}
    \mathbf{A}= \left[\begin{array}{ll} - \rho\mathbf{\Lambda}^2 & -\mathbf{\Lambda} \\ \mathbf{\Lambda}  & - \rho\mathbf{\Lambda}^2  \end{array}\right] \text{ and } \mathbf{B}\mathbf{B}^{\top} = \left[\begin{array}{ll}  \Sigma^2  &  \Sigma^2  \\  \Sigma^2  &  \Sigma^2  \end{array}\right] \circ \left[\begin{array}{ll} D_{1,1}  & D_{1,2}  \\ D_{2,1}  & D_{2,2}  \end{array}\right],
\end{equation}
where
\begin{align}
    & D_{1,1} := \diag((Y_t + \rho\mathbf{\Lambda} X_t) \circ (Y_t + \rho\mathbf{\Lambda} X_t) + \rho^2 \Sigma^2 \circ (\mathbf{\Lambda}^2 + \Sigma^2) \circ Y_t\circ Y_t), \\
    &  D_{2,2} := \diag((X_t - \rho\mathbf{\Lambda} Y_t) \circ (X_t - \rho\mathbf{\Lambda} Y_t) + \rho^2 \Sigma^2 \circ (\mathbf{\Lambda}^2 + \Sigma^2) \circ X_t\circ X_t),
\end{align}
and $D_{1,2}$ and $D_{2,1}$ do not matter for this calculation.

It is easy to see that
\begin{equation}
    \frac{\lVert Z_t \rVert^2}{2} = \sum_{i=1}^{d} \frac{\lVert Z^{i}_t \rVert^2}{2},
\end{equation}
where $Z^i:=(X^i,Y^i)$, and that
\begin{equation}
    d \left( \frac{\E \left[\lVert Z^{i}_t \rVert^2\right]}{2} \right)  = - 2 \rho \lambda_i^2 \frac{\E \left[\lVert Z^{i}_t \rVert^2\right]}{2} dt + \eta \sigma_i^2 ( 1 + \rho(2 \lambda_i^2 + \sigma_i^2))  \frac{\E \left[\lVert Z^{i}_t \rVert^2\right]}{2} dt,
\end{equation}
which ultimately implies that 
\begin{equation}
     \frac{\E \left[\lVert Z^{i}_t \rVert^2\right]}{2} = \frac{\lVert Z^i_0 \rVert^2}{2} e^{- \left(2 \rho \lambda^2_i - \eta \sigma_i^2 \left( 1 + \rho^2\left(2 \lambda_1^2 + \sigma_i^2\right)\right) \right) t}.
\end{equation}
Empirical validation of this result is provided in Figure \ref{fig:NoM}.
\end{proof}

\section{EXPERIMENTS} \label{app:Experiments}

In this section, we provide the details of the experiments we carried out to validate the theoretical results derived in the paper. We highlight that since we always use diagonal matrices, it is enough to validate our results in two dimensions. In the following, the choice of the initialization points does not have a special reason. When there is one, it is explained in the respective paragraphs.

\paragraph{Computational Infrastructure}
All experiments have been performed on Google Colaboratory without any premium subscription. The code to replicate the experiments is available at \url{https://github.com/eneamc/MinimaxSDEs}

\subsection{SDE Validation: Figure \ref{fig:SDEs}}
In this subsection, we provide the details to replicate the experiments shown in Figure \ref{fig:SDEs}. The objective is to provide empirical validation to Theorem \ref{thm:SEG_SDE_Insights} and \ref{thm:SHGD_SDE_Insights}: The trajectories of the simulated SDEs match that of the respective algorithms averaged over $5$ runs.

\paragraph{SGDA}
This paragraph refers to the \textit{top left} of Figure \ref{fig:SDEs}. Inspired by \textbf{Example 5.2} in \citep{hsieh2019convergence}, we study \textbf{Nonbilinear Game \# 1} $f(x,y) := x(y-0.45) + \phi(x) - \phi(y)$ where $\phi(z):=\frac{1}{4}z^2 - \frac{1}{2}z^4  + \frac{1}{6}z^6$. In the figure, we show the comparison between the average of $5$ realizations of the trajectories of SGDA with the average of $5$ simulations of the trajectories of the SDE of SGDA.

For each of the $5$ trajectories of SGDA, we initialize each trajectory at $(x_0,y_0)=(2.0,2.0)$ because this point is outside of the limit cycle that surrounds the optimal saddle point. We use a stepsize $\eta=0.01$, and run the optimizer for $N=10000$ iterations. The noise used to perturb the gradients is $Z \sim \mathcal{N}(0, \sigma^2 \mathbf{I}_{2})$ and $\sigma = 1.0$. Each trajectory is run with a different random seed.

For each of the $5$ trajectories of the SDE of SGDA, we initialize each trajectory at $(x_0,y_0)=(2.0,2.0)$, use a discretization step $dt= \frac{\eta}{10} = 0.001$ for the Euler–Maruyama method, and integrate the system for $N=100000$ iterations. The noise used to perturb the gradients is $Z \sim \mathcal{N}(0, \sigma^2 \mathbf{I}_{2})$ and $\sigma = 1.0$. Each trajectory is run with a different random seed.

\paragraph{SHGD}
This paragraph refers to the \textit{bottom left} of Figure \ref{fig:SDEs}. As a variation on \textbf{Example 5.1} and \textbf{Example 5.2} in \citep{hsieh2019convergence}, we study \textbf{Nonbilinear Game \# 3} $f(x,y) := xy + \phi(x) - \phi(y)$ where $\phi(z):=\frac{1}{2}z^2 - \frac{1}{4}z^4  + \frac{1}{6}z^6 - \frac{1}{8}z^8$. In the figure, we show the comparison between the average of $5$ realizations of the trajectories of SHGD with the average of $5$ simulations of the trajectories of the SDE of SHGD.

For each of the $5$ trajectories of SHGD, we initialize each trajectory at $(x_0,y_0)=(0.7,0.7)$. Given the extreme nonlinearity of this landscape, this initial point allows the use of sizeable stepsizes: we use a stepsize $\eta=0.0001$. Going further away from $(0,0)$ would require extremely smaller stepsizes to avoid numerical instabilities. We run the optimizer for $N=100000$ iterations. The noise used to perturb the gradients is $Z \sim \mathcal{N}(0, \sigma^2 \mathbf{I}_{2})$ and $\sigma = 1.0$. Each trajectory is run with a different random seed.

For each of the $5$ trajectories of the SDE of SHGD, we initialize each trajectory at $(x_0,y_0)=(0.7,0.7)$, use a discretization step $dt= \frac{\eta}{10} = 0.00001$ for the Euler–Maruyama method, and integrate the system for $N=1000000$ iterations. The noise used to perturb the gradients is $Z \sim \mathcal{N}(0, \sigma^2 \mathbf{I}_{2})$ and $\sigma = 1.0$. Each trajectory is run with a different random seed.

\paragraph{SEG}
This paragraph refers to the \textit{top right} and \textit{bottom right} of Figure \ref{fig:SDEs}. Inspired by \textbf{Example 5.1} in \citep{hsieh2019convergence}, we study \textbf{Nonbilinear Game \# 2} $f(x,y) := xy - \epsilon \phi(y)$ where $\phi(z):= \frac{1}{2}z^2 - \frac{1}{4}z^4$ and $\epsilon = 0.01$. For two different values of $\rho$, we show the comparison between the average of $5$ realizations of the trajectories of SEG with the average of $5$ simulations of the trajectories of the SDE of SEG. Additionally, we show the comparison with the average of $5$ simulations of the trajectories of the SDE of SGDA.

For each $ \rho \in \{0.1, 1 \}$, we repeat the following procedure: For each of the $5$ trajectories of SEG, we initialize each trajectory at $(x_0,y_0)=(1.0,1.0)$, use a stepsize $\eta=0.01$, and run the optimizer for $N=10000$ iterations. The noise used to perturb the gradients is $Z \sim \mathcal{N}(0, \sigma^2 \mathbf{I}_{2})$ and $\sigma = 1.0$. Each trajectory is run with a different random seed.

For each $ \rho \in \{0.1, 1 \}$, we repeat the following procedure: For each of the $5$ trajectories of the SDE of SEG, we initialize each trajectory at $(x_0,y_0)=(1.0,1.0)$, use a discretization step $dt= \frac{\eta}{10} = 0.001$ for the Euler–Maruyama method, and integrate the system for $N=100000$ iterations. The noise used to perturb the gradients is $Z \sim \mathcal{N}(0,\sigma^2 \mathbf{I}_{2})$ and $\sigma = 1.0$. Each trajectory is run with a different random seed.

For each of the $5$ trajectories of the SDE of SGDA, we initialize each trajectory at $(x_0,y_0)=(1.0,1.0)$, use a discretization step $dt= \frac{\eta}{10} = 0.001$ for the Euler–Maruyama method, and integrate the system for $N=100000$ iterations. The noise used to perturb the gradients is $Z \sim \mathcal{N}(0, \sigma^2 \mathbf{I}_{2})$ and $\sigma = 1.0$. Each trajectory is run with a different random seed.

For the \textit{top right} figure, we plot the average of the trajectories. For the \textit{bottom right} figure, we plot the average norm of the iterates $\mathbb{E}\left[\|x_k\|^2 + \|y_k\|^2\right]$ for the SEG and $\mathbb{E}\left[\|X_{\eta k}\|^2 + \|Y_{\eta k}\|^2\right]$ for the SDEs. We did not report the average norm of the SDE of SGDA as it diverges and would spoil the informativeness of the figure.

\subsection{Schedulers Validation: Figure \ref{fig:Schedulers}}
In this subsection, we provide the details to replicate the experiments shown in Figure \ref{fig:Schedulers}. The objective is to provide empirical validation to Prop. \ref{prop:SHGD_Convergence_PIBG_NoG_Insights_NoSched}, Prop. \ref{prop:SHGD_Convergence_PIBG_NoG_Insights_Sched}, Prop. \ref{prop:SEG_Convergence_PIBG_NoG_Insights_NoSched}, and Prop. \ref{prop:SEG_Convergence_PIBG_NoG_Insights_Sched}. Consistently with the assumptions of these theorems, the landscape is that of the Bilinear Game $f(x,y) = 2 x y$.

\paragraph{SHGD}
This paragraph refers to the \textit{left} of Figure \ref{fig:Schedulers}. As we use the stepsize scheduler $\eta_t:= \frac{1}{\left( t+1\right)^{\gamma}}$, for $\gamma \in \{0, 0.5, 1.0\}$, we compare the average norm of the iterates across $5$ realizations of the trajectories of SHGD with the exact dynamics of such a quantity prescribed in Prop. \ref{prop:SHGD_Convergence_PIBG_NoG_Insights_NoSched} and Prop. \ref{prop:SHGD_Convergence_PIBG_NoG_Insights_Sched}. For each of the $5$ trajectories of SHGD, we initialize each trajectory at $(x_0,y_0)=(0.1,0.1)$, use a stepsize $\eta=0.01$, and run the optimizer for $N=2000$ iterations. The noise used to perturb the gradients is $Z \sim \mathcal{N}(0, \sigma^2 \mathbf{I}_{2})$ and $\sigma = 0.001$. Each trajectory is run with a different random seed.

The left of Figure \ref{fig:Schedulers_New} reports an additional experiment with the same setup apart from $f(x,y)=xy$, $\sigma=0.01$, and $\rho=2$.

\paragraph{SEG}
This paragraph refers to the \textit{right} of Figure \ref{fig:Schedulers}. As we use the stepsize scheduler $\eta_t:= \frac{1}{\left( t+1\right)^{\gamma}}$, for $\gamma \in \{0, 0.5, 1.0\}$, we compare the average norm of the iterates across $5$ realizations of the trajectories of SEG with the exact dynamics of such a quantity prescribed in Prop. \ref{prop:SHGD_Convergence_PIBG_NoG_Insights_NoSched} and Prop. \ref{prop:SHGD_Convergence_PIBG_NoG_Insights_Sched}. For each of the $5$ trajectories of SEG, we initialize each trajectory at $(x_0,y_0)=(0.1,0.1)$, use a stepsize $\eta=0.01$, extra stepsize $\rho=1.0$, and run the optimizer for $N=2000$ iterations. The noise used to perturb the gradients is $Z \sim \mathcal{N}(0, \sigma^2 \mathbf{I}_{2})$ and $\sigma = 0.001$. Each trajectory is run with a different random seed.

The right of Figure \ref{fig:Schedulers_New} reports an additional experiment with the same setup apart from $f(x,y)=xy$, $\sigma=0.01$, and $\rho=2$.

\begin{figure}%
    \centering
    \subfloat{{\includegraphics[width=0.49\linewidth]{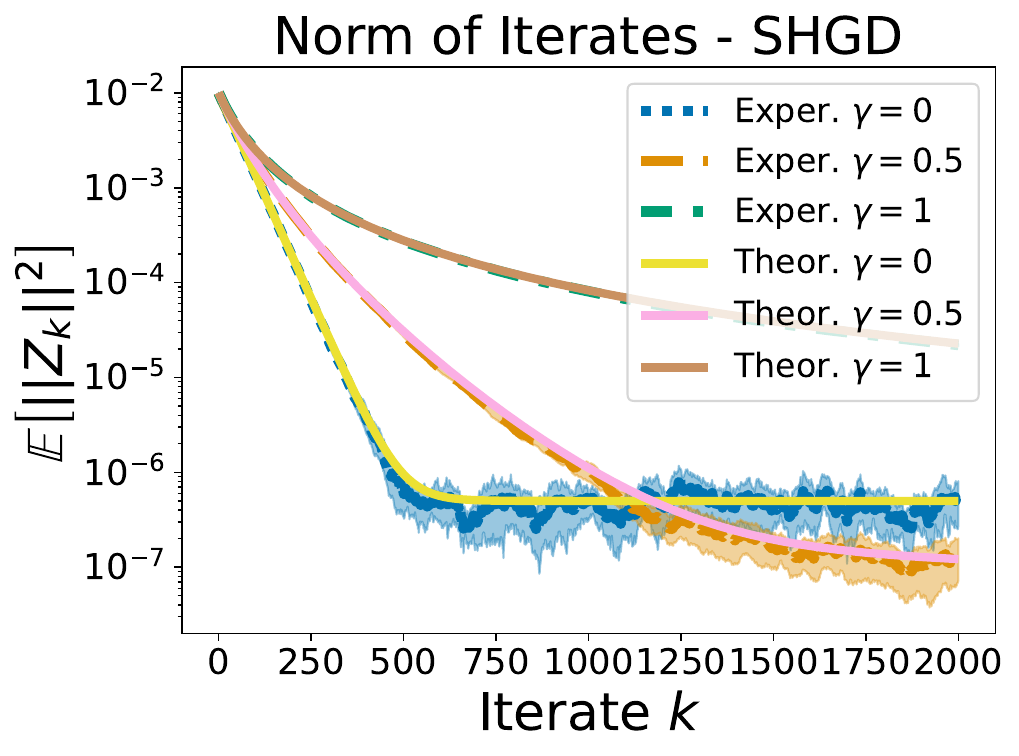} }}%
    \subfloat{{\includegraphics[width=0.49\linewidth]{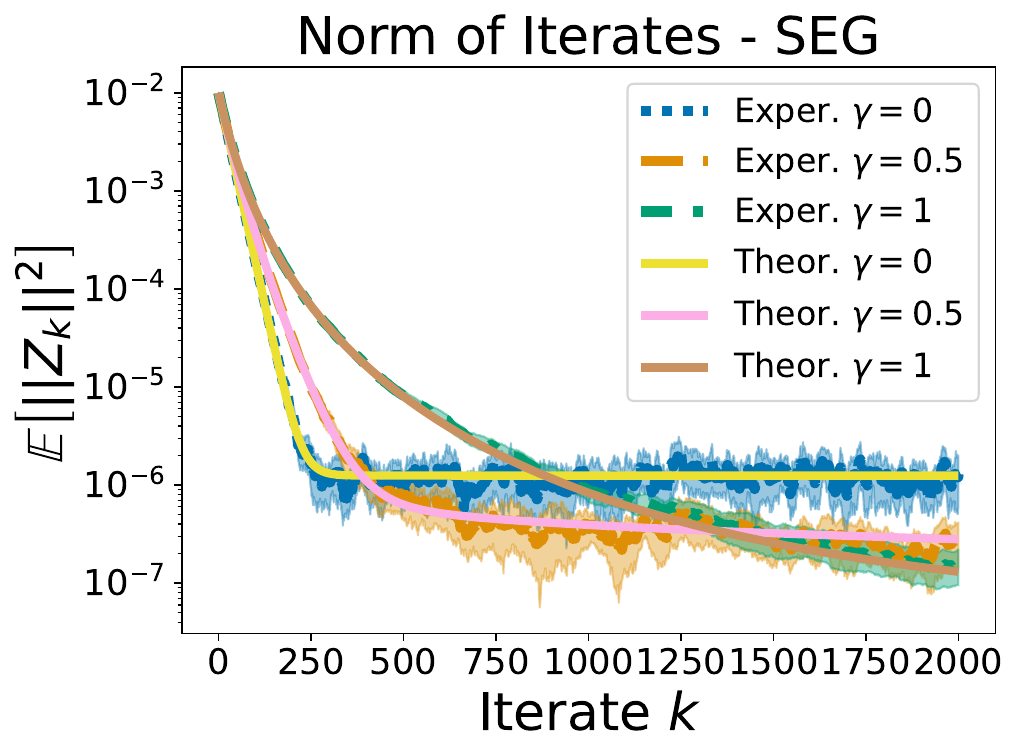} }}
    \caption{Empirical validation of Prop. \ref{prop:SHGD_Convergence_PIBG_NoG_Insights_NoSched} and Prop. \ref{prop:SHGD_Convergence_PIBG_NoG_Insights_Sched} (Left); Prop. \ref{prop:SEG_Convergence_PIBG_NoG_Insights_NoSched} and Prop. \ref{prop:SEG_Convergence_PIBG_NoG_Insights_Sched} (Right): The dynamics of $\E \left[\lVert Z_t \rVert^2 \right]$ averaged across $5$ runs perfectly matches that prescribed by our results for all schedulers. Both for SEG and SHGD, $\eta=0.01$, while $\rho=2$.}%
    \label{fig:Schedulers_New}%
\end{figure}

\subsection{Role of \texorpdfstring{$\rho$}{Lg}: \textit{Left} of  Figure \ref{fig:FBG_Comparison_Saddle}}
In this subsection, we provide the details to replicate the experiments shown on the \textit{left} of Figure \ref{fig:FBG_Comparison_Saddle}. The objective is to provide empirical validation to the insights derived in Paragraph \ref{parag:roleofrho_1} and Paragraph \ref{parag:roleofrho_2}. Consistently with the assumptions, the landscape is that of the Quadratic Game $f(x,y) = \frac{3}{2}x^2 + x y - \frac{3}{2}y^2$.

Let us remember that $\rho^{\text{H}}$ is meant to replicate the speed of the exponential decay of SHGD while $\rho^{\text{V}}$ is meant to achieve the lowest possible asymptotic variance of SEG. Finally, this is a case where \textbf{negative} $\rho$ converges to the optimum faster than SGDA and than any positive $\rho$. Of course, this particular choice confirms that large (absolute) values of $\rho$ result in larger suboptimality. It is key to notice that we indeed verify all these insights clearly in this Figure.

The left of Figure \ref{fig:FBG_Comparison_Saddle_New} reports an additional experiment with the same setting but $f(x,y) = x^2 + 3 x y - y^2$. In these cases, positive $\rho$ are the ones inducing fast convergence.

\begin{figure}%
    \centering
    \subfloat{{\includegraphics[width=0.49\linewidth]{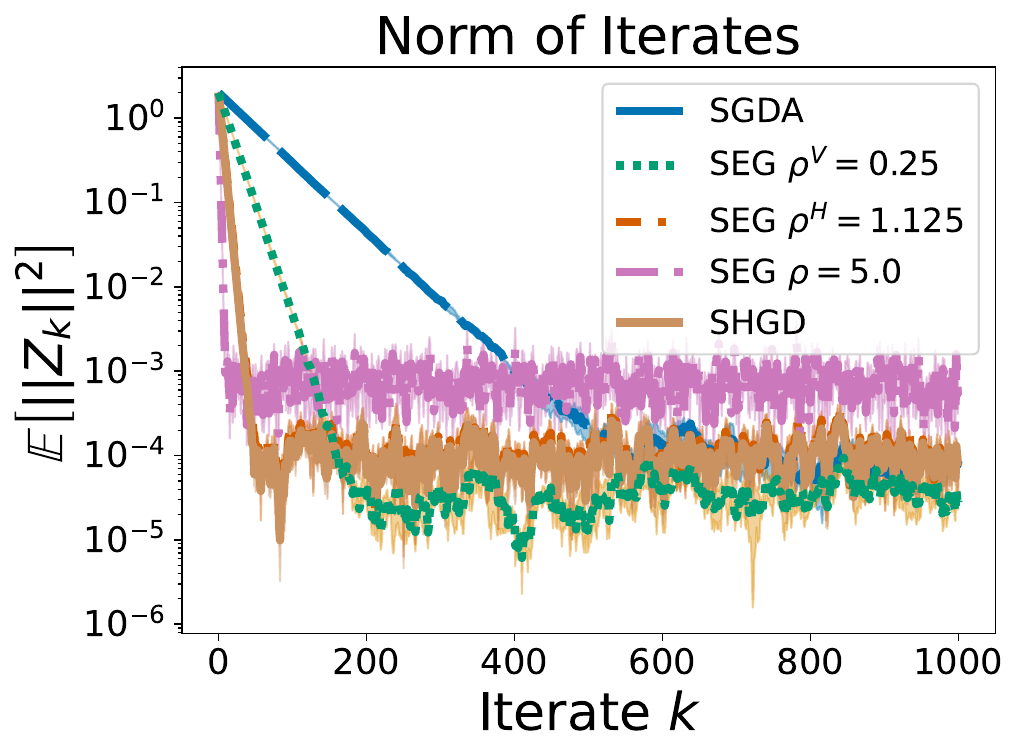} }}%
    \subfloat{{\includegraphics[width=0.49\linewidth]{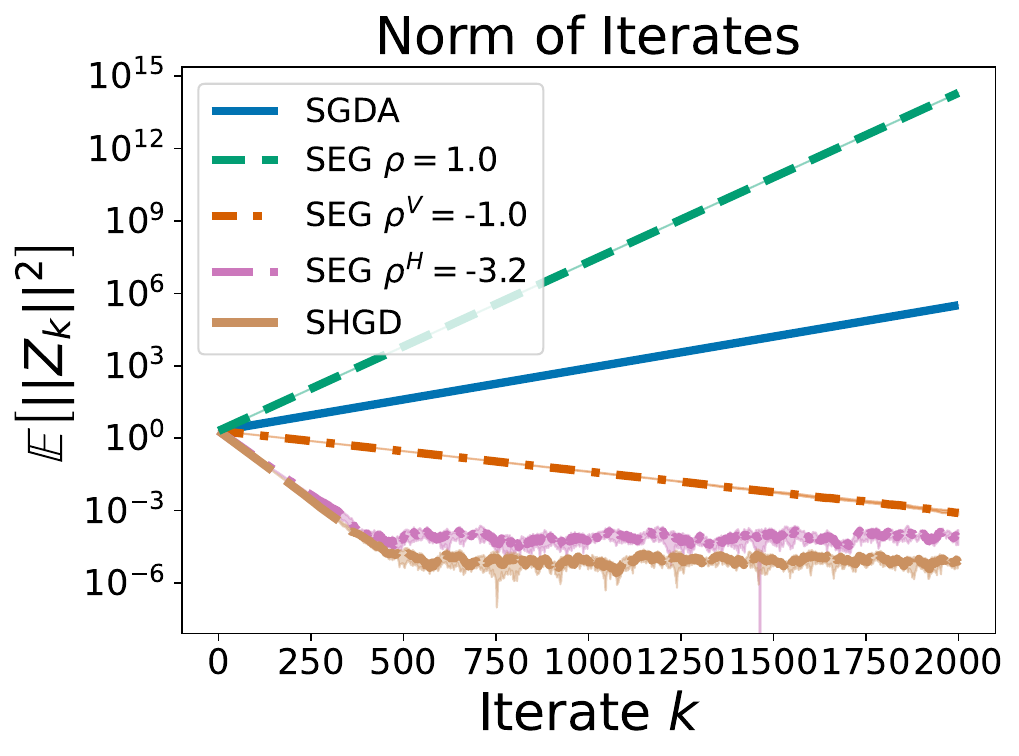} }}%
    \caption{Empirical validation of the comparison between SEG and SHGD on Quadratic Games: (Left), $\rho^{V}$ and $\rho^{H}$ clearly meet the designated goals. Large $|\rho|$ induces faster convergence which in turn results in larger suboptimality. (Right), positive $\rho$ escapes the \textit{bad saddle} faster than SGDA, negative ones induce convergence, and $\rho^{H}$ even matches the decay of SHGD. In both experiments, $\eta=0.01$.}%
    \label{fig:FBG_Comparison_Saddle_New}%
\end{figure}

\paragraph{SHGD}
We plot the average norm of the iterates across $5$ realizations of the trajectories of SHGD. For each of the $5$ trajectories of SHGD, we initialize each trajectory at $(x_0,y_0)=(1.0,1.0)$, use a stepsize $\eta=0.01$, and run the optimizer for $N=1000$ iterations. The noise used to perturb the gradients is $Z \sim \mathcal{N}(0, \sigma^2 \mathbf{I}_{2})$ and $\sigma = 0.1$. Each trajectory is run with a different random seed.

\paragraph{SEG}
For different values of $\rho$, we repeat the following procedure: We plot the average norm of the iterates across $5$ realizations of the trajectories of SEG. For each of the $5$ trajectories of SEG, we initialize each trajectory at $(x_0,y_0)=(1.0,1.0)$, use a stepsize $\eta=0.01$, and run the optimizer for $N=1000$ iterations. The noise used to perturb the gradients is $Z \sim \mathcal{N}(0, \sigma^2 \mathbf{I}_{2})$ and $\sigma = 0.1$. Each trajectory is run with a different random seed.
The values of $\rho$ we used are $\rho=-5$, $\rho^{\text{H}} = \frac{a^2 + \lambda^2 - a}{\lambda^2-a^2} = - 0.875$, $\rho^{\text{V}}=\frac{1}{a + \lambda} = \frac{1}{4}$, and $\rho=0$ which corresponds to SGDA in the Figure. 

\subsection{Escape from \textit{Bad Saddles}: \textit{Right} of  Figure \ref{fig:FBG_Comparison_Saddle}}
In this subsection, we provide the details to replicate the experiments shown on the \textit{right} of Figure \ref{fig:FBG_Comparison_Saddle}. The objective is to provide empirical validation to the insights derived in Paragraph \ref{parag:roleofrho_1} and Paragraph \ref{parag:roleofrho_2}, with a special focus on the ability of SEG to escape \textit{bad saddles} compared to SHGD that gets trapped. Consistently with the assumptions, the landscape is that of the Quadratic Game $f(x,y) = -\frac{1}{2}x^2 + 2 x y + \frac{1}{2}y^2$. We indeed observe that consistently with the theory, SHGD is attracted by such undesirable saddle points while suitable choices of $\rho$ allow SEG to escape the saddle. Interestingly, unfortunate choices of $\rho$ replicate the regrettable behavior of SHGD.

The right of Figure \ref{fig:FBG_Comparison_Saddle_New} reports an additional experiment with the same setting but $f(x,y) = -3x^2 + 2 x y +3y^2$. In these cases, positive $\rho$ are the ones inducing fast divergence while negative ones induce (undesirable) convergence to the saddle.

\paragraph{SHGD}
We plot the average norm of the iterates across $5$ realizations of the trajectories of SHGD. For each of the $5$ trajectories of SHGD, we initialize each trajectory at $(x_0,y_0)=(1.0,1.0)$, use a stepsize $\eta=0.001$, and run the optimizer for $N=2000$ iterations. The noise used to perturb the gradients is $Z \sim \mathcal{N}(0, \sigma^2 \mathbf{I}_{2})$ and $\sigma = 0.1$. Each trajectory is run with a different random seed.

\paragraph{SEG}
For different values of $\rho$, we repeat the following procedure: We plot the average norm of the iterates across $5$ realizations of the trajectories of SEG. For each of the $5$ trajectories of SEG, we initialize each trajectory at $(x_0,y_0)=(1.0,1.0)$, use a stepsize $\eta=0.001$, and run the optimizer for $N=2000$ iterations. The noise used to perturb the gradients is $Z \sim \mathcal{N}(0, \sigma^2 \mathbf{I}_{2})$ and $\sigma = 0.1$. Each trajectory is run with a different random seed.
The values of $\rho$ we used are $\rho=-1$, $\rho^{\text{H}} = \frac{a^2 + \lambda^2 - a}{\lambda^2-a^2} = 2$, $\rho^{\text{V}}=\frac{1}{a + \lambda} = 1$, and $\rho=0$ which corresponds to SGDA in the Figure.

\subsection{Empirical Validation of Figure \ref{fig:NoM}}
In this subsection, we provide the details to replicate the experiments shown in Figure \ref{fig:NoM}. The objective is to provide empirical validation to Corollary \ref{thm:SHGD_Convergence_PIBG_NoM} and Corollary \ref{thm:SEG_Convergence_PIBG_NoM}. Consistently with the assumptions, the landscape is that of the Stochastic Bilinear Game $f(x,y) = x^{\top} \E_{\xi}  \left[ \mathbf{\Lambda}_{\xi} \right] y$ such that $\mathbf{\Lambda}_{\xi}$, where $\mathbf{\Lambda} = 2 \mathbf{I}_{2}$ and $\xi \sim \mathcal{N}(0, \sigma^2 \mathbf{I}_{2})$, and $\sigma = 1.0$. We indeed observe that the average behavior of the norm of the iterates of SEG and SHGD matches that prescribed by Corollary \ref{thm:SHGD_Convergence_PIBG_NoM} and Corollary \ref{thm:SEG_Convergence_PIBG_NoM}.

\paragraph{SHGD}
We compare the average norm of the iterates across $5$ realizations of the trajectories of SHGD with the exact dynamics prescribed by Corollary \ref{thm:SHGD_Convergence_PIBG_NoM}. For each of the $5$ trajectories of SHGD, we initialize each trajectory at $(x_0,y_0)=(0.1,0.1)$, use a stepsize $\eta=0.01$, and run the optimizer for $N=200$ iterations. Each trajectory is run with a different random seed.

\paragraph{SEG}
We compare the average norm of the iterates across $5$ realizations of the trajectories of SEG with the exact dynamics prescribed by Corollary \ref{thm:SHGD_Convergence_PIBG_NoM}. For each of the $5$ trajectories of SHGD, we initialize each trajectory at $(x_0,y_0)=(0.1,0.1)$, use a stepsize $\eta=0.01$, extra stepsize $\rho \in \{0.5, 1, 2 \}$, and run the optimizer for $N=200$ iterations. Each trajectory is run with a different random seed.

\subsection{Empirical Validation of Figure \ref{fig:Variance}: The asymptotic variance of SEG is influenced by \texorpdfstring{$\rho$}{Lg}}
In this subsection, we provide the details to replicate the experiments shown in Figure \ref{fig:Variance}. The objective is to provide empirical validation to Eq. \ref{eq:varFBG_app}. Consistently with the assumptions, the landscape is that of the Quadratic Game $f(x,y) = x^2 + x y - y^2$. We indeed observe that the experimental average asymptotic variance of SEG matches the one prescribed by Eq. \ref{eq:varFBG_app}.

For different values of $\rho$, we repeat the following procedure: We compare the average asymptotic variance of the iterates across $5$ realizations of the trajectories of SEG with the exact formula prescribed by Eq. \ref{eq:varFBG_app}. For each of the $5$ trajectories of SEG, we initialize each trajectory at $(x_0,y_0)=(0.01,0.01)$, use a stepsize $\eta=0.01$, and run the optimizer for $N=200000$ iterations. Each trajectory is run with a different random seed. The values of $\rho$ used are $\rho \in \left\{-\frac{1}{6}, 0, \frac{1}{6}, \frac{1}{3}, \frac{2}{5}, \frac{1}{2}\right\}$.

\begin{figure}%
\centering

\subfloat{{\includegraphics[width=0.49\linewidth]{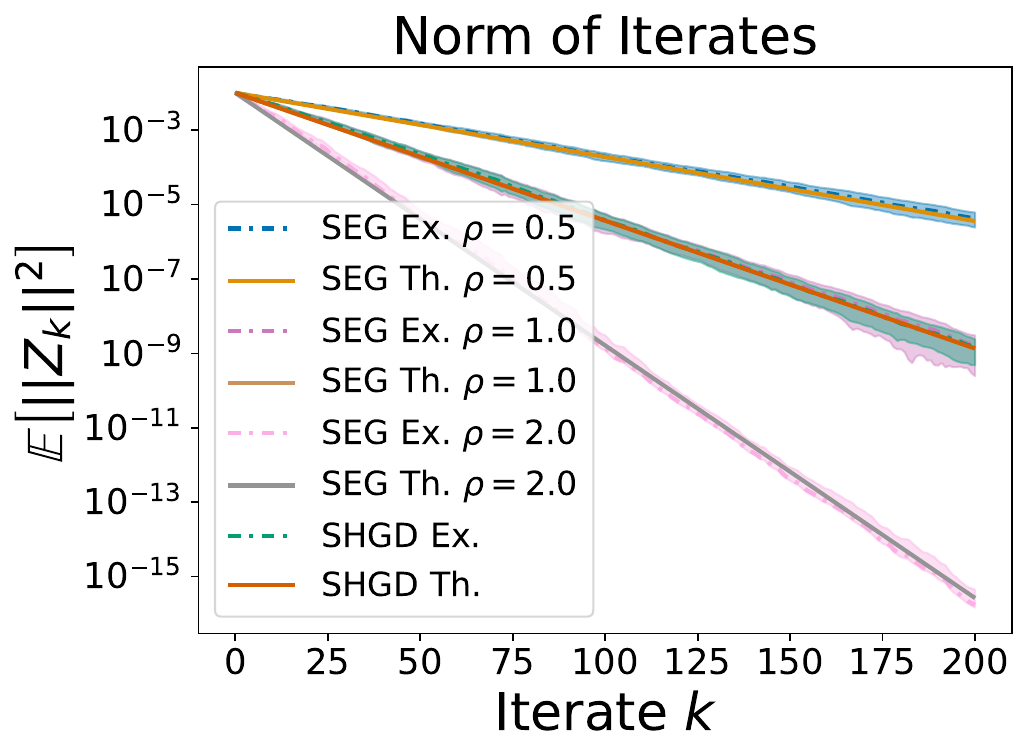} }}%
\subfloat{{\includegraphics[width=0.49\linewidth]{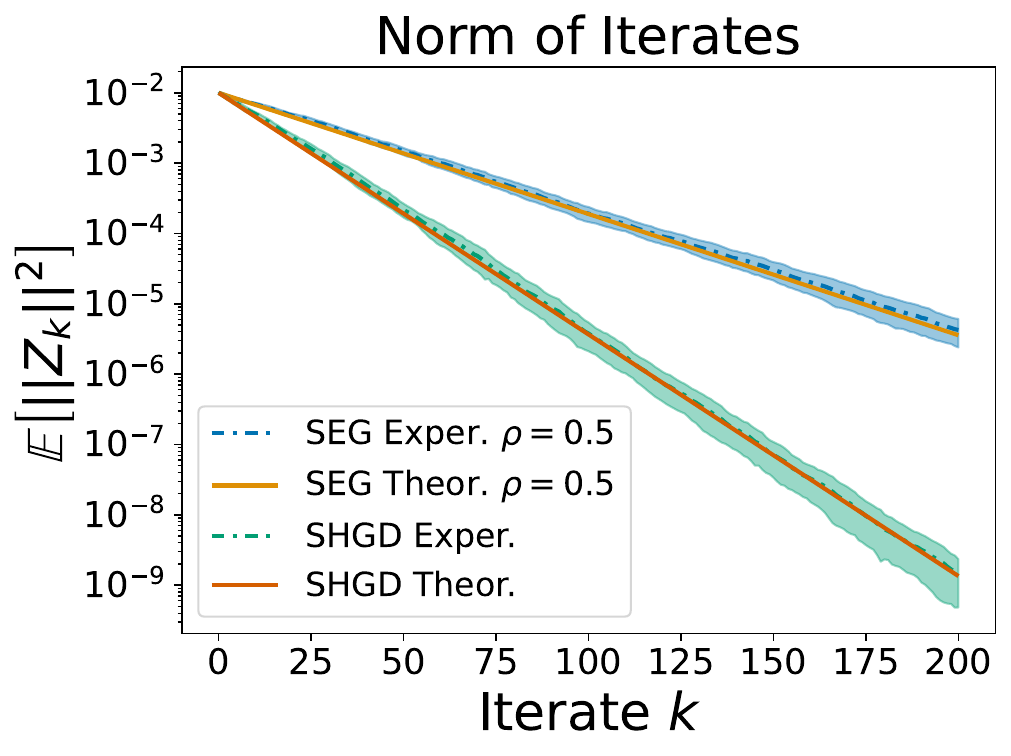} }}%

\caption{Empirical validation: Corollary \ref{thm:SHGD_Convergence_PIBG_NoM} and Corollary \ref{thm:SEG_Convergence_PIBG_NoM} (Left) and detail of SEG vs SHGD when their convergence speed matches (Right).}%
\label{fig:NoM}%
\end{figure}

\begin{figure}%
\centering
{\includegraphics[width=0.39\linewidth]{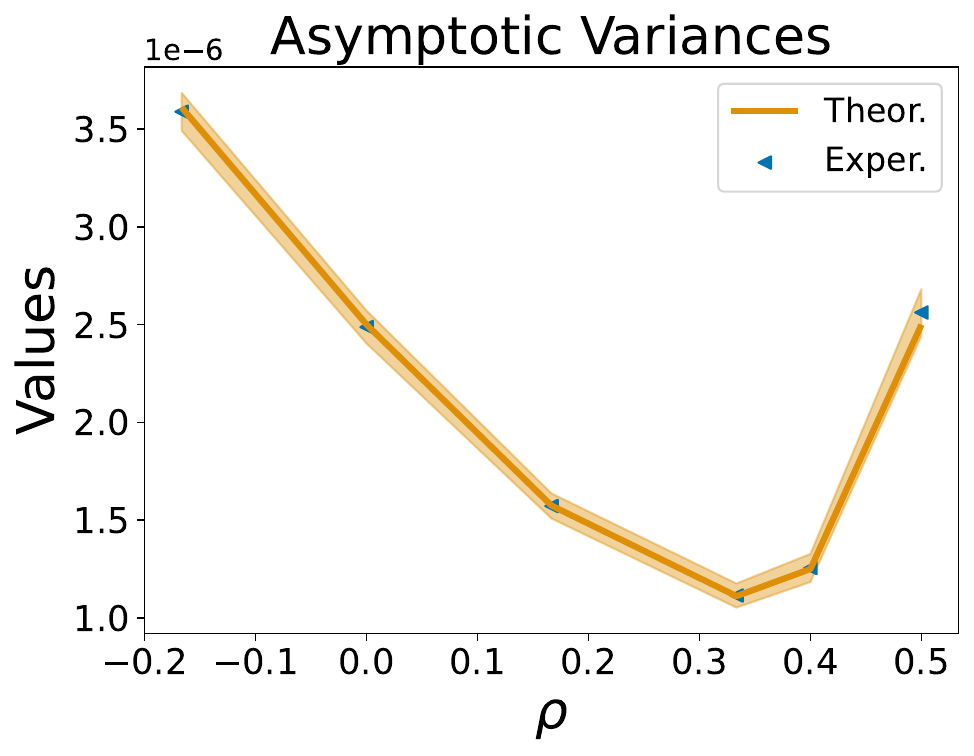} }%
    \caption{Empirical validation of Equation \eqref{eq:varFBG_app}.}%
    \label{fig:Variance}%
\end{figure}

\subsection{Empirical Validation of SDEs: Figure \ref{fig:SDE_Validation}}

In this subsection, we provide the details to replicate the experiments shown in Figure \ref{fig:SDE_Validation}. The objective is to show that if $\rho = \mathcal{O}(\eta)$ or even smaller, the SDE of SGDA models the dynamics of SEG accurately. However, once $\rho = \mathcal{O}(\sqrt{\eta})$ or even larger, the SDE of SGDA no longer models the dynamics of SEG correctly while the SDE of SEG does so. To simulate the SDEs, we use Algorithm \ref{algo:EulerMaruryama_SDE}.

\paragraph{Nonbilinear Game \# 1}
In this paragraph, we provide the details of the Nonbilinear Game \# 1 experiment. We optimize the loss function $f(x,y) := x(y-0.45) + \phi(x) - \phi(y)$ where $\phi(z):=\frac{1}{4}z^2 - \frac{1}{2}z^4  + \frac{1}{6}z^6$. The noise used to perturb the gradients is $Z \sim \mathcal{N}(0, \sigma^2 \mathbf{I}_{2})$ and $\sigma = 1.00$. We use $\eta = 0.001$, $\rho \in \{0.0001, 0.001, 0.0316, 0.1\}$. The results are averaged over $5$ experiments.

\paragraph{Nonbilinear Game \# 2}
In this paragraph, we provide the details of the Nonbilinear Game \# 2 experiment. We optimize the loss function $f(x,y) := xy - \epsilon \phi(y)$ where $\phi(z):= \frac{1}{2}z^2 - \frac{1}{4}z^4$ and $\epsilon = 0.01$. The noise used to perturb the gradients is $Z \sim \mathcal{N}(0, \sigma^2 \mathbf{I}_{2})$ and $\sigma = 1.00$. We use $\eta = 0.01$, $\rho \in \{0.001, 0.01, 0.01, 0.3\}$. The results are averaged over $5$ experiments.

\paragraph{Nonbilinear Game \# 3}
In this paragraph, we provide the details of the Nonbilinear Game \# 3 experiment. We optimize the loss function $f(x,y) := xy + \phi(x) - \phi(y)$ where $\phi(z):=\frac{1}{2}z^2 - \frac{1}{4}z^4  + \frac{1}{6}z^6 - \frac{1}{8}z^8$. The noise used to perturb the gradients is $Z \sim \mathcal{N}(0, \sigma^2 \mathbf{I}_{2})$ and $\sigma = 1.00$. We use $\eta = 0.0001$, $\rho \in \{0.00001, 0.0001, 0.01, 0.1\}$. The results are averaged over $5$ experiments.

\paragraph{Quadratic Game}
In this paragraph, we provide the details of the Quadratic Game experiment. We optimize the loss function $f(x,y) := x^2 + 2xy -y^2$. The noise used to perturb the gradients is $Z \sim \mathcal{N}(0, \sigma^2 \mathbf{I}_{2})$ and $\sigma = 1.00$. We use $\eta = 0.01$, $\rho \in \{0.001, 0.01, 0.1, 0.5\}$. The results are averaged over $5$ experiments.

\begin{figure}%
    \centering
    \subfloat{{\includegraphics[width=0.25\linewidth]{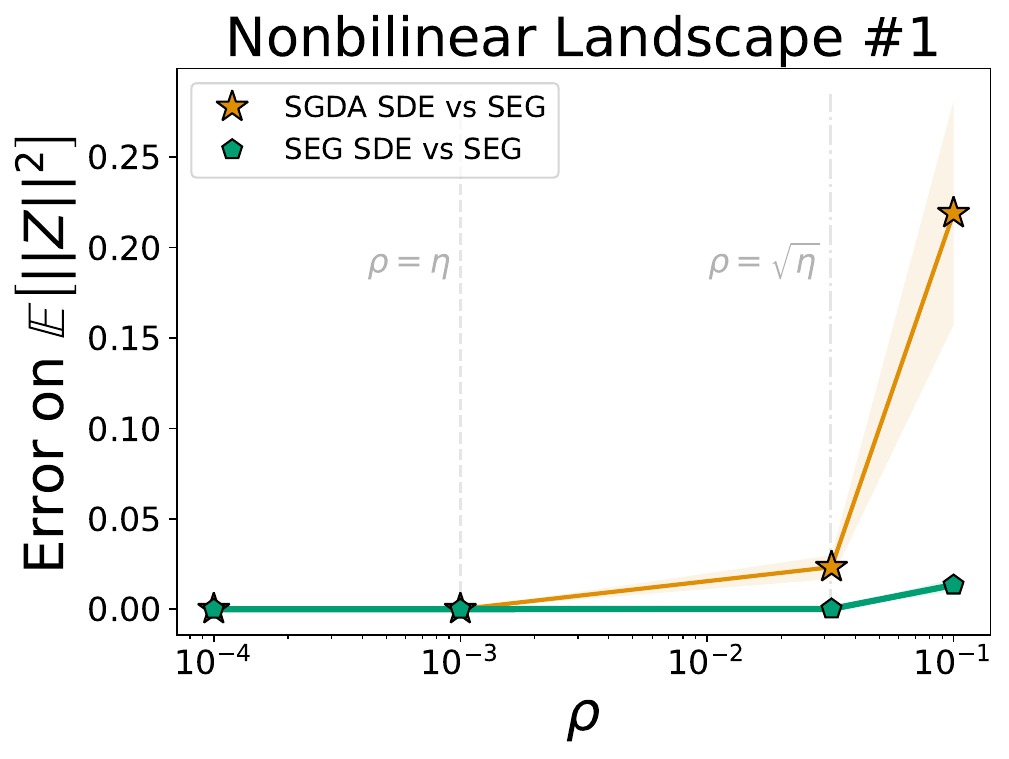} }}%
    \subfloat{{\includegraphics[width=0.24\linewidth]{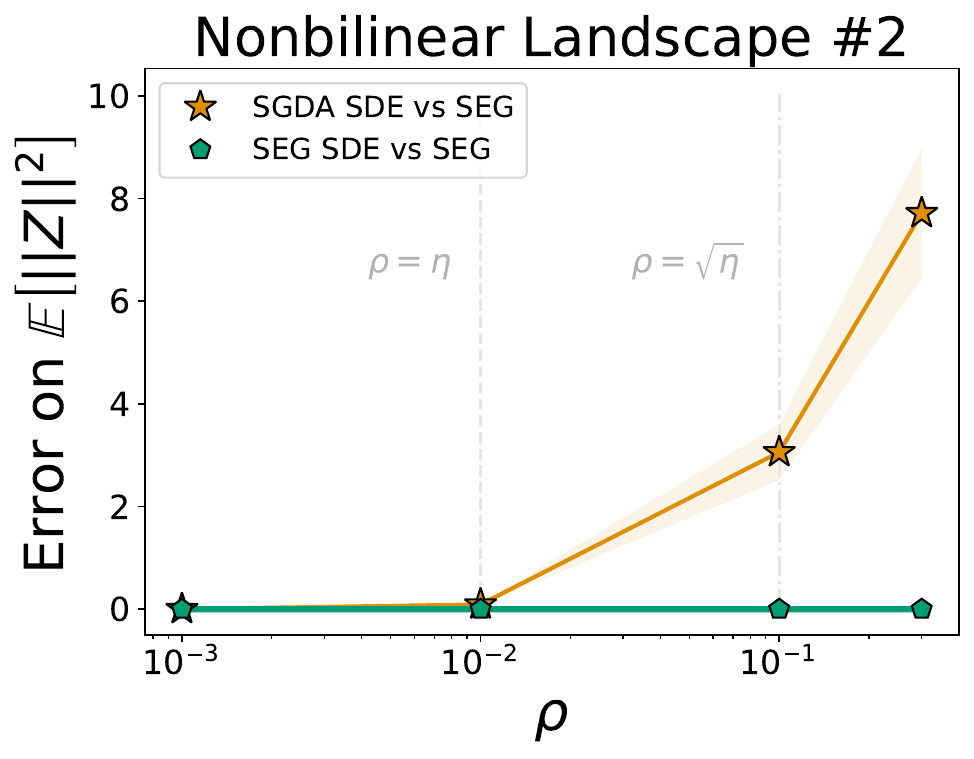} }}    \subfloat{{\includegraphics[width=0.25\linewidth]{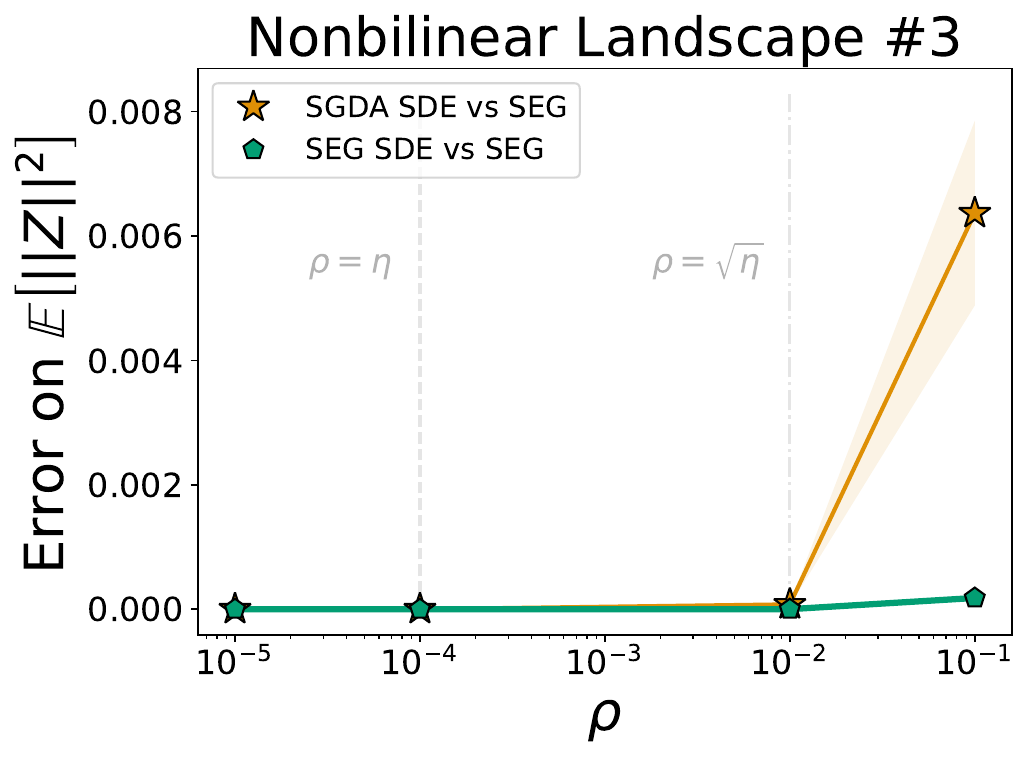} }}%
    \subfloat{{\includegraphics[width=0.24\linewidth]{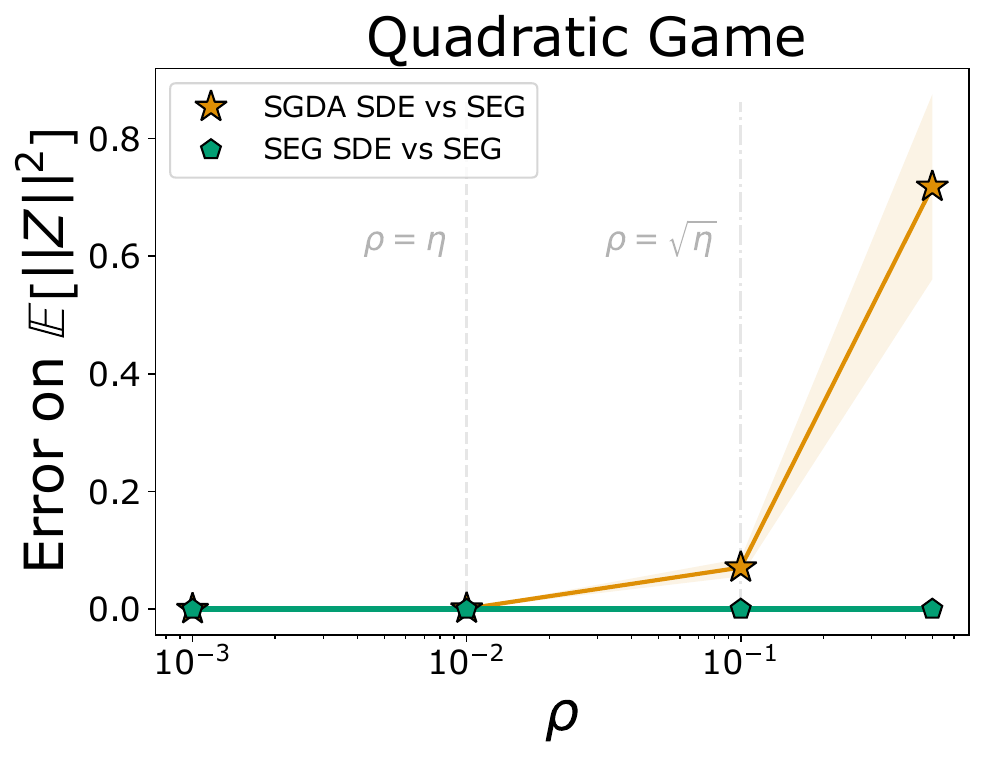} }}%
    \vspace{-2mm}
    \caption{Comparison in terms of $\E \left[ \lVert z \rVert^2 \right]$ with respect to $\rho$ - Nonbilinear Game \# 1 (Left); Nonbilinear Game \# 2 (Center Left); Nonbilinear Game \# 3 (Center Right); Quadratic Game (Right).}%
    \label{fig:SDE_Validation}%
\end{figure}

\end{document}